\documentclass{article} % For LaTeX2e
\usepackage{iclr2027_conference,times}
\iclrfinalcopy
\usepackage{amsmath,amsfonts,bm}

\def\eqref#1{equation~\ref{#1}}
\def\1{\bm{1}}

\DeclareMathAlphabet{\mathsfit}{\encodingdefault}{\sfdefault}{m}{sl}
\SetMathAlphabet{\mathsfit}{bold}{\encodingdefault}{\sfdefault}{bx}{n}

\newcommand{\R}{\mathbb{R}}

\usepackage{thm-restate}
\usepackage{hyperref}
\usepackage{url}
\usepackage{algorithm}
\usepackage{algpseudocode}
\usepackage{amsthm}
\usepackage{graphicx}
\usepackage{wrapfig}
\usepackage{booktabs}
\usepackage{comment}
\theoremstyle{plain}
\newtheorem{theorem}{Theorem}[section]
\newtheorem{proposition}[theorem]{Proposition}
\newtheorem{lemma}[theorem]{Lemma}

\theoremstyle{definition}

\theoremstyle{remark}
\newtheorem{remark}[theorem]{Remark}
\usepackage{tabularx}
\usepackage{mathtools}
\usepackage{caption}
\mathtoolsset{showonlyrefs}
\usepackage[table]{xcolor}

\usepackage{fontawesome5}
\usepackage{xcolor}
\usepackage{hyperref}

\definecolor{linkblue}{RGB}{0,0,128}
\newcommand{\iconlink}[3]{%
  \href{#1}{\textcolor{linkblue}{\textbf{#2\,#3}}}}

\newcommand{\PPd}{\mathcal{P}_2(\mathcal{P}_2(\mathbb{R}^d))}

\newcommand{\PV}{\mathcal{P}(\mathcal{V}^{\le L})}
\newcommand{\Pd}{\mathcal{P}(\mathbb{R}^d)}
\newcommand{\Po}{\mathcal{P}(\mathbb{R})}

\renewcommand{\d}{\mathop{}\!\mathrm{d}}

\title{SOL: Measuring Gaps between Text Distributions
by  Double Sliced Wasserstein Metrics}

\author{%
Gregor Kornhardt$^{1}$\thanks{Corresponding author.}
\quad Moritz Piening$^{2}$
\quad Jannis Chemseddine$^{1}$
\quad Gabriele Steidl$^{1}$\\[1ex]
$^{1}$Institute of Mathematics, TU Berlin, Berlin, Germany\\
$^{2}$Department of Computer Science, University of Göttingen,
Göttingen, Germany\\[1ex]
\texttt{\{kornhardt,steidl\}@math.tu-berlin.de}
}
\begin{document}

\maketitle
\lhead{Preprint} 
\vspace{-1cm}
\begin{center}
  %\iconlink{https://huggingface.co/your-models}{\faBrain}{Models}\hspace{2.5em}
  \iconlink{https://github.com/GregorKornhardt/SOL}{\faGithub}{Code}\hspace{2.5em}
  \iconlink{https://gregorkornhardt.github.io/SOL/}{\faGlobe}{Website}
\end{center}

\begin{abstract}
Evaluating text generation requires measuring how well the generated distribution matches the data distribution.
For autoregressive models, this is done by the perplexity.
Diffusion and flow-based language models can only provide a likelihood bound, whose tightness differs between model families.
Sample-based substitutes such as generative perplexity with entropy 
do not consider the distribution fit.
We propose SOL, 
 a distance between text distributions.
Each sequence is represented by the empirical measure of its hidden states under a fixed transformer and the distributions of these measures are compared by the double sliced Wasserstein distance. We prove that SOL is a metric if the transformer is injective.
Experiments show that SOL detects distributional failures, recovers expected model trends, and provides stable sample-based estimates.
We put forward SOL to fill the gap in the current evaluation protocol used for non-autoregressive models. As a first step we use SOL to re-evaluate a variety of models trained on OpenWebText. 
\end{abstract}

\section{Introduction}

Large language models based on autoregressive decoding are the prevailing
approach to text generation. Recently, diffusion and flow-based models, which
generate the whole sequence iteratively by reversing a noising process, have
become a serious alternative. Discrete diffusion models corrupt tokens by masking
or uniform replacement
\citep{austin2021structured,lou2024discretediffusionmodelingestimating,
sahoo2024simple,sahoo2025the}, while continuous models embed tokens into a
continuous space and define the noise process there
\citep{li2022diffusionlm,dieleman2022continuous,gulrajani2023likelihood,
hu2026elfembeddedlanguageflows,chen2026langflowcontinuousdiffusionrivals,CKS2026}.
This raises the question of how well the generated distribution
$p_{\mathrm{gen}}$ matches the data distribution $p_{\mathrm{data}}$.
 
For autoregressive models, the standard answer is the perplexity of held-out data, which is computed from the exact likelihood $p_\theta$.
For diffusion models, the likelihood is intractable.
Instead, an  upper bound of the perplexity is reported, based on an evidence lower bound (ELBO).
This causes two problems.
First, an ELBO is often not available.
For continuous models, the bound holds only in expectation over the noise, requires the change-of-variables formula along the flow and has only been derived for selected models \citep{gulrajani2023likelihood,chen2026langflowcontinuousdiffusionrivals}.
Moreover, most recent models use self-conditioning, where the network receives its own previous prediction as input.
This input arises along the sampling trajectory and cannot be recovered for a given data sequence.
Second, even when an ELBO exists, it is not comparable across model families.
\citet{sahoo2026scaling} observe that perplexity ranks models correctly within one family of diffusion models but not between families.
The reason is that the gap between ELBO and likelihood is a property of the noise process and its parameterization, not of the model alone.
For masked diffusion, the bound is tight \citep{sahoo2024simple,shi2024simplified}, whereas for uniform diffusion it is looser and not even optimized by the training objective \citep{gourevitch2026uniform}.
A better perplexity may thus reflect a tighter bound rather than a better model.
 
Since every model can be sampled, the generative perplexity (Gen.\ PPL) together with the token entropy has become the common substitute
\citep{zheng2025masked,franca2026hacking,hu2026elfembeddedlanguageflows,chen2026langflowcontinuousdiffusionrivals,pynadath2026generativefrontiersevaluationmatters}.
Here, Gen.\ PPL is the perplexity of the generated texts under a fixed reference model such as GPT-2 Large,
and the entropy is computed from token frequencies within sequences.
Since both are averaged over sequences, they do not reward variety in generated outputs, see Figure~\ref{fig:mode-collapse} for an example of mode collapse.
Moreover, Gen.\ PPL measures agreement with the reference model and not with $p_{\text{data}}$. Hence, it rewards imitation of the reference model instead of actual data.
In short, neither quantity measures the distance between $p_{\text{gen}}$ and $p_{\text{data}}$.
 
For image generation, the dominant evaluation method is FID \citep{heusel2017gans}.
It maps samples of both distributions into the feature space of a fixed network and compares the two feature distributions there.
We therefore seek an analogous score for text distributions that
\begin{enumerate}
\vspace{-3mm}
\setlength{\itemsep}{0pt}
\item[(1)] requires only samples from $p_{\text{gen}}$ and $p_{\text{data}}$,
\item[(2)] is a distance, i.e., it is zero if and only if the two distributions agree,
\item[(3)] reflects both the quality of the generated texts and their coverage of $p_{\text{data}}$,
\item[(4)] has an estimator which is stable in the number of samples, and
\item[(5)] is cheap to compute.
\vspace{-1.5mm}
\end{enumerate}
The most popular alternative to Gen.\ PPL is MAUVE
\citep{pillutla2021MAUVE,pillutla2023MAUVE}. It embeds each text by a single
latent representation, quantizes both samples by $K$-means clustering, and
compares the resulting histograms through a KL-divergence frontier. However,
MAUVE is not a distance, its value depends on the number of samples and
clusters, and its behavior depends strongly on the chosen representation. The
standard last hidden state of a causal transformer is insensitive to changes
near the beginning of a sequence \citep{he-etal-2023-blind}, while mean pooling
can map distinct token clouds to the same representation
\citep{hara2026mean}.

Instead, we represent each text by the empirical distribution of all its
contextual hidden states. A text distribution therefore becomes a distribution
over token clouds. Comparing these distributions naturally leads to a
Wasserstein-on-Wasserstein problem. Since solving this nested transport problem
directly is expensive, we use the double-sliced Wasserstein distance
\citep{piening2026slicing}. It first projects the hidden states onto
one-dimensional directions, represents the resulting measures by their quantile
functions, and then slices these functions again in $L_2([0,1])$. This
sorting-based construction gives \emph{sliced optimal transport for language}
(SOL), illustrated in Figure~\ref{fig:dsw4t-viz}.

\textbf{Contributions.}

\begin{enumerate}
\setlength{\itemsep}{1pt}

\item We propose SOL, a sample-based distance between text distributions, by representing each text as an empirical distribution of contextual token states and comparing the resulting distributions with the double sliced Wasserstein distance. Under an injectivity assumption, SOL is a metric on empirical text distributions (Proposition~\ref{prop}).

\item We establish the empirical behavior of SOL through controlled structural
perturbations, mode collapse, sampling changes, and human judgments. SOL detects
failures missed by Gen.\ PPL and token entropy, has a stable finite-sample
estimator, and correlates better with the evaluated human rankings than MAUVE
(Table~\ref{tab:encoder_human_corr_1024}).

\item  We use SOL to evaluate a broad range of autoregressive, discrete and continuous diffusion language models on OpenWebText (Table~\ref{tab:experiment-8d-owt-baselines}).

\end{enumerate}

A Python implementation of SOL is provided at \href{https://github.com/GregorKornhardt/SOL}{GitHub}.

\begin{figure}[t]
  \centering
  \includegraphics[width=\linewidth]{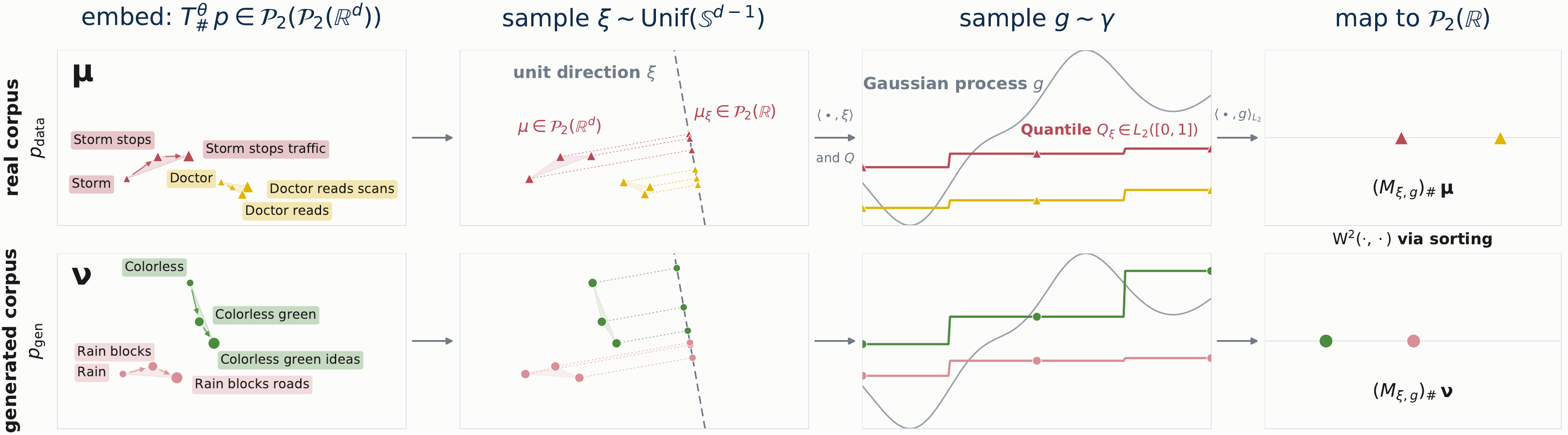
  }
\caption{\small{
Illustration of one SOL slice for fixed directions $\xi$ and $g$ using a causal
transformer. A sequence is represented by the point cloud
$T^\theta(\mathbf x)\in\mathcal P_2(\mathbb R^d)$ of its contextual hidden
states, and a corpus becomes a collection of such clouds. The first slice
projects each cloud along $\xi$; its one-dimensional empirical distribution is
then represented by a quantile function and sliced again along $g$. This maps
the corpus to $\mathcal P_2(\mathbb R)$, where the final Wasserstein distance is
computed. SOL integrates this quantity over $\xi$ and $g$.
}}
\vspace{-1.8em}
  \label{fig:dsw4t-viz}
\end{figure}

\section{Preliminaries}
Let  $\mathcal{V}$ be a finite vocabulary of tokens, and let $\mathcal V^{\le L}:= \bigcup_{l=1}^L \mathcal V^l$ be the set of  token sequences of length $\le L$.
By $\PV$, we denote the set of probability measures 
on $\mathcal V^{\le L}$.
Given a sequence of tokens $\mathbf x = (x_1,\ldots,x_{|\mathbf x|})$
of length $|\mathbf x|$, we write
$\mathbf x_{\le l} = (x_1,\ldots,x_l)$ for its partial sequence consisting of the first $l \le |\mathbf x|$ tokens.
Since all evaluation methods between text distributions in this paper
use transformers in a certain way, we give a brief description for fixing the notation.

\textbf{Causal transformers.}
A \textit{causal decoder-only transformer} $\mathbf H^\theta$ computes a hidden state at each position, which depends only on the preceding states, resulting in
\begin{align} \label{transf}
    \mathbf H^\theta: \mathcal{V}^{\leq L} \rightarrow \mathbb R^{d\times \leq L}, \quad  \mathbf x \mapsto (h^\theta_{   \mathbf x\le 1}, ... ,h^\theta_{\mathbf x\le L}).
\end{align}
From the output layer, we get the probabilities of the next token recursively by
$p_\theta(x_l ~\mid ~\mathbf x_{\le l -1}) 
    ~\coloneqq~
    \operatorname{softmax}(W_\theta  h^\theta_{\mathbf x \le l-1} + b_\theta)$ .
A transformer is used as an autoregressive generative model by factorizing the probability of a sequence $\mathbf x =(x_1,\ldots,x_L)$ as 
\begin{equation} \label{autoreg}
p_\theta(\mathbf{x}) = \prod_{l=1}^L p_\theta(x_l \mid \mathbf x_{\le l-1}),
\end{equation}
and
generating a new sequence is done by iteratively evaluating the model $p_\theta$ and sampling from it. 

\textbf{Bidirectional transformers.}
Alternatively, a \textit{bidirectional transformer} $\mathbf H^{\theta,\mathrm{bi}}$ computes hidden states at each position using the complete input sequence. In particular, for masked diffusion models, the input is a corrupted sequence $\mathbf x^t$ at diffusion time $t$ without time conditioning, resulting in
\begin{align}
    \mathbf H^{\theta,\mathrm{bi}}:
    \overline{\mathcal V}^L
    \rightarrow \mathbb R^{d\times L},
    \quad
    \mathbf x^t
    \mapsto
    \bigl(
    h^{\theta,\mathrm{bi}}_{\mathbf x^t,1},
    \ldots,
    h^{\theta,\mathrm{bi}}_{\mathbf x^t,L}
    \bigr),
\end{align}
where $\mathcal V \subset \overline{\mathcal V} $ may contain additional corruption tokens such as \texttt{[MASK]}. 

%%%%%%%%%%%%%%%%%%%%%%%%%%%%%%%%%%%%%%%%%%%%%%%%%%%%%%%%%%%%%%%%%%%%%%%%%%%%%%%%%%%%%%%%%%%%%%%%%%%%%%%%%%%%%%%%%%

\section{Current Evaluation Methods for Language Models}

We want to measure how well a generative text distribution $p_{\text{gen}}\in \PV$ matches the data distribution $p_{\text{data}} \in \PV$.
We consider two common approaches to quantitative text-generation evaluation: perplexity-based scores and latent token comparisons in the Euclidean space.

 \subsection{Perplexity-Based Methods}\label{subsec:Perplexity Based Scores}

\textbf{Perplexity.}
For an autoregressive \emph{learned} model $p_\theta$ in \eqref{autoreg}, 
the \textit{perplexity}  
\begin{align*}
    \operatorname{PPL}(p_\theta, p_{\mathrm{data}})
    = \exp \Big(-
    \mathbb E_{\mathbf x\sim p_{\mathrm{data}}}
    \Big[
    \log p_\theta(\mathbf x) 
    \Big] / \mathbb E_{\mathbf x\sim p_{\mathrm{data}}}[|\mathbf x|]
    \Big)
\end{align*}
is the exponential of the token-length normalized cross-entropy between the data distribution and the model distribution.
It measures how well the model assigns probability to samples from the data distribution $p_{\mathrm{data}}$.
Lower perplexity indicates that the model assigns higher likelihood to the observed data.

\textbf{Evidence Lower Bound Perplexity.}
For many diffusion and latent-variable models, computing the exact log-likelihood $\log p_\theta$ is intractable. Instead, we can evaluate - if available - an \emph{evidence lower bound} (ELBO) on the log-likelihood,
$
    \log p_\theta(\mathbf x)
    \geq \operatorname{ELBO}_\theta(\mathbf x)$.
This gives rise to the ELBO\emph{-based perplexity} $$\operatorname{PPL}_{\mathrm{ELBO}}(p_\theta, p_{\mathrm{data}})
    = \exp\Big(-
    \mathbb E_{\mathbf x\sim p_{\mathrm{data}}}
    \Big[
    \operatorname{ELBO}_\theta(\mathbf x)
    \Big] / \mathbb E_{\mathbf x\sim p_{\mathrm{data}}}[|\mathbf x|]
    \Big),
    $$
which upper-bounds the perplexity obtained from the exact model likelihood. However, comparisons based on the ELBO should be interpreted with care, since its value depends not only on the model likelihood, but also on the tightness of the variational bound. As an example, the ELBOs of masked and uniform diffusion language models seem to differ \cite{sahoo2026scaling}.
Moreover, for many generative language models 
no valid ELBOs are available. 

\textbf{Generative Perplexity.}
An alternative is therefore to evaluate generated samples under a \emph{fixed} reference model $p_\theta$.
Specifically, we choose a \emph{pretrained} autoregressive causal transformer $p_\theta$ as the reference model and define the \textit{generative perplexity} (Gen.\ PPL) as
\begin{align*}
    \operatorname{GenPPL}(p_{\mathrm{gen}},p_\theta)
    :=
    \exp\Big(-\frac{1}{\mathbb E_{\mathbf x\sim p_{\mathrm{gen}}}[|\mathbf x|]}
    \mathbb{E}_{\mathbf{x}\sim p_{\mathrm{gen}}}
    \Big[
    \sum_{l=1}^{|\mathbf{x}|}
    \log p_\theta(x_l\mid \mathbf{x}_{ \le l-1})
    \Big]
    \Big).
\end{align*}
Ignoring the sequence-length normalization, the corresponding cross-entropy decomposes as 
$\mathbb{E}_{\mathbf{x}\sim p_{\mathrm{gen}}}
    \left[-\log p_\theta(\mathbf{x})\right]
    =
    \text{H}(p_{\mathrm{gen}})
    +
    \operatorname{KL}(p_{\mathrm{gen}}\|p_\theta)
    $
with the  \emph{Kullback-Leibler divergence} KL    and the \emph{distribution entropy} H. 
Thus a generator can lower its Gen.\ PPL by lowering its entropy instead of moving closer to the reference model.
\begin{wraptable}{r}{0.28\textwidth}
\vspace{-0.0em}
\small
\centering
\footnotesize
\setlength{\tabcolsep}{6pt}
\begin{tabular}{l|c}
\toprule
GPT-2 & Gen.\ PPL ($\downarrow$) \\
\midrule
Small  & 12.5 \\
Medium & 11.4 \\
Large  & 7.5 \\
XL     & 8.5 \\
\bottomrule
\end{tabular}
\caption{Gen. PPL across GPT-2 sizes for unconditional generation prefers GPT-2 Large.}
\label{tab:gpt2_ppl}
\vspace{-0.8em}
\end{wraptable}
For this reason, Gen.\ PPL is usually reported together with the \emph{token entropy}
$
    H_{\text{tok}}(p_{\text{gen}}) := \mathbb E_{\mathbf x \sim p_{\text{gen}}} \big[ H(\hat p_{\mathbf x}) \big],
$
where $\hat p_{\mathbf x}$ denotes the empirical distribution of the tokens in $\mathbf x$ \citep{hu2026elfembeddedlanguageflows,chen2026langflowcontinuousdiffusionrivals,CKS2026, pynadath2026generativefrontiersevaluationmatters}.
Both quantities are expectations of a score of a single sequence, the log-likelihood under the reference model and the entropy within the sequence.
Neither of them looks at the distribution of the sequences, and $p_{\text{data}}$ does not enter at all.
In Figure~\ref{fig:mode-collapse}, we replace a growing fraction of a text sample by copies of a single sequence, the token entropy remains unchanged and Gen.\ PPL even improves.
Table~\ref{tab:gpt2_ppl} shows that Gen.\ PPL also favors the reference model itself, as GPT-2 Large scores better than the more powerful GPT-2 XL under a GPT-2 Large reference.
 
%-------------------------------------------------------------------------------
\subsection{Euclidean Latent Space Methods}
Standard methods for comparing language distributions, such as MAUVE \citep{pillutla2021MAUVE} and BERT-FID \citep{alihosseini2019jointly}, rely on a single latent token representation obtained from the final hidden layer of a transformer. Alternatively, the hidden states across positions can be aggregated by mean pooling. We briefly introduce both representations before turning to MAUVE.
\subsubsection{Text Representation in the Euclidean Space}

 \textbf{Last hidden-token representation.}
The  \emph{last hidden-token map} of a causal transformer in \eqref{transf} reads as
\begin{equation}\label{last_hidden}
\mathbf h^\theta:\mathcal V^{\leq L}\to\mathbb R^d, \quad \mathbf{x} \mapsto h^\theta_{ \mathbf x \le |\mathbf x| },
\end{equation}
where, in contrast to the whole sequence generator $\mathbf H^\theta$, we use $\mathbf h^\theta$ for the last-token representation.
It allows  to push the text distribution $p_{\text{data}} \in \mathcal{P}(\mathcal V^{ \le L})$ to a distribution in Euclidean space 
$\mathbf h^{\theta}_\# p_{\text{data}} \in \Pd.$
Under the assumption that language models are invertible, which is validated in the next remark, this distribution uniquely characterizes the language distribution.

\begin{remark}[Language models are almost surely invertible]\label{rem:inj}
In \citep{nikolaou2026language}, it was proved that the last-token map is almost surely invertible. Empirically, the authors found no collisions across several autoregressive language models and large prompt collections. They further showed that, under certain assumptions, causal transformers are injective after gradient-based training.
\end{remark}

\textbf{Mean-pooling.} In contrast to the above representation, 
\emph{mean-pooling} takes into account the tokens at all positions: 

\begin{align}\label{mean_pool}
     \overline{\mathbf H}^\theta: V^{\leq L}\rightarrow \mathbb{R}^d, \quad \mathbf  x\mapsto \frac{1}{|\mathbf x|}\sum_{l=1}^{|\mathbf x|} h^\theta_{\mathbf x \le l}.
\end{align}

The associated distribution $ \overline{\mathbf H}^{\theta}_\# p_{\text{data}} \in \mathcal P (\mathbb R^d)$ allows for standard methods such as MAUVE or FBD (BERT-FID). Mean pooling retains only the first moment of the token-cloud distribution and discards higher-order structure, so token clouds with the same mean but different spread or shape become indistinguishable, see \cite{hara2026mean} and Appendix~\ref{sec:dsw_mean_sw}. Moreover, MAUVE on mean-pooled features saturates under token pollution: it is already close to zero at 10\% \texttt{[CLS]} insertion and no longer distinguishes moderate from severe corruption, see Appendix~\ref{subsec:MAUVE with Mean Hidden Representation}.

 \subsubsection{MAUVE}
 \begin{wrapfigure}[13]{r}{0.6\linewidth}
 \vspace{-1cm}
    \centering
    \includegraphics[width=\linewidth]{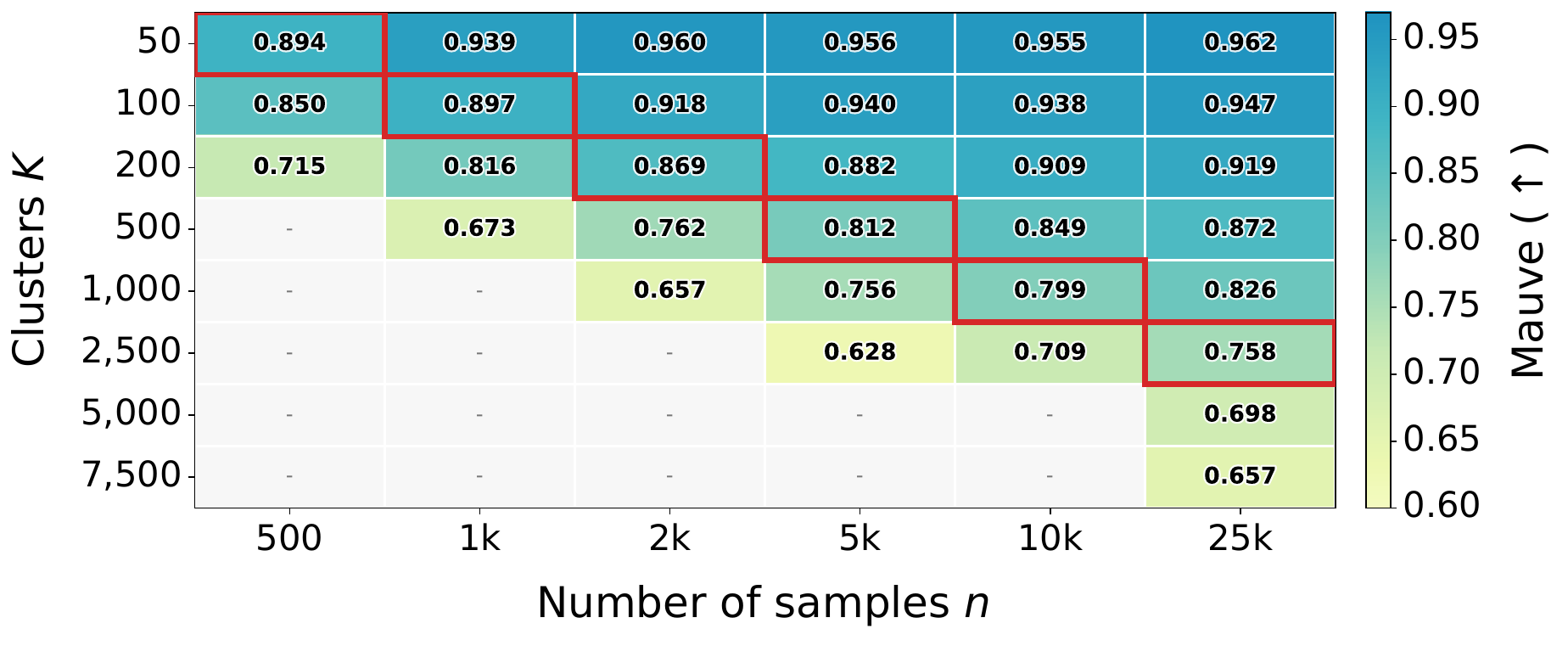}
    \vspace{-8mm}
    \caption{\small{For a fixed diffusion language model and dataset, the reported MAUVE score depends on the number of evaluation samples $n$ and clusters $K$. In particular, for the standard choice $K=n/10$ (red boxes), even after 5k samples it is not stable. Generation details are given in Appendix~\ref{sec:LM1B MAUVE Experiment}.}}
    \label{fig:MAUVE scaling}
    \vspace{-2mm}
\end{wrapfigure}
 MAUVE \citep{pillutla2021MAUVE,pillutla2023MAUVE} is a sample-based score for comparing generated and reference text distributions. It embeds each text by a single latent representation, jointly quantizes both samples using $k$-means, and represents them by the resulting empirical cluster frequencies. MAUVE then summarizes the trade-off between
$\operatorname{KL}(\tilde p_{\mathrm{gen}}\|\tilde r_\lambda)$ and
$\operatorname{KL}(\tilde p_{\mathrm{data}}\|\tilde r_\lambda)$,
where
$\tilde r_\lambda=\lambda\tilde p_{\mathrm{data}}+
(1-\lambda)\tilde p_{\mathrm{gen}}$,
by the area under their so-called divergence frontier, see Appendix~\ref{app:MAUVE}
for a detailed algorithmic description.

\begin{wrapfigure}[9]{r}{0.30\textwidth}
    \vspace{-3em}
    \centering
    \includegraphics[width=\linewidth]{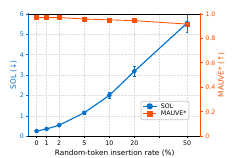}
          \vspace{-5mm}
    \caption{\small{Replacing random tokens in the sequence beginning is not detected by MAUVE, but by SOL.}}
    \label{fig:MAUVE-random-token-begining}
\end{wrapfigure} 

The method has two drawbacks relevant to our goal. First, MAUVE is sensitive to the finite-sample quantization and can be unstable with respect to both the number of samples and clusters, see Figure~\ref{fig:MAUVE scaling}. Second, its behavior depends strongly on the chosen text representation. 
In particular, MAUVE based on the last hidden state of a causal transformer is insensitive to perturbations at the beginning or middle of a sequence \citep{he-etal-2023-blind}, which we reproduce in Figure~\ref{fig:MAUVE-random-token-begining}. Mean pooling alleviates this particular problem, but the resulting score saturates at low corruption levels, see Appendix~\ref{subsec:MAUVE with Mean Hidden Representation}.
\medskip

\section{Evaluation of Language Models by  Wasserstein Distances}
We represent each sequence by the empirical measure of its hidden states. This
turns a text distribution into a distribution over measures in $\mathbb R^d$.
Direct optimal transport between these distributions leads to an expensive
Wasserstein-on-Wasserstein problem \citep{piening2026slicing}. We therefore
introduce the sliced Wasserstein construction underlying double slicing and
then define SOL in Section~\ref{subsec: our method}, where we show that it
satisfies the requirements (1)--(5) from the introduction.

\subsection{Lifting a Sequence to a Probability Measure}
\label{subsec:lifting_token_measures}
A text distribution $p_{\mathrm{data}}\in\PV$ is a distribution over token
sequences. The canonical lift
$\mathbf x\mapsto\delta_{\mathbf x}$ would represent each sequence as a Dirac
measure, but $\mathcal V^{\le L}$ has no natural semantic ground metric.
Instead, we propose to use the learned causal transformer $\mathbf H^\theta$, and introduce the associated \emph{token-cloud distribution} by 
\begin{align}\label{eq: PC embedding}
    T^\theta :  \mathcal{V}^{\le L} \rightarrow \mathcal{P}_2(\mathbb R^{d}), \qquad \mathbf x \mapsto \frac{1}{|\mathbf x|} \sum_{l=1}^{|\mathbf x|} \delta_{h^\theta_{\mathbf x \le l}}.
\end{align}
This allows us to push forward
$p_{\mathrm{data}}\in\mathcal P(\mathcal V^{\le L})$
to $T^\theta_\#p_{\mathrm{data}}\in\PPd$. This representation has two advantages. First, it compares sequences $\mathbf x_1$ and $\mathbf x_2$ of different lengths, $|\mathbf x_1|\neq|\mathbf x_2|$. Second, it allows reordered sentences such as `\textit{cat on a mat}' and `\textit{mat under a cat}' to remain close without imposing an artificial ordering. As shown in Section~\ref{Sec:proofs}, the map remains injective due to the injectivity of the last hidden-state representation, where token representations encode their context. For a bidirectional encoder, we replace $h^\theta_{\mathbf x\leq l}$ by $h^{\theta,\mathrm{bi}}_{\mathbf x,l}$.

\subsection{Double-Sliced Wasserstein Distance}
We write $\mathcal P(X)$ for the Borel probability measures on a metric space
$(X,\mathrm{dist})$ and $\mathcal P_2(X)$ for those with finite second moment.
Below, $X$ is either $\mathbb R^d$ or the Hilbert space  of (equivalence classes of) quadratic integrable functions  $L_2([0,1])$.

\textbf{Wasserstein distance.}
The \textit{Wasserstein-2 distance} $\text{W}$ on $\mathcal P_2(X)$ is defined by 
\begin{align*}
    \text{W}^2(\mu, \nu) := \inf _{\pi \in \Pi(\mu,\nu)} \int_{X\times X } \text{dist}^2 (x,y) \d \pi (x,y), \quad \forall \mu,\nu \in \mathcal P_2(X)
\end{align*}
where $\Pi(\mu,\nu) \subset \mathcal{P}(X \times X)$ denotes set of  measure with marginals $\mu$ and $\nu$. For measures on the real line $\mu \in \mathcal  P_2(\R)$, 
using the isometric embedding 
\begin{equation} \label{eq:quqntile}
Q: \mathcal P_2(\R) \to  L_2([0,1]), \qquad  \mu \mapsto Q_\mu,
\end{equation}
which maps a measure to its quantile function, the Wasserstein distance becomes
\begin{align}\label{eq:1D WD}
   \text{W}^2(\mu,\nu  ) = \int_0^1 |Q_\mu(s) - Q_\nu(s) |^2 \d s, \qquad \mu,\nu \in \mathcal P_2(\R).
\end{align}
In one dimension, for empirical measures $\mu =  \frac1n \sum_{i=1}^n \delta_{x_i}$, $x_i \in \mathbb R$,
the computation of $\text{W}$ reduces to the sorting of real numbers, which can be done in $O(n \log n)$ operations.

\textbf{Sliced Wasserstein distance.}
The idea behind  sliced Wasserstein distances arises from the simple access to the problem in 1D.
It is based on a  \emph{slicing operator} and its push-forward
\begin{align}    \label{eq:projection}
    S_\xi
    &\colon
    \R^d \to \R
    ,\quad
    x \mapsto \langle \xi, x \rangle, \qquad
   \mathbf{e}_{\xi} 
    \colon
    \mathcal P_2(\R^d) \to \mathcal P_2(\R)
    ,\qquad
    \mu \mapsto S_{\xi,\#} \, \mu,
\end{align}
where 
$\xi \in
    \mathbb S^{d{-}1} := \{x \in \R^d \mid \lVert x \rVert = 1\}$.
The \emph{sliced Wasserstein-2 distance} reads as
\begin{equation}   
    \label{eq:sliced_wasserstein}
    \text{SW}^2(\mu,\nu) 
    :=
    \int_{\mathbb S^{d{-}1}}
    \text{W}^2(\mathbf{e}_{\xi} \, \mu,\mathbf{e}_{\xi} \, \nu) 
    \d \mathcal U(\mathbb S^{d-1})(\xi)  , 
\end{equation}
where we integrate with respect to the uniform distribution $\mathcal U(\mathbb S^{d-1})$ on $\mathbb S^{d-1}$ and the Wasserstein distance on $\R$ is given by \eqref{eq:1D WD}.
Similar to the Wasserstein distance,
$\text{SW}$ metrizes the weak convergence, see 
\citep{bonnotte2013unidimensional, nadjahi2020statistical}.
The spherical integral is usually approximated
by Monte Carlo schemes, see \citep{bonneel2015slicedbarycenters, nguyen2024quasi, hertrich2025qmc_slicing}.

\textbf{Double-Sliced Wasserstein distance.}
Having lifted the data distribution from $p_{\text{data}}\in \PV$ to $T^{\theta}_\# p_{\text{data}} \in \PPd$, we follow \citep{piening2026slicing} to construct the double-sliced Wasserstein distance for meta-measures $\bm{\mu} \in \PPd$.
Since $\mathcal P_2(\R^d)$ is not a Euclidean space, but, together with the Wasserstein distance $\text{W}$,  an infinite-dimensional metric space, we have to define an appropriate slicing procedure. 

\emph{First}, using the slicing operation $S_\xi$ from \eqref{eq:projection}, 
we  map measures in $\Pd$ to measures in $\Po$ using  its push-forward operator 
$\mathbf{e}_{\xi}: \mathcal P_2(\R^d) \to \mathcal P_2(\R)$.

\emph{Second}, we  apply the isometry $Q: \mathcal P_2(\R) \to L_2([0,1])$ in \eqref{eq:quqntile} to map
$$
Q \circ \mathbf{e}_{\xi}: \mathcal P_2(\R^d) \to L_2([0,1]), \quad 
(Q \circ \mathbf{e}_{\xi})_\#: \mathcal P_2(\mathcal P_2(\R^d)) \to \mathcal P_2( L_2([0,1])).
$$

\emph{Third}, we define a second slicing on  $\mathcal P_2 (L_2[0,1])$. 
Since this is a space of probability measures on the infinite dimensional Hilbert space $L_2([0,1])$, such slicing can no longer rely on the uniform distributions on $\mathbb S^{d-1}$.
Instead, we slice functions on $L_2([0,1])$ with functions
$g \in L_2([0,1])$ by
\begin{align}
\mathcal S_g&: L_2([0,1]) \to \R, \qquad f \mapsto \langle f,g \rangle_{L_2},\\
\mathcal S_{g,\#}&: \mathcal P_2(L_2([0,1])) \to \mathcal P_2(\R), 
\qquad \eta \mapsto \mathcal S_{g,\#} \eta.
\end{align}
Then, instead of the uniform measure on the sphere, we choose a Gaussian reference measure $\gamma \in \mathcal P_2 (L_2([0,1]))$, see \cite{bogachev1998gaussian},  and define the
\textit{sliced  Wasserstein distance} on $\mathcal P_2(L_2([0,1]))$
by 

\begin{align*}
        \text{SW}^2 (\eta, \zeta \mid \gamma ) 
        := \int_{L_2} \text{W}^2 ( \mathcal S_{g,\#} \eta, \mathcal S_{g,\#} \zeta) \d \gamma(g).
    \end{align*}
To realize this sliced distance, i.e., the integral, we use the fact that there exists 
a one-to-one correspondence between Gaussian measures on $L_2([0,1])$ and Gaussian processes, see \cite[Thm 2]{rajput1972gaussian}. For details see \citep{piening2026slicing}.

\emph{Fourth}, to combine the three steps, we introduce, 
for \(\xi\in\mathbb S^{d-1}\) and \(g\in L_2([0,1])\),
the operator
\begin{align}
    M_{\xi,g}&:  \mathcal P_2(\R^d) \to \R,
    \qquad
   \mu
    \mapsto
    \left\langle
        Q_{\bf e_\xi  \mu}, g
    \right\rangle_{L_2},
    \\
    (M_{\xi,g})_\#&:  \mathcal P_2(  \mathcal P_2(\R^d) ) \to  \mathcal P_2(\R),
    \qquad
   \bm{\mu}
    \mapsto
    (M_{\xi,g})_\# \bm{\mu}
\end{align}
Now we can define the \textit{double sliced Wasserstein distance} on meta-measures  by
\begin{align}\label{eq:DSW} 
    \text{DSW}^2(\bm\mu,\bm \nu\mid \gamma) 
    :&= \int_{\mathbb S^{d-1}} {\text{SW}}^2 (( Q \circ \mathbf e_{\xi})_\#\bm\mu,  ( Q \circ \mathbf e_{\xi})_\#\bm\nu \mid \gamma ) \d \mathcal U(\mathbb S^{d-1})(\xi)\\
    &= \mathbb E_{\xi \sim \mathcal U(\mathbb S^{d-1}), g \sim \gamma}\left[\text{W}^2((M_{\xi,g})_\#\bm \mu,(M_{\xi,g})_\#\bm \nu) \right].
\end{align}
For empirical measures over empirical measures, $\mathrm{DSW}$ is a metric, see~\cite{piening2026slicing}. 
\begin{proposition}
\label{prop:dsw-sample-complexity}
Let the assumptions of
Propositions~\ref{prop:dsw-projection-approximation} and
\ref{prop:dsw-measure-approximation} hold, and assume in addition that
the samples are independent of the slices. Then
\begin{align}
    \mathbb E\left|
        \mathrm{DSW}_{L_\xi,L_g}(\bm\mu_n,\bm\nu_n\mid\gamma)
        -\mathrm{DSW}(\bm\mu,\bm\nu\mid\gamma)
    \right|
    \leq
    \frac{\sqrt8\,R\sqrt{m_2(\gamma)}}{n^{1/4}}
    +2R\left(
        \frac{m_2(\gamma)^2}{L_\xi}
        +\frac{m_4(\gamma)}{L_g}
    \right)^{1/4}.
    \label{eq:dsw-finite-slice-total}
\end{align}
\end{proposition}
The proof of the sample complexity is given in Appendix \ref{app:proof-dsw-sample-complexity}.

 \subsection{SOL - Sliced Optimal Transport for Language}\label{subsec: our method}
We propose to apply the \emph{double sliced Wasserstein distance for text comparison} after applying a reference causal transformer $\mathbf H^\theta$
and the resulting token-cloud distributions \eqref{eq: PC embedding} of the data distribution $p_{\text{data}}$ and the generated distribution $p_{\text{gen}}$, resulting in
\begin{align*}
    \text{SOL}_\theta(p_{\text{data}}, p_{\text{gen}}):= \sqrt{d} \text{ DSW}(T^{\theta}_\# p_{\text{data}}, T^{\theta}_\# p_{\text{gen}} \mid \gamma) .
\end{align*} 

{\color{blue}} SOL is scaled by the dimension of the hidden dimension.
Based on the injectivity of the last hidden-token representation in Remark \ref{rem:inj}, we obtain the following proposition, where the proof is given in Appendix~\ref{Sec:proofs}. 

%\begin{proposition}\label{prop}
\begin{restatable}{proposition}{faithfulness}
\label{prop}
% For discrete distributions $p,q \in \mathcal P (\mathcal V^{\le L})$, we have that 
% ${ \rm{SOL}}_\theta(p, q) = 0$ if and only if $p=q$.
Assume that \(\mathbf h^\theta\) is injective and $\gamma \in P_2(L_2(0,1))$ has full support. For empirical distributions $p,q\in\mathcal P(\mathcal V^{\le L})$
with uniform %sample 
weights, we have
$\operatorname{SOL}_\theta(p,q)=0$ if and only if $p=q$.
%\end{proposition}
\end{restatable}

%%%%%%%%%%%%%%%%%%%%%%%%%%%%%%%%%%%%%%%%%%%%%%%%%%%%%%%%%%%%%%%%%%%%%%%%%%%%%%%%%%%%%%%%%%%%%%%%%%%%%%%%%%%%%%%%%%

\section{Numerical Experiments}

We evaluate SOL on small- to medium-scale language models, with a particular focus on diffusion-based generation. Following \cite{pillutla2021MAUVE}, we first test whether SOL reflects expected changes in generation quality, including sequence length, decoding strategy, and model size. We then study the numerical stability and sample efficiency of the estimator. Finally, we evaluate SOL on failure cases of existing metrics, in particular MAUVE, and use it to compare autoregressive, discrete-diffusion, and continuous-diffusion language models on OpenWebText. A detailed comparison with MAUVE is given in Section~\ref{sec:Comparing MAUVE with SOL}.

\textbf{Setup.} We are mostly interested in unconditional text generation and text continuations. As generative models can only approximate the distributions $p_{\text{data}}$ that they have seen in training, for the reference distribution, a holdout set of the data distribution is used. 
The preprocessing of the reference distribution follows the one during training. 
The SOL metric is approximated with a finite number of inner ($L_g$), outer slices ($L_\xi$) and 
a Gaussian process to sample $g$, see Appendix~\ref{sec:dsw-numerical}.
By default, we use $L_\xi=L_g=1024$ and $n=5{,}000$ samples. For Dream, we increase $L_\xi=L_g$ to $8192$ to reduce estimator variance in its higher-dimensional hidden space ($d=3584$, compared with $d=1280$ for GPT-2 Large).

\textbf{Embedding Models} 
We evaluate different embedding backbones in $T^\theta$ for \eqref{eq: PC embedding}. Unless stated otherwise, we use GPT-2 Large as backbone embedding, matching the standard MAUVE setup and enabling direct comparison. We also evaluate more recent models such as OLMo2-1B \citep{olmo20252olmo2furious} and Qwen3-0.6B Base \citep{yang2025qwen3technicalreport}. Moreover, we consider  Dream \citep{ye2025Dream7bdiffusionlarge} which uses a bidirectional transformer. 
For further comparisons with 
contrastive embedding models 
 see Appendix \ref{subsec: Text Embedding Models}

%-----------------------------------

\subsection{Generative Perplexity and Entropy are Blind to Mode Collapse} 
\label{subsec:mode-collapse-blindness}
\begin{wrapfigure}[9]{r}{0.25\textwidth}
\vspace{-1.8em}
    \centering
    \includegraphics[width=\linewidth]{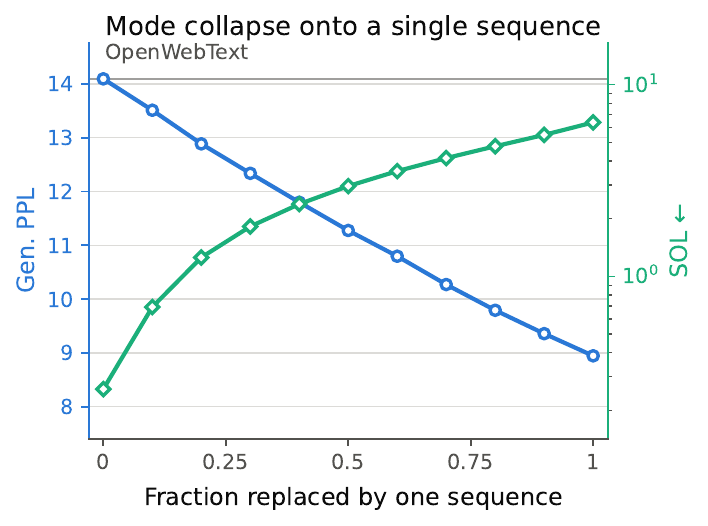}
    \vspace{-7mm}
    \caption{\small{Mode collapse is invisible to  Gen.PPL, while SOL rises sharply, see Subsection~\ref{subsec:mode-collapse-blindness}.}}
    \label{fig:mode-collapse}
          \vspace{-5mm}
\end{wrapfigure}

Generative perplexity (Gen.\ PPL) and average token entropy can fail to detect
mode collapse. We split OpenWebText into two disjoint sets of $3000$ sequences
of $1024$ tokens, using one as reference and progressively replacing samples
in the other with copies of a single sequence. Throughout this collapse,
token entropy remains nearly constant ($5.437 \pm 0.001$), while Gen.\ PPL
improves from $14.1$ to $8.9$. Thus, even when all samples are identical,
these metrics suggest improved quality without reduced diversity.
In contrast, SOL rises sharply 
(Figure~\ref{fig:mode-collapse}).

\subsection{Does the SOL behave as expected?} \label{subsec: Does behave as expected?} 
\begin{figure}[t]
    \centering
    \includegraphics[width=0.32\linewidth]{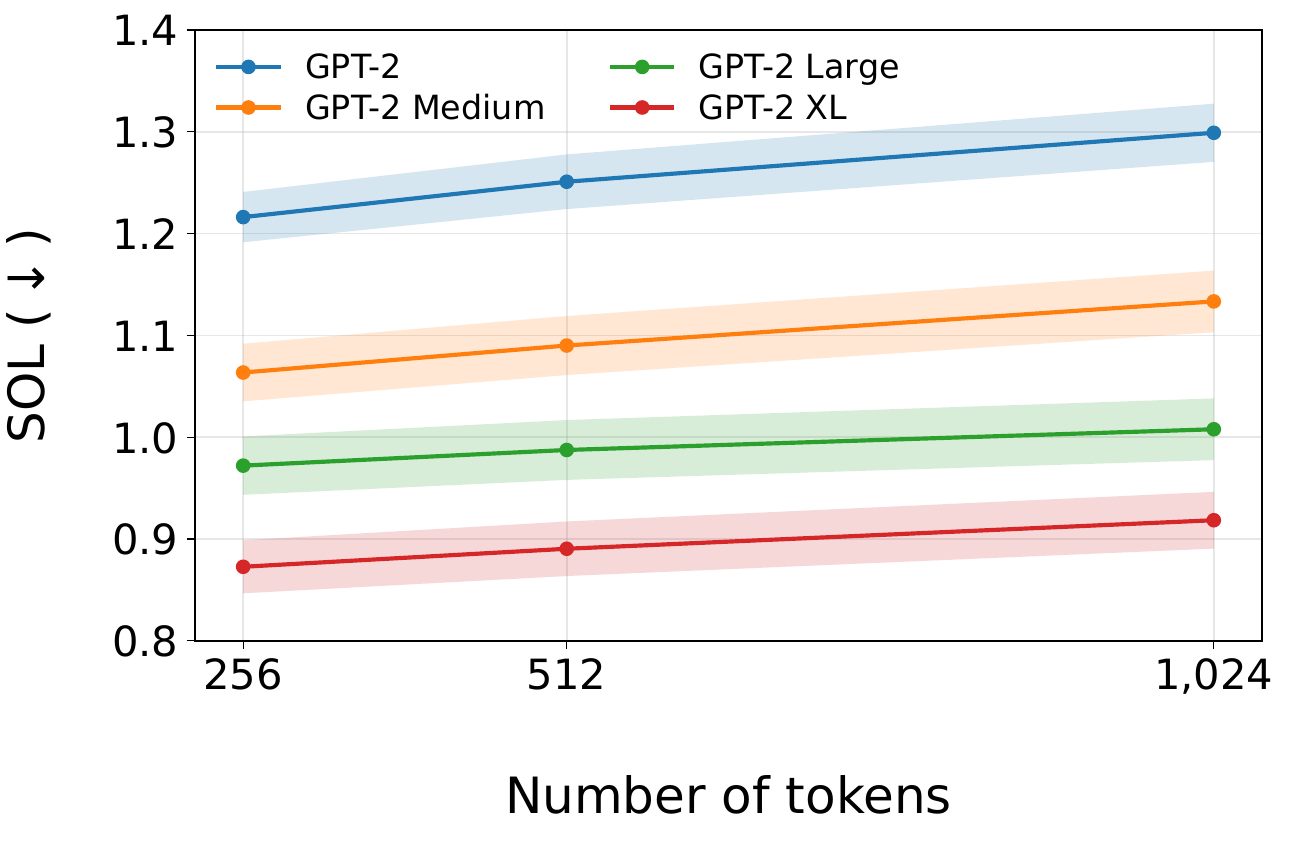}
    \hfill
          \includegraphics[width=0.32\linewidth]{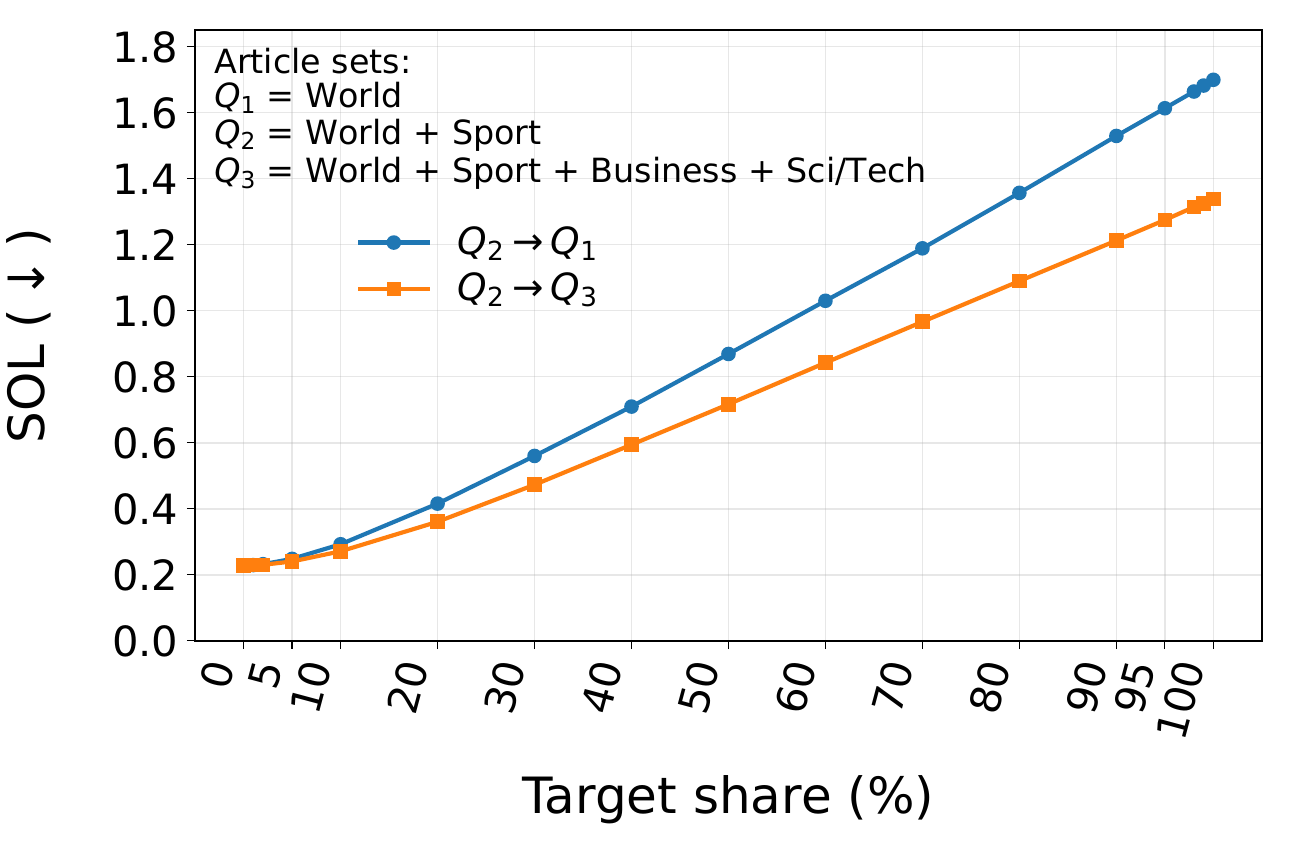}
          \hfill
     \includegraphics[width=0.32\linewidth]{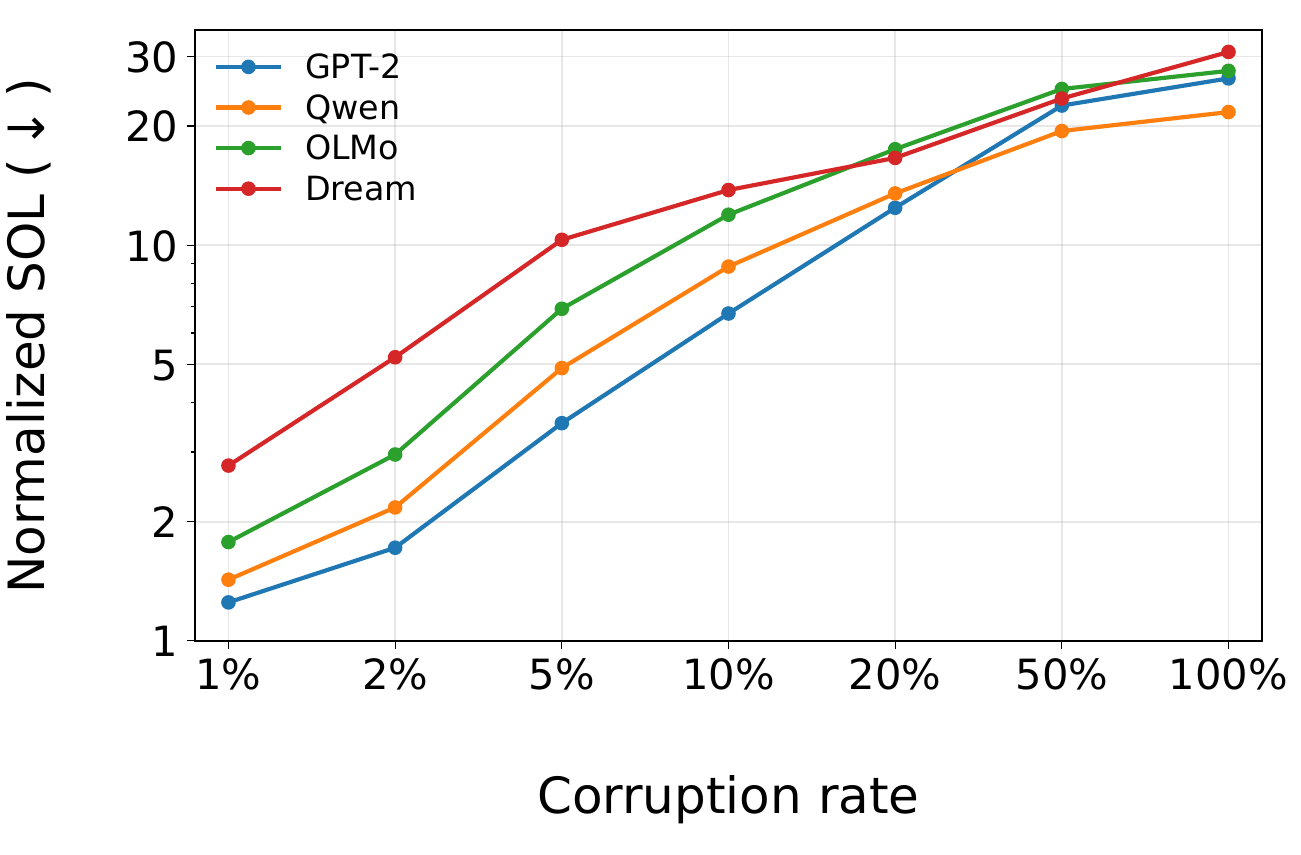}
     \vspace{-3mm}
    \caption{\small{Left: \textit{increasing the sequence length} diminishes the generation quality as expected. Middle:  behavior under \textit{mode-collapse}, i.e. going from $Q_2$ to $Q_1$ (blue), and if we introduce \textit{additional modes}, i.e. going from $Q_2$ to $Q_3$ (orange). Right:
    SOL under \textit{word shuffling} for different embeddings.
    }}
    \label{fig:gpt2_increasing_seq_length}
\end{figure}
To validate SOL, we first consider settings where the expected behavior is known.

\textbf{Generation length and model size.}
Increasing the \textit{generation length} of GPT-2 degrades the generated text, and SOL reflects this deterioration, see Figure~\ref{fig:gpt2_increasing_seq_length} (left). The same experiment also recovers the expected ordering across model sizes, from GPT-2 XL to GPT-2 Small.

\textbf{Mode collapse and support expansion.}
Following \citet{bronnec2024exploringprecisionrecallassess}, we evaluate SOL under \textit{mode collapse} and \textit{support expansion} on AG News \citep{NIPS2015_250cf8b5}. Starting from a mixture of WORLD and SPORTS articles ($Q_2$), we progressively remove SPORTS articles ($Q_1$) or add articles from two additional topics ($Q_3$). Both changes move the candidate distribution away from the original mixture, and SOL increases accordingly in both directions, see Figure~\ref{fig:gpt2_increasing_seq_length} (middle).

\textbf{Structural corruption.}
We progressively \textit{shuffle} tokens or \textit{pollute} the text with a fixed token. In both cases, stronger corruption leads to a larger SOL value, see Figure~\ref{fig:gpt2_increasing_seq_length} (right) and Figure~\ref{fig:text-embedding}. Further corruption experiments are given in Appendix~\ref{sec:encoder_robustness}. Appendix~\ref{subsec:gpt2xl-samplers} additionally compares greedy, ancestral, and nucleus sampling in conditional and unconditional generation, while Section~\ref{sec:Sensitivity to Sampling Temperature} shows that reducing token entropy alone does not improve SOL.

\textbf{Sampling behavior.}
SOL also recovers the expected dependence on sampling temperature: across the
evaluated encoders, moderate temperatures around $0.9$--$1.0$ give the smallest
distance to the data distribution, while both lower and higher temperatures
increase SOL; see Section~\ref{sec:Sensitivity to Sampling Temperature}.
Appendix~\ref{subsec:gpt2xl-samplers} further compares greedy, ancestral, and
nucleus sampling in conditional and unconditional generation.

\textbf{Positional sensitivity.}
Causal transformers are known to exhibit positional biases
\citep{he-etal-2023-blind,springer2025repetition}, so we perturb different
regions of the sequence in Appendix~\ref{sec: Different Pollution Location}
and Appendix~\ref{sec:Alternative to Bidirectional Attention}. Bidirectional
attention produces a substantially more balanced response across token positions
than the causal encoders. This positional balance is one of the main reasons we
use Dream as the reference encoder in our OpenWebText evaluation.

\subsection{Correlation with Human Judges}

\begin{table*}[t]
  \vspace{-0.45cm}
  \small
      \centering
      \begin{tabular}{lccc|ccc}
      \toprule
      & \multicolumn{3}{c}{MAUVE}
      & \multicolumn{3}{c}{SOL} \\
      \cmidrule(lr){2-4}
      \cmidrule(lr){5-7}
      Criterion $\uparrow$
      & GPT-2 Large & OLMo2 1B & Dream 7B
      & GPT-2 Large & OLMo2 1B & Dream 7B \\
      \midrule
      Human-like
      & 0.857 & 0.881 & 0.881
      & 0.881 & 0.929 & \textbf{0.952} \\
      Interesting
      & 0.714 & 0.762 & 0.762
      & 0.762 & \textbf{0.786} & 0.762 \\
      Sensible
      & 0.762 & 0.786 & 0.786
      & 0.786 & 0.881 & \textbf{0.905} \\
      \bottomrule
      \end{tabular}
      \caption{\small{
      Human judgement comparison at sequence length 1024 for different embedding models.
      Worst-case Spearman rank correlation $\rho_{\min}$ between
      metric-induced model rankings and the Bradley--Terry (BT) human ranking.
      Results are shown for MAUVE and SOL for a full setup see Section~\ref{sec:Human Judges correlation}.
      }}
      \label{tab:encoder_human_corr_1024}
      \vspace{-5mm}
  \end{table*}

We follow the human-evaluation setup of MAUVE \citep{pillutla2021MAUVE}. 
The study compares eight GPT-2 model--decoding configurations using pairwise human judgments of whether continuations are \emph{human-like}, \emph{interesting}, and \emph{sensible}. 
For each criterion, the pairwise preferences are converted into system scores by fitting a Bradley--Terry model \citep{19ff28b9-64f9-3656-ba40-08326a05748e}. 
We reuse these human judgments and recompute the automatic metrics on fresh samples from the same eight configurations. 
We report the worst-case Spearman rank correlation $\rho_{\text{min}}$ of SOL  with the Bradley--Terry scores for different embeddings as well as MAUVE. 
Full experimental details are given in Appendix~\ref{sec:Human Judges correlation}.

The comparison in Table~\ref{tab:encoder_human_corr_1024} serves two purposes. First, we compare SOL and MAUVE using the same GPT-2 Large representation. Under this controlled setting, SOL achieves slightly higher correlation with the human ranking than MAUVE. Second, we investigate whether stronger embedding models further improve SOL. Replacing GPT-2 Large with more recent models leads to substantially higher correlations, indicating that the choice of embedding model has a strong impact on SOL's agreement with human judgments. Because the original human study displayed at most 256 tokens, we additionally evaluate a sequence length of $256$ in Table~\ref{tab:encoder_human_corr_256},
where SOL again outperforms MAUVE.

\subsection{Approximation in SOL} \label{subsec:Approximation}
\begin{wrapfigure}[18]{r}{0.32\textwidth}
    \centering
    \vspace{-1.5em}

    % ------------------------------------------------
    % First figure
    % ------------------------------------------------
    \includegraphics[width=\linewidth]
        {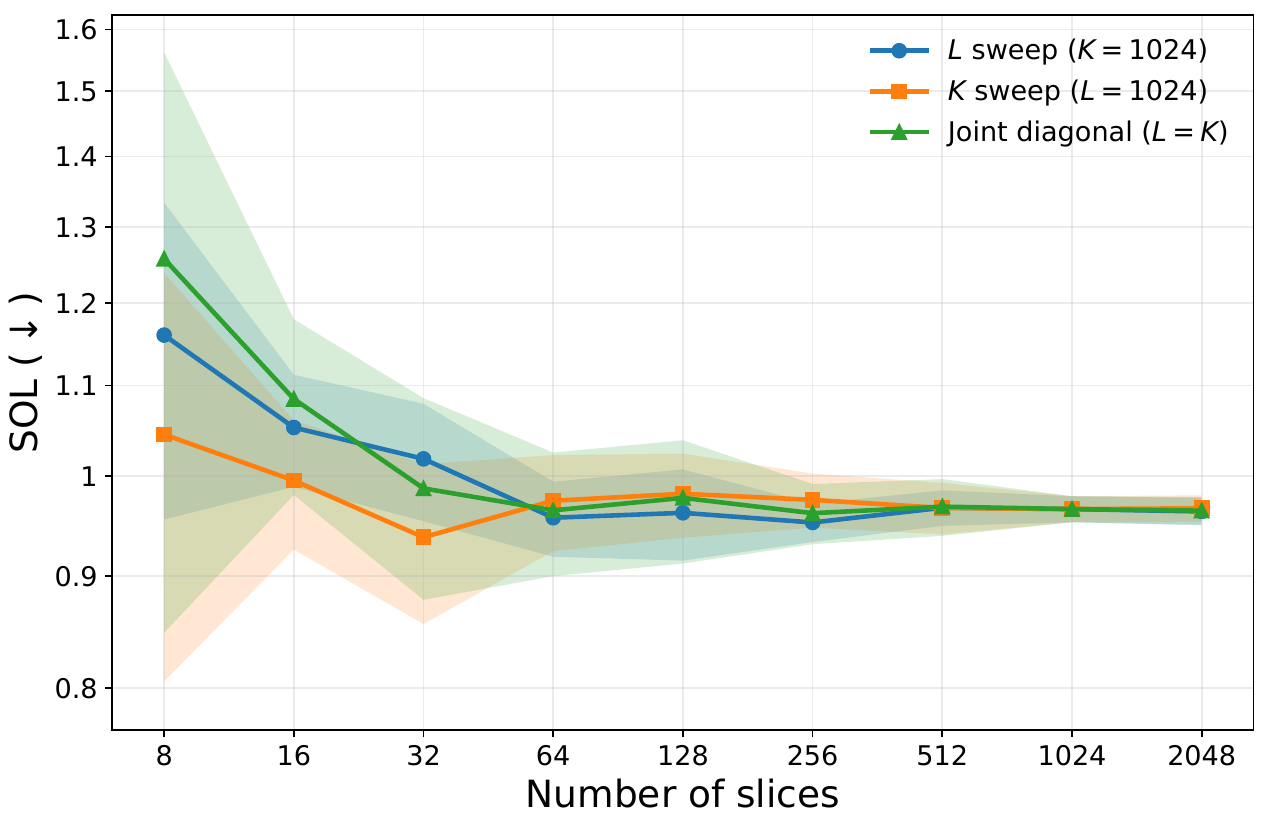}
    \vspace{-8mm}
    \captionof{figure}{\small
        Increasing the number of slices stabilizes SOL.
    }
    \label{fig:k-l-stabiltiy}

    \vspace{2mm}

    % ------------------------------------------------
    % Second figure
    % ------------------------------------------------
    \includegraphics[width=\linewidth]
        {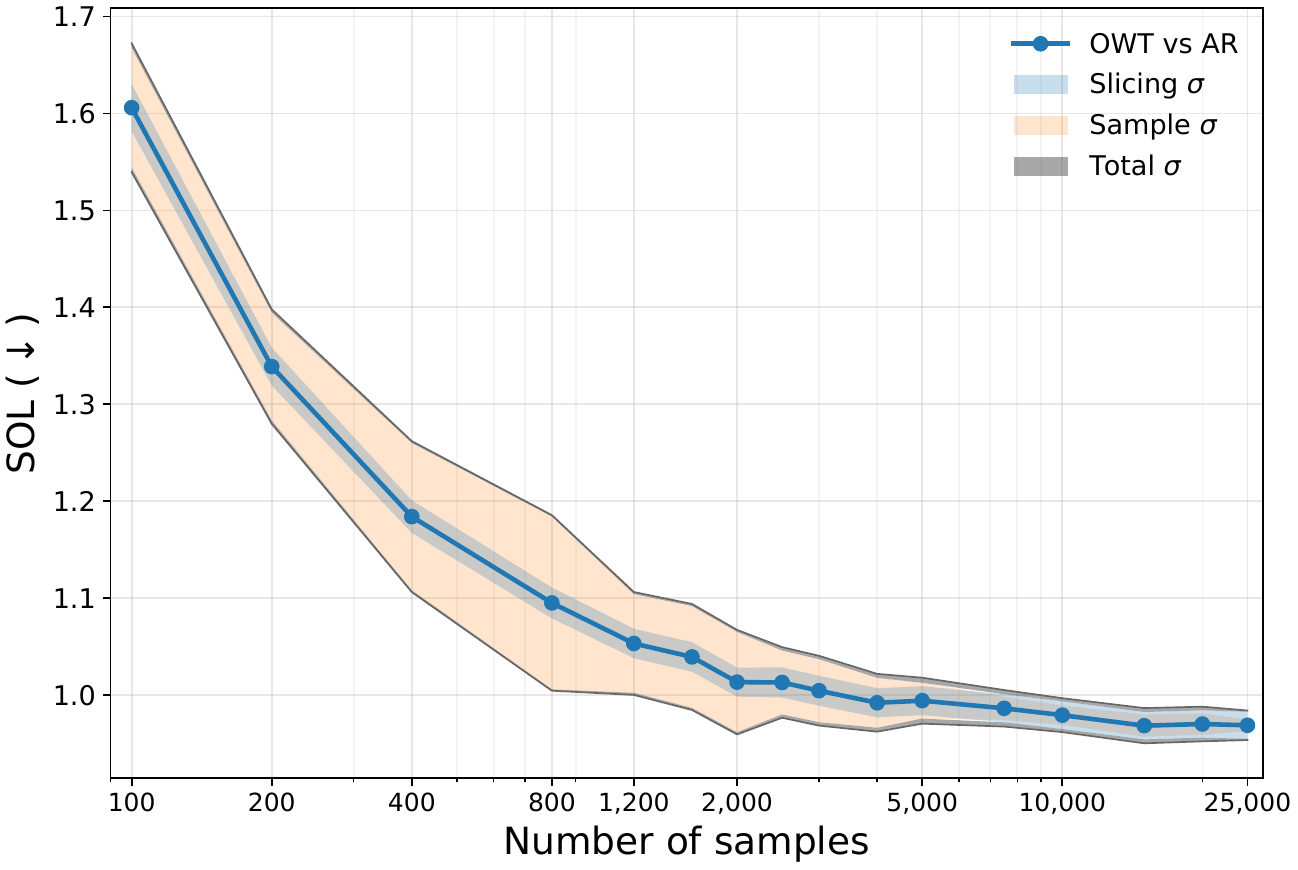}
    \vspace{-8mm}
    \captionof{figure}{\small
        SOL converges with sample size, with std over slices and samples.
    }
    \label{fig:sample-size}

    \vspace{-4mm}
\end{wrapfigure}
\textbf{Numerical stability.}
We study the approximation error with respect to the number of slices
$L_\xi,L_g$ and samples $n$. Using autoregressive OpenWebText generations
\citep{sahoo2024simple}, we vary $L_\xi$ and $L_g$ separately and jointly.
Figure~\ref{fig:k-l-stabiltiy} shows that the estimator is already stable around
$L_\xi=L_g=1024$. Increasing the number of corpus samples further reduces
variability (Figure~\ref{fig:sample-size}); up to roughly $5{,}000$ samples,
the variation from the corpus subsampling dominates the variation from the random slices.
In this experiment, SOL is substantially more stable with respect to sample
size than MAUVE (Figure~\ref{fig:MAUVE scaling}). The corresponding Dream
experiment is given in Section~\ref{subsec:Dream Number of samples}, where we use 8192 slices due to the larger hidden dimension.

\textbf{Computational Cost.}
The main computational cost comes from sample and hidden-state generation. On an NVIDIA GeForce RTX 5090, computing the hidden states takes $71$s for GPT-2 Large and $376$s for Dream. Computing the DSW for $5000$ samples takes about $1$s with $L_\xi=L_g=1024$ slices, see Table~\ref{tab:dsw-runtime-kl} for the runtime sweep.

\subsection{Diffusion Language Models on OWT}

Diffusion language models, particularly continuous ones, remain difficult to evaluate. ELBO-based perplexities are unavailable for most continuous models, are not comparable across noise processes, and do not distinguish different samplers of the same discrete model. Gen.\ PPL is also problematic, as it can be improved by reducing the sample entropy \citep{franca2026hacking}. 
Since SOL only requires samples, it applies to all models and samplers alike. We evaluate the discrete diffusion models MDLM \citep{sahoo2024simple}, DUO \citep{sahoo2025the}, LDDM \citep{jo2026loopholing}, BD3-LM \citep{arriola2025block}, ADLM \citep{rout2025anchored} and Soft-Masked \citep{hersche2026softmasked}, and the continuous models vMF and VP \citep{CKS2026}, LangFlow \citep{chen2026langflowcontinuousdiffusionrivals}, FLM \citep{lee2026flowmaplanguagemodels}, FMLM$^{*}$ \citep{yoo2026selfconditionedflowmaplanguage} and ELF \citep{hu2026elfembeddedlanguageflows}, together with an autoregressive baseline \citep{sahoo2024simple}.
ELF is trained on a corpus that is differently tokenized and compared with its own reference (E). Full experimental details are given in Appendix~\ref{sec:owt-setup}.

Table~\ref{tab:experiment-owt-baselines-reduced} shows that sampling strategy
has a substantial effect on distributional quality. ReMDM sampling strongly
improves the discrete Soft-Masked model, which attains a lower SOL than the
autoregressive baseline, while stochastic sampling also improves the continuous
models relative to their ODE samplers. Self-conditioning has a similarly large
effect in the controlled vMF comparison; see
Appendix~\ref{app:Extended Table}.

\begin{wraptable}[21]{r}{0.4\textwidth}
    \vspace{-1\baselineskip}
    \centering
    \small
    \setlength{\tabcolsep}{4pt}
    \renewcommand{\arraystretch}{0.94}

    \resizebox{\linewidth}{!}{%
    \begin{tabular}{@{}llrrr@{}}
      \toprule
      Model & Setting & $H$ & Gen. ppl. & SOL $\downarrow$ \\
      \midrule

      \emph{Real OWT}\textsuperscript{G}
        & {\scriptsize ---}
        & 5.449 & 14.9 & $2.066$ \\

      \rowcolor{gray!10}
      \emph{Real ELF OWT}\textsuperscript{E}
        & {\scriptsize ---}
        & 5.402 & 17.0 & $2.649$ \\

      \midrule
      \multicolumn{5}{l}{\textit{Autoregressive}} \\

      AR\textsuperscript{G}
        & {\scriptsize Temperature $1.0$}
        & 5.582 & 34.9 & $22.207$ \\

      \midrule
      \multicolumn{5}{l}{\textit{Discrete diffusion}} \\

      DUO-base\textsuperscript{G}
        & {\scriptsize Ancestral, $T=1024$}
        & 5.537 & 76.9 & $32.642$ \\

      DUO-distilled\textsuperscript{G}
        & {\scriptsize Ancestral, $T=32$}
        & 5.518 & 62.3 & $30.595$ \\

      MDLM\textsuperscript{G}
        & {\scriptsize $T=1000$}
        & 5.445 & 59.4 & $36.394$ \\

      LDDM-M\textsuperscript{G}
        & {\scriptsize $T=1024$}
        & 5.536 & 48.2 & $24.601$ \\

      Soft-Masked\textsuperscript{G}
        & {\scriptsize ReMDM, $T=1024$}
        & 5.341 & 16.7 & $\mathbf{18.966}$ \\

      ADLM-large\textsuperscript{G}
        & {\scriptsize ReMDM, $T=1024$}
        & 5.352 & 28.9 & $23.601$ \\

      BD3-LM, block 4\textsuperscript{B}
        & {\scriptsize First-hitting, $T=5000$}
        & 5.336 & 24.4 & $26.070$ \\

      \midrule
      \multicolumn{5}{l}{\textit{Continuous diffusion and flow}} \\

      vMF\textsuperscript{G}
        & {\scriptsize ODE, steps$=1024$}
        & 5.456 & 75.4 & $33.832$ \\

      vMF\textsuperscript{G}
        & {\scriptsize SDE, steps$=128$}
        & 5.348 & 43.6 & $30.375$ \\

      VP\textsuperscript{G}
        & {\scriptsize ODE, steps$=1024$}
        & 5.435 & 66.1 & $33.275$ \\

      VP\textsuperscript{G}
        & {\scriptsize SDE, steps$=128$}
        & 5.355 & 50.9 & $31.972$ \\

      LangFlow\textsuperscript{G}
        & {\scriptsize ODE, steps$=1024$}
        & 5.257 & 36.9 & $30.541$ \\

      \rowcolor{gray!10}
      ELF-B\textsuperscript{E}
        & {\scriptsize SDE, steps$=128$}
        & 5.123 & 21.9 & $32.878$ \\

      \rowcolor{gray!10}
      ELF-B\textsuperscript{E}
        & {\scriptsize ODE, steps$=1024$}
        & 5.090 & 24.0 & $33.320$ \\

      \midrule
      \multicolumn{5}{l}{\textit{Few-step models}} \\

      FMLM\textsuperscript{G}
        & {\scriptsize 2 steps}
        & 5.295 & 135.1 & $71.134$ \\

    FLM\textsuperscript{G}
        & {\scriptsize Euler, steps$=1024$}
        & 5.345 & 63.1 & $48.387$ \\

      ELF$^{*}$\textsuperscript{G}
        & {\scriptsize CDEQ, steps$=128$, CFG$=5$}
        & 5.384 & 60.8 & $35.950$ \\

      FMLM$^{*}$\textsuperscript{G}
        & {\scriptsize 2 steps, CFG$=5$}
        & 5.458 & 97.4 & $41.120$ \\
     
      FMLM$^{*}$\textsuperscript{G}
        & {\scriptsize 128 steps, CFG$=5$}
        & 5.456 & 79.0 & $38.482$ \\

      \bottomrule
    \end{tabular}%
    }

    \caption{\small{
OpenWebText baselines using 5,000 samples per distribution and a
1,024-token evaluation horizon. SOL uses Dream-7B as the reference encoder.
The complete benchmark is reported in
Table~\ref{tab:experiment-8d-owt-baselines}.
}}
    \label{tab:experiment-owt-baselines-reduced}
    \vspace{-1\baselineskip}
\end{wraptable}

Gen.\ PPL and SOL can nevertheless disagree. For example, LangFlow achieves a
much lower Gen.\ PPL than vMF with predictor--corrector sampling
($36.9$ versus $56.3$, see Table~\ref{tab:experiment-8d-owt-baselines}), while their SOL values are much closer. Thus,
improvements in Gen.\ PPL need not correspond to comparable improvements in
distributional fit. Results for three encoders, additional sampler settings,
and a diagnosis of MAUVE's last-token bias are given in
Appendix~\ref{sec: OWT comparision}.

\section{Conclusion}

Current evaluation methods are insufficient for comparing modern language models. We propose SOL, a sample-based distance between text distributions that represents each sequence by the empirical measure of its hidden states and compares the resulting distributions using the double sliced Wasserstein distance.
SOL is a metric for injective reference transformers and admits a stable estimator dominated by a single reference-model forward pass. Unlike Gen.\ PPL and entropy, SOL detects mode collapse, unlike MAUVE, it is less biased toward the end of the sequence and correlates better with human judgments.
Because SOL only requires samples, it compares autoregressive, discrete, and continuous diffusion models on equal footing. On OpenWebText, we find that better Gen.\ PPL does not necessarily imply a better fit to the data distribution. We recommend the bidirectional Dream backbone, since it weights all token positions similarly and achieves the strongest overall agreement with human judgments, together with $L_\xi = L_g = 8192$ slices and $5{,}000$ samples per distribution, at which the estimator stabilizes.

\paragraph{Acknowledgement.}
G.K. acknowledges funding by the
BMBF VIScreenPRO (ID: 100715327).
J. Ch. acknowledges funding of the DFG project STE 571/17-2
within the SPP 2298 “Theoretical Foundations of Deep Learning”. G.S. thanks CNRS for support within the CNRS Ambassador Fellowship.
% in the document

  %%%%%%%%%%%%%%%%%%%%%%%%%%%%%%%%%%%%%%%%%%%%%%%%%%%%%%%%%%%%%%%%%%%%%%%%%%%%%%%%%%%%%%%%%%%%%%%%%%%%%%%%%%%%
 \bibliography{iclr2027_conference}

@article{CKS2026,
title={Spherical Flows for Sampling Categorical Data},
author = {J Chemseddine and G Kornhardt and G Steidl},
journal = {arXiv preprint arXiv:2605.05629},
year = {2026}
}

@book{bogachev1998gaussian,
  title={Gaussian measures},
  author={Bogachev, Vladimir Igorevich},
  number={62},
  year={1998},
  publisher={American Mathematical Soc.}
}

@article{rajput1972gaussian,
  title={Gaussian processes and {G}aussian measures},
  author={Rajput, Balram S and Cambanis, Stamatis},
  journal={The Annals of Mathematical Statistics},
  pages={1944--1952},
  year={1972},
  publisher={JSTOR}
}

@inproceedings{bonet2025flowing,
title={Flowing datasets with {W}asserstein over {W}asserstein gradient flows},
author={Cl{\'e}ment Bonet and Christophe Vauthier and Anna Korba},
booktitle={Proceedings of ICML'25},
year={2025},
publisher={OpenReview.net}
}

@misc{pynadath2026generativefrontiersevaluationmatters,
      title={Generative Frontiers: Why Evaluation Matters for Diffusion Language Models}, 
      author={Patrick Pynadath and Jiaxin Shi and Ruqi Zhang},
      year={2026},
      eprint={2604.02718},
      archivePrefix={arXiv},
      primaryClass={cs.LG},
      url={https://arxiv.org/abs/2604.02718}, 
}

@misc{yang2025qwen3technicalreport,
      title={Qwen3 Technical Report}, 
      author={An Yang and Anfeng Li and Baosong Yang and Beichen Zhang and Binyuan Hui and Bo Zheng and Bowen Yu and Chang Gao and Chengen Huang and Chenxu Lv and Chujie Zheng and Dayiheng Liu and Fan Zhou and Fei Huang and Feng Hu and Hao Ge and Haoran Wei and Huan Lin and Jialong Tang and Jian Yang and Jianhong Tu and Jianwei Zhang and Jianxin Yang and Jiaxi Yang and Jing Zhou and Jingren Zhou and Junyang Lin and Kai Dang and Keqin Bao and Kexin Yang and Le Yu and Lianghao Deng and Mei Li and Mingfeng Xue and Mingze Li and Pei Zhang and Peng Wang and Qin Zhu and Rui Men and Ruize Gao and Shixuan Liu and Shuang Luo and Tianhao Li and Tianyi Tang and Wenbiao Yin and Xingzhang Ren and Xinyu Wang and Xinyu Zhang and Xuancheng Ren and Yang Fan and Yang Su and Yichang Zhang and Yinger Zhang and Yu Wan and Yuqiong Liu and Zekun Wang and Zeyu Cui and Zhenru Zhang and Zhipeng Zhou and Zihan Qiu},
      year={2025},
      eprint={2505.09388},
      archivePrefix={arXiv},
      primaryClass={cs.CL},
      url={https://arxiv.org/abs/2505.09388}, 
}

@inproceedings{NIPS2015_250cf8b5,
	author = {Zhang, Xiang and Zhao, Junbo and LeCun, Yann},
	booktitle = {Advances in Neural Information Processing Systems},
	editor = {Cortes, C. and Lawrence, N. and Lee, D. and Sugiyama, M. and Garnett, R.},
	pages = {},
	publisher = {Curran Associates, Inc.},
	title = {Character-level Convolutional Networks for Text Classification},
	url = {https://proceedings.neurips.cc/paper_files/paper/2015/file/250cf8b51c773f3f8dc8b4be867a9a02-Paper.pdf},
	volume = {28},
	year = {2015}
}

@inproceedings{nguyen2024quasi,
  title     = {Quasi‑{Monte C}arlo for 3D Sliced {W}asserstein},
  author    = {Khai Nguyen and Nicola Bariletto and Nhat Ho},
  booktitle = {Proceedings of the ICLR'24},
  year      = {2024},
  publisher={OpenReview.net}
}

@InProceedings{pmlr-v119-naeem20a,
  title = 	 {Reliable Fidelity and Diversity Metrics for Generative Models},
  author =       {Naeem, Muhammad Ferjad and Oh, Seong Joon and Uh, Youngjung and Choi, Yunjey and Yoo, Jaejun},
  booktitle = 	 {Proceedings of the 37th International Conference on Machine Learning},
  pages = 	 {7176--7185},
  year = 	 {2020},
  editor = 	 {III, Hal Daumé and Singh, Aarti},
  volume = 	 {119},
  series = 	 {Proceedings of Machine Learning Research},
  month = 	 {13--18 Jul},
  publisher =    {PMLR},
  url = 	 {https://proceedings.mlr.press/v119/naeem20a.html}
}

@inproceedings{hertrich2025qmc_slicing,
  author       = {Johannes Hertrich and Tim Jahn and Michael Quellmalz},
  title        = {Fast summation of radial kernels via {QMC} slicing},
  booktitle    = {Proceedings of the ICLR'25},
  year         = {2025},
  publisher={OpenReview.net}
}

@inproceedings{nadjahi2020statistical,
  title = {Statistical and topological properties of sliced probability divergences},
  author = {Nadjahi, Kimia and Durmus, Alain and Chizat, L{\'e}na{\"i}c and Kolouri, Soheil and Shahrampour, Shahin and Simsekli, Umut},
  booktitle = {Advances in Neural Information Processing Systems},
  volume = {33},
  pages = {15657--15669},
  year = {2020},
  publisher = {Curran Associates},
}

@phdthesis{bonnotte2013unidimensional,
  author       = {Nicolas Bonnotte},
  title        = {Unidimensional and Evolution Methods for Optimal Transportation},
  school       = {Université Paris Sud–Paris XI},
  year         = {2013},
  address      = {Orsay, France},
  note         = {PhD Thesis}
}

@inproceedings{nguyen2025sotdd,
  author    = {Nguyen, Khai and Nguyen, Hai and Pham, Tuan and Ho, Nhat},
  title     = {Lightspeed geometric dataset distance via sliced optimal transport},
  booktitle = {Proceedings of ICML'25},
  year      = {2025},
    publisher={OpenReview.net}
}

@article{bonneel2015slicedbarycenters,
  title={Sliced and {Radon W}asserstein barycenters of measures},
  author={Bonneel, Nicolas and Rabin, Julien and Peyr{\'e}, Gabriel and Pfister, Hanspeter},
  journal={Journal of Mathematical Imaging and Vision},
  volume={51},
  pages={22--45},
  year={2015},
  publisher={Springer},
}

@article{PDS2026,
title = {Generalized {W}asserstein Flow Matching: Transport Plans, Everywhere, All at Once},
author = {M Piening and R Duong and G Steidl},
journal = {arXiv preprint arXiv:2605.08424},
year = {2026}
}

@inproceedings{piening2026slicing,
title={Slicing {W}asserstein over {W}asserstein via Functional Optimal Transport},
author={Moritz Piening and Robert Beinert},
booktitle={The Fourteenth International Conference on Learning Representations},
year={2026},
url={https://openreview.net/forum?id=l3KtyVZde3}
}

@inproceedings{
nikolaou2026language,
title={Language Models are Injective and Hence Invertible},
author={Giorgos Nikolaou and Tommaso Mencattini and Donato Crisostomi and Andrea Santilli and Yannis Panagakis and Emanuele Rodol{\`a}},
booktitle={The Fourteenth International Conference on Learning Representations},
year={2026},
url={https://openreview.net/forum?id=0kHbD6ad07}
}

@inproceedings{
pillutla2021mauve,
title={{MAUVE}: Measuring the Gap Between Neural Text and Human Text using Divergence Frontiers},
author={Krishna Pillutla and Swabha Swayamdipta and Rowan Zellers and John Thickstun and Sean Welleck and Yejin Choi and Zaid Harchaoui},
booktitle={Advances in Neural Information Processing Systems},
editor={A. Beygelzimer and Y. Dauphin and P. Liang and J. Wortman Vaughan},
year={2021},
url={https://openreview.net/forum?id=Tqx7nJp7PR}
}

@inproceedings{he-etal-2023-blind,
    title = "On the Blind Spots of Model-Based Evaluation Metrics for Text Generation",
    author = "He, Tianxing  and
      Zhang, Jingyu  and
      Wang, Tianle  and
      Kumar, Sachin  and
      Cho, Kyunghyun  and
      Glass, James  and
      Tsvetkov, Yulia",
    editor = "Rogers, Anna  and
      Boyd-Graber, Jordan  and
      Okazaki, Naoaki",
    booktitle = "Proceedings of the 61st Annual Meeting of the Association for Computational Linguistics (Volume 1: Long Papers)",
    month = jul,
    year = "2023",
    address = "Toronto, Canada",
    publisher = "Association for Computational Linguistics",
    url = "https://aclanthology.org/2023.acl-long.674/",
    doi = "10.18653/v1/2023.acl-long.674",
    pages = "12067--12097"
}

@inproceedings{
sahoo2026scaling,
title={Scaling Beyond Masked Diffusion Language Models},
author={Subham Sekhar Sahoo and Jean-Marie Lemercier and Zhihan Yang and Justin Deschenaux and Jingyu Liu and John Thickstun and Ante Juki{\'c}},
booktitle={Forty-third International Conference on Machine Learning},
year={2026},
url={https://openreview.net/forum?id=xKJG1CNtQM}
}

@misc{bge-m3,
      title={BGE M3-Embedding: Multi-Lingual, Multi-Functionality, Multi-Granularity Text Embeddings Through Self-Knowledge Distillation}, 
      author={Jianlv Chen and Shitao Xiao and Peitian Zhang and Kun Luo and Defu Lian and Zheng Liu},
      year={2024},
      eprint={2402.03216},
      archivePrefix={arXiv},
      primaryClass={cs.CL}
}

@misc{li2023generaltextembeddingsmultistage,
      title={Towards General Text Embeddings with Multi-stage Contrastive Learning}, 
      author={Zehan Li and Xin Zhang and Yanzhao Zhang and Dingkun Long and Pengjun Xie and Meishan Zhang},
      year={2023},
      eprint={2308.03281},
      archivePrefix={arXiv},
      primaryClass={cs.CL},
      url={https://arxiv.org/abs/2308.03281}, 
}

@misc{wang2024textembeddingsweaklysupervisedcontrastive,
      title={Text Embeddings by Weakly-Supervised Contrastive Pre-training}, 
      author={Liang Wang and Nan Yang and Xiaolong Huang and Binxing Jiao and Linjun Yang and Daxin Jiang and Rangan Majumder and Furu Wei},
      year={2024},
      eprint={2212.03533},
      archivePrefix={arXiv},
      primaryClass={cs.CL},
      url={https://arxiv.org/abs/2212.03533}, 
}

@misc{olmo20252olmo2furious,
      title={2 OLMo 2 Furious}, 
      author={Team OLMo and Pete Walsh and Luca Soldaini and Dirk Groeneveld and Kyle Lo and Shane Arora and Akshita Bhagia and Yuling Gu and Shengyi Huang and Matt Jordan and Nathan Lambert and Dustin Schwenk and Oyvind Tafjord and Taira Anderson and David Atkinson and Faeze Brahman and Christopher Clark and Pradeep Dasigi and Nouha Dziri and Allyson Ettinger and Michal Guerquin and David Heineman and Hamish Ivison and Pang Wei Koh and Jiacheng Liu and Saumya Malik and William Merrill and Lester James V. Miranda and Jacob Morrison and Tyler Murray and Crystal Nam and Jake Poznanski and Valentina Pyatkin and Aman Rangapur and Michael Schmitz and Sam Skjonsberg and David Wadden and Christopher Wilhelm and Michael Wilson and Luke Zettlemoyer and Ali Farhadi and Noah A. Smith and Hannaneh Hajishirzi},
      year={2025},
      eprint={2501.00656},
      archivePrefix={arXiv},
      primaryClass={cs.CL},
      url={https://arxiv.org/abs/2501.00656}, 
}

@misc{openai2019gpt2output,
  author       = {{OpenAI}},
  title        = {{GPT-2 Output Dataset}},
  year         = {2019},
  howpublished = {\url{https://github.com/openai/gpt-2-output-dataset}},
  note         = {GitHub repository}
}

@misc{bronnec2024exploringprecisionrecallassess,
      title={Exploring Precision and Recall to assess the quality and diversity of LLMs}, 
      author={Florian Le Bronnec and Alexandre Verine and Benjamin Negrevergne and Yann Chevaleyre and Alexandre Allauzen},
      year={2024},
      eprint={2402.10693},
      archivePrefix={arXiv},
      primaryClass={cs.CL},
      url={https://arxiv.org/abs/2402.10693}, 
}

@inproceedings{zheng2025masked,
  title={Masked Diffusion Models are Secretly Time-Agnostic Masked Models and Exploit Inaccurate Categorical Sampling},
  author={Zheng, Kaiwen and Chen, Yongxin and Mao, Hanzi and Liu, Ming-Yu and Zhu, Jun and Zhang, Qinsheng},
  booktitle={Proceedings of the ICLR'25},
  year={2025},
  publisher={OpenReview.net}
}

@article{franca2026hacking,
  title={Hacking Generative Perplexity: Why Unconditional Text Evaluation Needs Distributional Metrics},
  author={Franca, Antonio and Tong, Alexander},
  journal={arXiv preprint arXiv:2606.08417},
  year={2026}
}

@inproceedings{xiang2021assessing,
  title={Assessing Dialogue Systems with Distribution Distances},
  author={Xiang, Jiannan and Liu, Yahui and Cai, Deng and Li, Huayang and Lian, Defu and Liu, Lemao},
  booktitle={Findings of the Association for Computational Linguistics: ACL-IJCNLP 2021},
  pages={2192--2198},
  year={2021},
  publisher={Association for Computational Linguistics}
}

@article{pillutla2023mauve,
  title={{MAUVE} Scores for Generative Models: Theory and Practice},
  author={Pillutla, Krishna and Liu, Lang and Thickstun, John and Welleck, Sean
  and Swayamdipta, Swabha and Zellers, Rowan and Oh, Sewoong and Choi, Yejin
  and Harchaoui, Zaid},
  journal={Journal of Machine Learning Research},
  volume={24},
  number={356},
  pages={1--92},
  year={2023}
}

@inproceedings{heusel2017gans,
  title={{GAN}s Trained by a Two Time-Scale Update Rule Converge to a Local Nash Equilibrium},
  author={Heusel, Martin and Ramsauer, Hubert and Unterthiner, Thomas and Nessler, Bernhard and Hochreiter, Sepp},
  booktitle={Advances in Neural Information Processing Systems},
  volume={30},
  year={2017},
  publisher={Curran Associates}
}

@article{berthet2026mind,
  title={{MIND}: {M}onge {I}nception {D}istance for Generative Models Evaluation},
  author={Berthet, Quentin and Wu, Yu-Han and Crepy, Clement and Elie, Romuald
  and Greff, Klaus and Sander, Michael Eli},
  journal={arXiv preprint arXiv:2605.06797},
  year={2026}
}

@inproceedings{hara2026mean,
  title={Why Mean Pooling Works: Quantifying Second-Order Collapse in Text Embeddings},
  author={Hara, Tomomasa and Kurita, Hiroto and Imaizumi, Masaaki and Inui, Kentaro and Yokoi, Sho},
  booktitle={Proceedings of the 64th Annual Meeting of the Association for Computational Linguistics},
  pages={47180--47201},
  year={2026},
  publisher={Association for Computational Linguistics}
}

@inproceedings{kusner2015word,
  title={From Word Embeddings to Document Distances},
  author={Kusner, Matt and Sun, Yu and Kolkin, Nicholas and Weinberger, Kilian},
  booktitle={Proceedings of the 32nd International Conference on Machine Learning},
  volume={37},
  pages={957--966},
  year={2015},
  publisher={PMLR}
}

@inproceedings{yurochkin2019hierarchical,
  title={Hierarchical Optimal Transport for Document Representation},
  author={Yurochkin, Mikhail and Claici, Sebastian and Chien, Edward
  and Mirzazadeh, Farzaneh and Solomon, Justin M.},
  booktitle={Advances in Neural Information Processing Systems},
  volume={32},
  year={2019},
  publisher={Curran Associates}
}

@article{bonet2025busemann,
  title={Busemann Functions in the {W}asserstein Space: Existence, Closed-Forms, and Applications to Slicing},
  author={Cl{\'e}ment Bonet and Elsa Cazelles and Lucas Drumetz and Nicolas Courty},
  booktitle={Proceedings of AISTATS'26},
  volume={300},
  pages={1972--1980},
  year={2026},
  publisher={PMLR}
}

@article{kreitmeier2011optimal,
  title={Optimal Vector Quantization in Terms of {W}asserstein Distance},
  author={Kreitmeier, Wolfgang},
  journal={Journal of Multivariate Analysis},
  volume={102},
  number={8},
  pages={1225--1239},
  year={2011},
  doi={10.1016/j.jmva.2011.04.005}
}

@article{radford2019language,
  title={Language Models are Unsupervised Multitask Learners},
  author={Radford, Alec and Wu, Jeff and Child, Rewon and Luan, David and Amodei, Dario and Sutskever, Ilya},
  year={2019}
}

@misc{ye2025dream7bdiffusionlarge,
      title={Dream 7B: Diffusion Large Language Models}, 
      author={Jiacheng Ye and Zhihui Xie and Lin Zheng and Jiahui Gao and Zirui Wu and Xin Jiang and Zhenguo Li and Lingpeng Kong},
      year={2025},
      eprint={2508.15487},
      archivePrefix={arXiv},
      primaryClass={cs.CL},
      url={https://arxiv.org/abs/2508.15487}, 
}

@InProceedings{pmlr-v267-geuter25a,
  title = 	 {Universal Neural Optimal Transport},
  author =       {Geuter, Jonathan and Kornhardt, Gregor and Tomasson, Ingimar and Laschos, Vaios},
  booktitle = 	 {Proceedings of the 42nd International Conference on Machine Learning},
  pages = 	 {19196--19232},
  year = 	 {2025},
  editor = 	 {Singh, Aarti and Fazel, Maryam and Hsu, Daniel and Lacoste-Julien, Simon and Berkenkamp, Felix and Maharaj, Tegan and Wagstaff, Kiri and Zhu, Jerry},
  volume = 	 {267},
  series = 	 {Proceedings of Machine Learning Research},
  month = 	 {13--19 Jul},
  publisher =    {PMLR},
  url = 	 {https://proceedings.mlr.press/v267/geuter25a.html}
}

@Article{Clerc2026,
author={Clerc, Antonin
and Quellmalz, Michael
and Piening, Moritz
and Flotho, Philipp
and Kornhardt, Gregor
and Steidl, Gabriele},
title={HOT-POT: Optimal Transport for Sparse Stereo Matching},
journal={Journal of Mathematical Imaging and Vision},
year={2026},
month={Aug},
day={05},
volume={68},
number={4},
pages={53},
issn={1573-7683},
doi={10.1007/s10851-026-01336-3},
url={https://doi.org/10.1007/s10851-026-01336-3}
}

@InProceedings{pmlr-v97-dukler19a,
  title = 	 {{W}asserstein of {W}asserstein Loss for Learning Generative Models},
  author =       {Dukler, Yonatan and Li, Wuchen and Lin, Alex and Montufar, Guido},
  booktitle = 	 {Proceedings of the 36th International Conference on Machine Learning},
  pages = 	 {1716--1725},
  year = 	 {2019},
  editor = 	 {Chaudhuri, Kamalika and Salakhutdinov, Ruslan},
  volume = 	 {97},
  series = 	 {Proceedings of Machine Learning Research},
  month = 	 {09--15 Jun},
  publisher =    {PMLR},
  url = 	 {https://proceedings.mlr.press/v97/dukler19a.html}
}

@article{piening2025slicinggmm,
  title={Slicing the {G}aussian Mixture {W}asserstein Distance},
  author={Piening, Moritz and Beinert, Robert},
  journal={Transactions on Machine Learning Research},
  year={2025},
  publisher={OpenReview.net}
}

@article{nguyen2026summarizing,
  title={Summarizing Nonparametric {B}ayesian Mixture Posteriors---Sliced Optimal Transport Metrics for {G}aussian Mixtures},
  author={Nguyen, Khai and Mueller, Peter},
  journal={Journal of Computational and Graphical Statistics},
  year={2026},
  publisher={Taylor \& Francis},
  doi={10.1080/10618600.2026.2634831}
}

@inproceedings{alvarezmelis2020geometric,
  title={Geometric Dataset Distances via Optimal Transport},
  author={Alvarez-Melis, David and Fusi, Nicolo},
  booktitle={Advances in Neural Information Processing Systems},
  volume={33},
  pages={21428--21439},
  year={2020},
  publisher={Curran Associates}
}

@inproceedings{zhang2020bertscore,
  title={{BERTScore}: Evaluating Text Generation with {BERT}},
  author={Zhang, Tianyi and Kishore, Varsha and Wu, Felix and Weinberger, Kilian Q. and Artzi, Yoav},
  booktitle={Proceedings of the ICLR'20},
  year={2020},
  publisher={OpenReview.net}
}

@inproceedings{zhao2019moverscore,
  title={{MoverScore}: Text Generation Evaluating with Contextualized Embeddings and Earth Mover Distance},
  author={Zhao, Wei and Peyrard, Maxime and Liu, Fei and Gao, Yang and Meyer, Christian M. and Eger, Steffen},
  booktitle={Proceedings of the 2019 Conference on Empirical Methods in Natural Language Processing and the 9th International Joint Conference on Natural Language Processing},
  pages={563--578},
  year={2019},
  publisher={Association for Computational Linguistics}
}

@inproceedings{alihosseini2019jointly,
  title={Jointly Measuring Diversity and Quality in Text Generation Models},
  author={Alihosseini, Danial and Montahaei, Ehsan and Soleymani Baghshah, Mahdieh},
  booktitle={Proceedings of the Workshop on Methods for Optimizing and Evaluating Neural Language Generation},
  pages={90--98},
  year={2019},
  publisher={Association for Computational Linguistics}
}

@inproceedings{sajjadi2018precision,
  title={Assessing Generative Models via Precision and Recall},
  author={Sajjadi, Mehdi S. M. and Bachem, Olivier and Lucic, Mario and Bousquet, Olivier and Gelly, Sylvain},
  booktitle={Advances in Neural Information Processing Systems},
  volume={31},
  pages={5228--5237},
  year={2018},
  publisher={Curran Associates}
}

@inproceedings{chong2020effectively,
  title={Effectively Unbiased {FID} and Inception Score and Where to Find Them},
  author={Chong, Min Jin and Forsyth, David},
  booktitle={Proceedings of the IEEE/CVF Conference on Computer Vision and Pattern Recognition},
  pages={6070--6079},
  year={2020},
  publisher={IEEE}
}

@inproceedings{piening2026novel,
  title={A novel sliced fused Gromov-Wasserstein distance},
  author={Piening, Moritz and Beinert, Robert},
  booktitle={Proceedings of the AAAI Conference on Artificial Intelligence},
  volume={40},

  pages={24828--24836},
  year={2026}
}

@article{yachimura2024scegot,
  title={scEGOT: single-cell trajectory inference framework based on entropic {G}aussian mixture optimal transport},
  author={Yachimura, Toshiaki and Wang, Hanbo and Imoto, Yusuke and Yoshida, Momoko and Tasaki, Sohei and Kojima, Yoji and Yabuta, Yukihiro and Saitou, Mitinori and Hiraoka, Yasuaki},
  journal={BMC bioinformatics},
  volume={25},
  number={1},
  pages={388},
  year={2024},
  publisher={Springer}
}

@inproceedings{
haviv2025wasserstein,
title={Wasserstein Flow Matching: {G}enerative Modeling Over Families of Distributions},
author={Doron Haviv and Aram-Alexandre Pooladian and Dana Pe'er and Brandon Amos},
booktitle={Forty-second International Conference on Machine Learning},
year={2025},
url={https://openreview.net/forum?id=MRmI68k3gd}
}

@inproceedings{
Holtzman2020The,
title={The Curious Case of Neural Text Degeneration},
author={Ari Holtzman and Jan Buys and Li Du and Maxwell Forbes and Yejin Choi},
booktitle={International Conference on Learning Representations},
year={2020},
url={https://openreview.net/forum?id=rygGQyrFvH}
}

@misc{hu2026elfembeddedlanguageflows,
      title={ELF: Embedded Language Flows}, 
      author={Keya Hu and Linlu Qiu and Yiyang Lu and Hanhong Zhao and Tianhong Li and Yoon Kim and Jacob Andreas and Kaiming He},
      year={2026},
      eprint={2605.10938},
      archivePrefix={arXiv},
      primaryClass={cs.CL},
      url={https://arxiv.org/abs/2605.10938}, 
}

@inproceedings{
springer2025repetition,
title={Repetition Improves Language Model Embeddings},
author={Jacob Mitchell Springer and Suhas Kotha and Daniel Fried and Graham Neubig and Aditi Raghunathan},
booktitle={The Thirteenth International Conference on Learning Representations},
year={2025},
url={https://openreview.net/forum?id=Ahlrf2HGJR}
}

@inproceedings{Ewais_2026,
   title={Just Pass Twice: Efficient Token Classification with LLMs for Zero-Shot NER},
   url={http://dx.doi.org/10.18653/v1/2026.acl-long.526},
   DOI={10.18653/v1/2026.acl-long.526},
   booktitle={Proceedings of the 64th Annual Meeting of the Association for Computational Linguistics (Volume 1: Long Papers)},
   publisher={Association for Computational Linguistics},
   author={Ewais, Ahmed and Hashish, Ahmed and Ali, Amr},
   year={2026},
   pages={11480–11497} }

@article{19ff28b9-64f9-3656-ba40-08326a05748e,
 ISSN = {00063444, 14643510},
 URL = {http://www.jstor.org/stable/2334029},
 author = {Ralph Allan Bradley and Milton E. Terry},
 journal = {Biometrika},
 number = {3/4},
 pages = {324--345},
 publisher = {[Oxford University Press, Biometrika Trust]},
 title = {Rank Analysis of Incomplete Block Designs: I. The Method of Paired Comparisons},
 urldate = {2026-09-04},
 volume = {39},
 year = {1952}
}

@article{austin2021structured,
  title={Structured denoising diffusion models in discrete state-spaces},
  author={Austin, Jacob and Johnson, Daniel D and Ho, Jonathan and Tarlow, Daniel and Van Den Berg, Rianne},
  journal={Advances in Neural Information Processing Systems},
  volume={34},
  pages={17981--17993},
  year={2021}
}

@inproceedings{
    sahoo2025the,
    title={The Diffusion Duality},
    author={Subham Sekhar Sahoo and Justin Deschenaux and Aaron Gokaslan and Guanghan Wang and Justin T Chiu and Volodymyr Kuleshov},
    booktitle={Forty-second International Conference on Machine Learning},
    year={2025},
    url={https://openreview.net/forum?id=9P9Y8FOSOk}
}

@inproceedings{
arriola2025block,
title={Block Diffusion: Interpolating Between Autoregressive and Diffusion Language Models},
author={Marianne Arriola and Aaron Gokaslan and Justin T Chiu and Zhihan Yang and Zhixuan Qi and Jiaqi Han and Subham Sekhar Sahoo and Volodymyr Kuleshov},
booktitle={The Thirteenth International Conference on Learning Representations},
year={2025},
url={https://arxiv.org/abs/2503.09573}
}

@inproceedings{
rout2025anchored,
title={Anchored Diffusion Language Model},
author={Litu Rout and Constantine Caramanis and Sanjay Shakkottai},
booktitle={The Thirty-ninth Annual Conference on Neural Information Processing Systems},
year={2025},
url={https://openreview.net/forum?id=E8adS5srds}
}

@misc{dieleman2022continuous,
      title={Continuous diffusion for categorical data}, 
      author={Sander Dieleman and Laurent Sartran and Arman Roshannai and Nikolay Savinov and Yaroslav Ganin and Pierre H. Richemond and Arnaud Doucet and Robin Strudel and Chris Dyer and Conor Durkan and Curtis Hawthorne and Rémi Leblond and Will Grathwohl and Jonas Adler},
      year={2022},
      eprint={2211.15089},
      archivePrefix={arXiv},
      primaryClass={cs.CL},
      url={https://arxiv.org/abs/2211.15089}, 
}

@misc{gulrajani2023likelihood,
      title={Likelihood-Based Diffusion Language Models}, 
      author={Ishaan Gulrajani and Tatsunori B. Hashimoto},
      year={2023},
      eprint={2305.18619},
      archivePrefix={arXiv},
      primaryClass={cs.CL},
      url={https://arxiv.org/abs/2305.18619}, 
}

@misc{lee2026flowmaplanguagemodels,
      title={Flow Map Language Models: One-step Language Modeling via Continuous Denoising}, 
      author={Chanhyuk Lee and Jaehoon Yoo and Manan Agarwal and Sheel Shah and Jerry Huang and Aditi Raghunathan and Seunghoon Hong and Nicholas M. Boffi and Jinwoo Kim},
      year={2026},
      eprint={2602.16813},
      archivePrefix={arXiv},
      primaryClass={cs.CL},
      url={https://arxiv.org/abs/2602.16813}, 
}

@misc{sahoo2024simple,
      title={Simple and Effective Masked Diffusion Language Models}, 
      author={Subham Sekhar Sahoo and Marianne Arriola and Yair Schiff and Aaron Gokaslan and Edgar Marroquin and Justin T Chiu and Alexander Rush and Volodymyr Kuleshov},
      year={2024},
      eprint={2406.07524},
      archivePrefix={arXiv},
      primaryClass={cs.CL},
      url={https://arxiv.org/abs/2406.07524}, 
}

@misc{lou2024discretediffusionmodelingestimating,
      title={Discrete Diffusion Modeling by Estimating the Ratios of the Data Distribution}, 
      author={Aaron Lou and Chenlin Meng and Stefano Ermon},
      year={2024},
      eprint={2310.16834},
      archivePrefix={arXiv},
      primaryClass={stat.ML},
      url={https://arxiv.org/abs/2310.16834}, 
}

@misc{breton2025neobertnextgenerationbert,
      title={NeoBERT: A Next-Generation BERT}, 
      author={Lola Le Breton and Quentin Fournier and Mariam El Mezouar and John X. Morris and Sarath Chandar},
      year={2025},
      eprint={2502.19587},
      archivePrefix={arXiv},
      primaryClass={cs.CL},
      url={https://arxiv.org/abs/2502.19587}, 
}

@misc{chen2026langflowcontinuousdiffusionrivals,
      title={LangFlow: Continuous Diffusion Rivals Discrete in Language Modeling}, 
      author={Yuxin Chen and Chumeng Liang and Hangke Sui and Ruihan Guo and Chaoran Cheng and Jiaxuan You and Ge Liu},
      year={2026},
      eprint={2604.11748},
      archivePrefix={arXiv},
      primaryClass={cs.CL},
      url={https://arxiv.org/abs/2604.11748}, 
}

@article{lm1b,
  author       = {Ciprian Chelba and
                  Tom{\'{a}}s Mikolov and
                  Mike Schuster and
                  Qi Ge and
                  Thorsten Brants and
                  Phillipp Koehn},
  title        = {One Billion Word Benchmark for Measuring Progress in Statistical Language
                  Modeling},
  journal      = {CoRR},
  volume       = {abs/1312.3005},
  year         = {2013},
  url          = {http://arxiv.org/abs/1312.3005},
  eprinttype   = {arXiv},
  eprint       = {1312.3005},
  bibsource    = {dblp computer science bibliography, https://dblp.org}
}

@misc{Gokaslan2019OpenWeb,
    title={OpenWebText Corpus},
    author={Gokaslan, Aaron and Cohen, Vanya and Pavlick, Ellie and Tellex, Stefanie},
    howpublished={\url{http://Skylion007.github.io/OpenWebTextCorpus}},
    year={2019}
}

@misc{li2022diffusionlm,
      title={Diffusion-LM Improves Controllable Text Generation}, 
      author={Xiang Lisa Li and John Thickstun and Ishaan Gulrajani and Percy Liang and Tatsunori B. Hashimoto},
      year={2022},
      eprint={2205.14217},
      archivePrefix={arXiv},
      primaryClass={cs.CL},
      url={https://arxiv.org/abs/2205.14217}, 
}

@misc{shi2024simplified,
      title={Simplified and Generalized Masked Diffusion for Discrete Data}, 
      author={Jiaxin Shi and Kehang Han and Zhe Wang and Arnaud Doucet and Michalis K. Titsias},
      year={2025},
      eprint={2406.04329},
      archivePrefix={arXiv},
      primaryClass={cs.LG},
      url={https://arxiv.org/abs/2406.04329}, 
}

@misc{gourevitch2026uniform,
      title={Uniform Diffusion Models Revisited: Leave-One-Out Denoiser and Absorbing State Reformulation}, 
      author={Samson Gourevitch and Yazid Janati and Dario Shariatian and Umut Simsekli and Eric Moulines and Eric P. Xing and Alain Durmus},
      year={2026},
      eprint={2605.22765},
      archivePrefix={arXiv},
      primaryClass={cs.LG},
      url={https://arxiv.org/abs/2605.22765}, 
}

@misc{yoo2026selfconditionedflowmaplanguage,
      title={Self-conditioned Flow Map Language Models via Fixed-point Flows}, 
      author={Jaehoon Yoo and Wonjung Kim and Floor Eijkelboom and Chanhyuk Lee and Nicholas M. Boffi and Seunghoon Hong and Jinwoo Kim},
      year={2026},
      eprint={2607.00714},
      archivePrefix={arXiv},
      primaryClass={cs.CL},
      url={https://arxiv.org/abs/2607.00714}, 
}

@inproceedings{
hersche2026softmasked,
title={Soft-Masked Diffusion Language Models},
author={Michael Hersche and Samuel Moor-Smith and Thomas Hofmann and Abbas Rahimi},
booktitle={The Fourteenth International Conference on Learning Representations},
year={2026},
url={https://openreview.net/forum?id=Gba02UMvrG}
}

@inproceedings{
jo2026loopholing,
title={Loopholing Discrete Diffusion: Deterministic Bypass of the Sampling Wall},
author={Mingyu Jo and Jaesik Yoon and Justin Deschenaux and Caglar Gulcehre and Sungjin Ahn},
booktitle={The Fourteenth International Conference on Learning Representations},
year={2026},
url={https://openreview.net/forum?id=wM3Mdk2Pa5}
}
\bibliographystyle{iclr2027_conference}

\newpage
\appendix

\tableofcontents
\setcounter{topnumber}{4}
\setcounter{bottomnumber}{3}
\setcounter{totalnumber}{6}
\renewcommand{\topfraction}{0.92}
\renewcommand{\bottomfraction}{0.85}
\renewcommand{\textfraction}{0.05}
\renewcommand{\bottomfraction}{0.9}
\renewcommand{\floatpagefraction}{0.5}

\newpage
\section{Proofs and Theoretical Details}

\subsection{SOL is a Metric}
\label{Sec:proofs}
\faithfulness*
\begin{proof}
We will show that the injectivity of the last hidden-token representation implies injectivity of $T^\theta$. Then its push forward
$T_\#^\theta$ is injective as well and the assertion follows from the fact that $\text{DSW}$ is a metric on empirical meta-measures, see
\cite{piening2026slicing}.

It remains to prove that $T^\theta$ is indeed injective.
Let $\mathbf x,\mathbf y\in\mathcal V^{\leq L}$ satisfy
$T^\theta(\mathbf x)
=
T^\theta(\mathbf y)$.
By the injectivity of the last hidden-token representation it follows
that the components of 
$(h^\theta_{\mathbf{x} \le 1}, \ldots, h^\theta_{\mathbf{x} \le |\mathbf {x}|}) $
are pairwise distinct and similarly for $\mathbf{y}$. 
By $T^\theta(\mathbf{x}) = T^\theta(\mathbf{y})$,
%assumption, 
we thus have that $|\mathbf{x}| = |\mathbf{y}| = n$.
Further, it follows that %by assumption, 
there exits $j \le n$ such that
$h^\theta_{\mathbf{x} \le n} = h^\theta_{\mathbf{y} \le j}$.
By injectivity of the last hidden-token representation this implies that
$\mathbf{x}_{\le n} =\mathbf{x} = \mathbf{y}_{\le j}$, so that $j=n$ and we are done.
\end{proof}

\subsection{Approximation of the Double-Sliced Wasserstein Distance}
\label{app:proof-dsw-sample-complexity}

Throughout this section, let $X\subset B_R(0)\subset\mathbb R^d$ be
nonempty and Borel for some $R>0$. We equip $\mathcal P_2(X)$ with the
Wasserstein distance $\text{W}$. On the sphere, we use the uniform
measure $\sigma:=\mathcal U(\mathbb S^{d-1})$. On $L_2([0,1])$, we fix
a measure $\gamma\in\mathcal P_2(L_2([0,1]))$ and denote its moments by
\begin{align*}
    m_r(\gamma):=\int\|g\|_{L_2}^r\,\d\gamma(g),
    \qquad r\in\{2,4\}.
\end{align*}
The second moment is finite by assumption, while $m_4(\gamma)=\infty$
is allowed.

Recall from the main text that each pair of slices $(\xi,g)$ defines
the map
\begin{align*}
    M_{\xi,g}\colon\mathcal P_2(X)\to\mathbb R,
    \qquad
    M_{\xi,g}(\mu)=\langle Q_{\mathbf e_\xi\mu},g\rangle_{L_2},
\end{align*}
where $\mathbf e_\xi\mu=(S_\xi)_\#\mu$ is the projection of $\mu$
under $S_\xi(x):=\langle\xi,x\rangle$.
For $\bm\lambda,\bm\eta\in\mathcal P_2(\mathcal P_2(X))$, the
corresponding one-dimensional cost is
\begin{align*}
    c_{\bm\lambda,\bm\eta}(\xi,g)
    :=\text{W}^2\bigl((M_{\xi,g})_\#\bm\lambda,(M_{\xi,g})_\#\bm\eta\bigr),
\end{align*}
and the double-sliced Wasserstein distance reads
\begin{align*}
    \mathrm{DSW}(\bm\lambda,\bm\eta\mid\gamma)
    =\left(\int c_{\bm\lambda,\bm\eta}\,\d(\sigma\otimes\gamma)\right)^{1/2}.
\end{align*}
In practice, the slicing measures are replaced by empirical measures.
For $L_\xi,L_g\geq1$ and slices $\xi_i\in\mathbb S^{d-1}$ and
$g_j\in L_2([0,1])$, we set
\begin{align}
    \sigma_{L_\xi}:=\frac1{L_\xi}\sum_{i=1}^{L_\xi}\delta_{\xi_i},
    \qquad
    \gamma_{L_g}:=\frac1{L_g}\sum_{j=1}^{L_g}\delta_{g_j}.
    \label{eq:dsw-slicing-measures}
\end{align}
The finite-slice approximation is then given by
\begin{align}
    \mathrm{DSW}_{L_\xi,L_g}(\bm\lambda,\bm\eta\mid\gamma)
    :=\left(\int c_{\bm\lambda,\bm\eta}
        \,\d(\sigma_{L_\xi}\otimes\gamma_{L_g})\right)^{1/2}.
    \label{eq:dsw-finite-slice-definition}
\end{align}

We bound the error that arises when both the slicing measures and the
measures $\bm\lambda,\bm\eta$ are replaced by empirical measures. This
is done in two steps.
Proposition~\ref{prop:dsw-projection-approximation} treats the slicing
error for fixed measures, and
Proposition~\ref{prop:dsw-measure-approximation} treats the sampling
error for fixed slices.
Proposition~\ref{prop:dsw-sample-complexity} combines both.
All three results rely on the regularity properties collected in the
following lemma.

\begin{lemma}
\label{lem:double-slice-continuity}
Let $\xi,\xi'\in\mathbb S^{d-1}$, $g,h\in L_2([0,1])$, and
$\mu,\nu\in\mathcal P_2(X)$. Then
\begin{align}
    |M_{\xi,g}(\mu)-M_{\xi',h}(\nu)|
    \leq\|g\|_{L_2}\bigl(
        \text{W}(\mu,\nu)+R\|\xi-\xi'\|_2
    \bigr)+R\|g-h\|_{L_2}.
    \label{eq:double-slice-continuity}
\end{align}
In particular, the map $(\xi,g,\mu)\mapsto M_{\xi,g}(\mu)$ is jointly
continuous, and it holds $|M_{\xi,g}(\mu)|\leq R\|g\|_{L_2}$.
Moreover, for every $\bm\rho\in\mathcal P_2(\mathcal P_2(X))$, the
pushforwards satisfy
\begin{align}
    \text{W}\bigl((M_{\xi,g})_\#\bm\rho,(M_{\xi',h})_\#\bm\rho\bigr)
    \leq R\|g\|_{L_2}\|\xi-\xi'\|_2+R\|g-h\|_{L_2}.
    \label{eq:double-slice-pushforward}
\end{align}
Finally, the function $c_{\bm\lambda,\bm\eta}$ is continuous with
\begin{align}
    0\leq c_{\bm\lambda,\bm\eta}(\xi,g)
    \leq4R^2\|g\|_{L_2}^2.
    \label{eq:dsw-slice-bound}
\end{align}
\end{lemma}

\begin{proof}
\emph{Estimates for the quantile functions.}
Set $q_{\xi,\mu}:=Q_{\mathbf e_\xi\mu}$. The quantile representation
of $\text{W}$ on $\mathbb R$ and the triangle inequality in
$L_2([0,1])$ give
\begin{align*}
    \|q_{\xi,\mu}-q_{\xi',\nu}\|_{L_2}
    &\leq\|q_{\xi,\mu}-q_{\xi,\nu}\|_{L_2}
        +\|q_{\xi,\nu}-q_{\xi',\nu}\|_{L_2}\\
    &=\text{W}\bigl((S_\xi)_\#\mu,(S_\xi)_\#\nu\bigr)
        +\text{W}\bigl((S_\xi)_\#\nu,(S_{\xi'})_\#\nu\bigr).
\end{align*}
We bound the two terms separately. Since $S_\xi$ is $1$-Lipschitz,
pushing an optimal coupling of $\mu$ and $\nu$ forward under
$(S_\xi,S_\xi)$ yields
\begin{align*}
    \text{W}\bigl((S_\xi)_\#\mu,(S_\xi)_\#\nu\bigr)
    \leq\text{W}(\mu,\nu).
\end{align*}
For the second term, we use the coupling $(S_\xi,S_{\xi'})_\#\nu$.
Since $\|x\|_2\leq R$ on $X$, this gives
\begin{align*}
    \text{W}\bigl((S_\xi)_\#\nu,(S_{\xi'})_\#\nu\bigr)
    \leq\left(\int_X|\langle\xi-\xi',x\rangle|^2
        \,\d\nu(x)\right)^{1/2}
    \leq R\|\xi-\xi'\|_2.
\end{align*}
In the same way, the norm of each quantile function is bounded by
\begin{align*}
    \|q_{\xi,\mu}\|_{L_2}^2
    =\int_X|\langle\xi,x\rangle|^2\,\d\mu(x)
    \leq R^2.
\end{align*}

\emph{Continuity of $M$.}
We decompose
\begin{align*}
    M_{\xi,g}(\mu)-M_{\xi',h}(\nu)
    =\langle q_{\xi,\mu}-q_{\xi',\nu},g\rangle_{L_2}
    +\langle q_{\xi',\nu},g-h\rangle_{L_2}.
\end{align*}
Applying Cauchy--Schwarz to both terms and inserting the estimates
above proves \eqref{eq:double-slice-continuity}. The joint continuity
and the bound on $|M_{\xi,g}(\mu)|$ follow directly.

\emph{Continuity of the pushforwards.}
Let $\bm\rho\in\mathcal P_2(\mathcal P_2(X))$. Since $M_{\xi,g}$ and
$M_{\xi',h}$ are continuous, the measure
$(M_{\xi,g},M_{\xi',h})_\#\bm\rho$ is a coupling of the two
pushforwards. Using this coupling in the definition of $\text{W}$ gives
\begin{align*}
    \text{W}^2\bigl((M_{\xi,g})_\#\bm\rho,(M_{\xi',h})_\#\bm\rho\bigr)
    \leq\int|M_{\xi,g}(\mu)-M_{\xi',h}(\mu)|^2\,\d\bm\rho(\mu).
\end{align*}
By \eqref{eq:double-slice-continuity} with $\nu=\mu$, the integrand
is bounded by the square of the right-hand side of
\eqref{eq:double-slice-pushforward}. This proves
\eqref{eq:double-slice-pushforward}.

\emph{Properties of $c$.}
The triangle inequality for $\text{W}$ gives
\begin{align*}
    |\text{W}(a,b)-\text{W}(a',b')|
    \leq\text{W}(a,a')+\text{W}(b,b')
\end{align*}
for all $a,b,a',b'\in\mathcal P_2(\mathbb R)$. Together with
\eqref{eq:double-slice-pushforward} for $\bm\rho=\bm\lambda$ and
$\bm\rho=\bm\eta$, this shows that $\sqrt{c_{\bm\lambda,\bm\eta}}$ is
continuous, and hence so is $c_{\bm\lambda,\bm\eta}$.
Finally, both slice laws in the definition of
$c_{\bm\lambda,\bm\eta}(\xi,g)$ are supported in the interval
$[-R\|g\|_{L_2},R\|g\|_{L_2}]$. Their squared Wasserstein distance is
therefore at most the squared length of this interval, which proves
\eqref{eq:dsw-slice-bound}.
\end{proof}

\paragraph{Approximation of the integrals.}
We first keep $\bm\mu,\bm\nu$ fixed and approximate only the slicing
measures $\sigma$ and $\gamma$.

\begin{proposition}
\label{prop:dsw-projection-approximation}
Let $\bm\mu,\bm\nu\in\mathcal P_2(\mathcal P_2(X))$ and assume
$m_4(\gamma)<\infty$. Let the slices $\xi_1,\ldots,\xi_{L_\xi}$ be iid
with law $\sigma$, let $g_1,\ldots,g_{L_g}$ be iid with law $\gamma$,
and let both families be independent. Then
\begin{align}
    \mathbb E_s\left|
        \mathrm{DSW}_{L_\xi,L_g}(\bm\mu,\bm\nu\mid\gamma)
        -\mathrm{DSW}(\bm\mu,\bm\nu\mid\gamma)
    \right|
    \leq2R\left(
        \frac{m_2(\gamma)^2}{L_\xi}
        +\frac{m_4(\gamma)}{L_g}
    \right)^{1/4},
    \label{eq:dsw-projection-approximation}
\end{align}
where $\mathbb E_s$ denotes the expectation over the slices.
\end{proposition}

\begin{proof}
\emph{Decomposition.}
Write $c:=c_{\bm\mu,\bm\nu}$ and denote the exact and the approximate
integral by
\begin{align*}
    I:=\int c\,\d(\sigma\otimes\gamma),
    \qquad
    \widehat I:=\int c\,\d(\sigma_{L_\xi}\otimes\gamma_{L_g}).
\end{align*}
We replace the two slicing measures one after the other and split
$\widehat I-I=A+B$, where
\begin{align*}
    A:=\int c\,\d\bigl((\sigma_{L_\xi}-\sigma)\otimes\gamma\bigr),
    \qquad
    B:=\int c\,\d\bigl(\sigma_{L_\xi}\otimes(\gamma_{L_g}-\gamma)\bigr).
\end{align*}
The term $A$ replaces $\sigma$ by $\sigma_{L_\xi}$. The term $B$ then
replaces $\gamma$ by $\gamma_{L_g}$.

\emph{A variance identity.}
Let $\rho$ be a probability measure, and let $\rho_m$ be its empirical
approximation based on $m$ iid samples. For every $f\in L_2(\rho)$,
expanding the square gives
\begin{align}
    \mathbb E\left|\int f\,\d(\rho_m-\rho)\right|^2
    =\frac1m\left(\int f^2\,\d\rho
        -\left(\int f\,\d\rho\right)^2\right)
    \leq\frac1m\int f^2\,\d\rho.
    \label{eq:empirical-integral-bound}
\end{align}

\emph{Bounding $A$ and $B$.}
For $A$, we apply \eqref{eq:empirical-integral-bound} to
$\sigma_{L_\xi}$ with the function
\begin{align*}
    \xi\mapsto\int c(\xi,g)\,\d\gamma(g).
\end{align*}
By \eqref{eq:dsw-slice-bound}, this function is bounded by
$4R^2m_2(\gamma)$. Hence
\begin{align*}
    \mathbb E_s A^2
    \leq\frac1{L_\xi}\int
        \left(\int c(\xi,g)\,\d\gamma(g)\right)^2\d\sigma(\xi)
    \leq\frac{16R^4m_2(\gamma)^2}{L_\xi}.
\end{align*}
For $B$, we fix $\xi_1,\ldots,\xi_{L_\xi}$ and apply
\eqref{eq:empirical-integral-bound} to $\gamma_{L_g}$ with the function
\begin{align*}
    g\mapsto\int c(\xi,g)\,\d\sigma_{L_\xi}(\xi).
\end{align*}
By \eqref{eq:dsw-slice-bound}, this function is bounded by
$4R^2\|g\|_{L_2}^2$. Taking the expectation over the $\xi_i$
afterwards gives
\begin{align*}
    \mathbb E_s B^2
    \leq\frac1{L_g}\mathbb E_s\int
        \left(\int c(\xi,g)\,\d\sigma_{L_\xi}(\xi)\right)^2\d\gamma(g)
    \leq\frac{16R^4m_4(\gamma)}{L_g}.
\end{align*}
This is the only place where the fourth moment enters.

\emph{The cross term.}
For every fixed collection of $\xi_i$, the expectation of $B$ with
respect to $g_1,\ldots,g_{L_g}$ vanishes. Since $A$ depends only on
the $\xi_i$, it follows that $\mathbb E_s[AB]=0$. Hence
\begin{align*}
    \mathbb E_s|\widehat I-I|^2
    =\mathbb E_s A^2+\mathbb E_s B^2
    \leq16R^4\left(
        \frac{m_2(\gamma)^2}{L_\xi}
        +\frac{m_4(\gamma)}{L_g}
    \right).
\end{align*}

\emph{Taking square roots.}
For $a,b\geq0$, it holds $|\sqrt a-\sqrt b|\leq|a-b|^{1/2}$. Together
with Jensen's inequality for the concave map $t\mapsto t^{1/4}$, we
obtain
\begin{align*}
    \mathbb E_s|\sqrt{\widehat I}-\sqrt I|
    \leq\mathbb E_s|\widehat I-I|^{1/2}
    \leq\bigl(\mathbb E_s|\widehat I-I|^2\bigr)^{1/4}.
\end{align*}
Inserting the previous bound proves the assertion.
\end{proof}

\paragraph{Approximation of the measures.}
We next keep the slices fixed and approximate the measures
$\bm\mu,\bm\nu$. The key ingredient is the following estimate on
$\mathbb R$.

\begin{lemma}
\label{lem:empirical-one-dimensional-w2}
Let $\alpha\in\mathcal P_2(\mathbb R)$ be supported in $[-B,B]$ for
some $B\geq0$, and let $\alpha_n$ be its empirical measure based on
$n\geq1$ iid samples. Then
\begin{align}
    \mathbb E\,\text{W}^2(\alpha_n,\alpha)
    \leq\frac{2B^2}{\sqrt n}.
    \label{eq:empirical-one-dimensional-w2}
\end{align}
\end{lemma}

\begin{proof}
Let $F_n$ and $F$ denote the distribution functions of $\alpha_n$ and
$\alpha$. Both quantile functions take values in $[-B,B]$, so that
$|Q_{\alpha_n}-Q_\alpha|\leq2B$. Using the quantile representation of
$\text{W}$, we obtain
\begin{align*}
    \text{W}^2(\alpha_n,\alpha)
    =\int_0^1|Q_{\alpha_n}-Q_\alpha|^2\,\d s
    \leq2B\int_0^1|Q_{\alpha_n}-Q_\alpha|\,\d s
    =2B\int_{-B}^B|F_n(t)-F(t)|\,\d t.
\end{align*}
The last identity is the representation of the one-dimensional
$1$-Wasserstein distance by distribution functions.
For fixed $t$, the random variable $nF_n(t)$ is binomially distributed
with parameters $n$ and $F(t)$. By Cauchy--Schwarz,
\begin{align*}
    \mathbb E|F_n(t)-F(t)|
    \leq\sqrt{\frac{F(t)(1-F(t))}{n}}
    \leq\frac1{2\sqrt n}.
\end{align*}
Taking expectations in the first estimate and using Tonelli's theorem,
we conclude
\begin{align*}
    \mathbb E\,\text{W}^2(\alpha_n,\alpha)
    \leq2B\int_{-B}^B\mathbb E|F_n(t)-F(t)|\,\d t
    \leq\frac{2B^2}{\sqrt n}.
\end{align*}
\end{proof}

\begin{proposition}
\label{prop:dsw-measure-approximation}
Let $\bm\mu,\bm\nu\in\mathcal P_2(\mathcal P_2(X))$. Let
$\mu_1,\ldots,\mu_n$ be iid with law $\bm\mu$, let
$\nu_1,\ldots,\nu_n$ be iid with law $\bm\nu$, and consider the
empirical measures
\begin{align*}
    \bm\mu_n:=\frac1n\sum_{k=1}^n\delta_{\mu_k},
    \qquad
    \bm\nu_n:=\frac1n\sum_{k=1}^n\delta_{\nu_k}.
\end{align*}
Then, for arbitrary fixed slices in \eqref{eq:dsw-slicing-measures},
it holds
\begin{align}
    \mathbb E_n\left|
        \mathrm{DSW}_{L_\xi,L_g}(\bm\mu_n,\bm\nu_n\mid\gamma)
        -\mathrm{DSW}_{L_\xi,L_g}(\bm\mu,\bm\nu\mid\gamma)
    \right|
    \leq\frac{\sqrt8\,R}{n^{1/4}}
        \left(\frac1{L_g}\sum_{j=1}^{L_g}
            \|g_j\|_{L_2}^2\right)^{1/2},
    \label{eq:dsw-fixed-slice-sampling}
\end{align}
where $\mathbb E_n$ denotes the expectation over the samples.
The same bound holds for $\mathrm{DSW}$ if the empirical second moment
on the right-hand side is replaced by $m_2(\gamma)$.
\end{proposition}

\begin{proof}
\emph{Reduction to one measure.}
By the triangle inequality for $\text{W}$ and Minkowski's inequality,
$\mathrm{DSW}_{L_\xi,L_g}$ is a pseudometric on
$\mathcal P_2(\mathcal P_2(X))$. Hence
\begin{align*}
    \left|
        \mathrm{DSW}_{L_\xi,L_g}(\bm\mu_n,\bm\nu_n\mid\gamma)
        -\mathrm{DSW}_{L_\xi,L_g}(\bm\mu,\bm\nu\mid\gamma)
    \right|
    \leq
        \mathrm{DSW}_{L_\xi,L_g}(\bm\mu_n,\bm\mu\mid\gamma)
        +\mathrm{DSW}_{L_\xi,L_g}(\bm\nu_n,\bm\nu\mid\gamma).
\end{align*}
It suffices to bound the expectation of the first term. The second
term is treated in the same way.

\emph{Reduction to one dimension.}
Fix $(\xi,g)$. The pushforward of $\bm\mu_n$ is
\begin{align*}
    (M_{\xi,g})_\#\bm\mu_n
    =\frac1n\sum_{k=1}^n\delta_{M_{\xi,g}(\mu_k)}.
\end{align*}
Since the $M_{\xi,g}(\mu_k)$ are iid with law $(M_{\xi,g})_\#\bm\mu$,
this is the empirical measure of $(M_{\xi,g})_\#\bm\mu$. By
Lemma~\ref{lem:double-slice-continuity}, both measures are supported
in $[-R\|g\|_{L_2},R\|g\|_{L_2}]$.
Lemma~\ref{lem:empirical-one-dimensional-w2} with $B=R\|g\|_{L_2}$
therefore gives
\begin{align*}
    \mathbb E_n\,c_{\bm\mu_n,\bm\mu}(\xi,g)
    \leq\frac{2R^2\|g\|_{L_2}^2}{\sqrt n}.
\end{align*}

\emph{Conclusion.}
By Jensen's inequality for the square root, we obtain
\begin{align*}
    \mathbb E_n\,
        \mathrm{DSW}_{L_\xi,L_g}(\bm\mu_n,\bm\mu\mid\gamma)
    \leq\left(\int
        \mathbb E_n\,c_{\bm\mu_n,\bm\mu}
        \,\d(\sigma_{L_\xi}\otimes\gamma_{L_g})\right)^{1/2}
    \leq\frac{\sqrt2\,R}{n^{1/4}}
        \left(\int\|g\|_{L_2}^2\,\d\gamma_{L_g}(g)\right)^{1/2}.
\end{align*}
Adding the same bound for $\bm\nu$ proves
\eqref{eq:dsw-fixed-slice-sampling}.

\emph{The case of $\mathrm{DSW}$.}
Here we replace $\sigma_{L_\xi}\otimes\gamma_{L_g}$ by
$\sigma\otimes\gamma$ in the argument above. The exchange of
$\mathbb E_n$ and the integral is then justified by Tonelli's theorem.
Indeed, by Lemma~\ref{lem:double-slice-continuity}, the atoms
$M_{\xi,g}(\mu_k)$ and the measure $(M_{\xi,g})_\#\bm\mu$ depend
continuously on $(\xi,g,\mu_1,\ldots,\mu_n)$. Hence the map
\begin{align*}
    (\xi,g,\mu_1,\ldots,\mu_n)\mapsto c_{\bm\mu_n,\bm\mu}(\xi,g)
\end{align*}
is continuous on a product of separable metric spaces and thus jointly
Borel measurable.
\end{proof}

\paragraph{Combined approximation.}
Finally, we approximate the slicing measures and the measures
simultaneously.

\begin{proposition}
\label{prop:dsw-sample-complexity}
Let the assumptions of
Propositions~\ref{prop:dsw-projection-approximation} and
\ref{prop:dsw-measure-approximation} hold, and assume in addition that
the samples are independent of the slices. Then
\begin{align}
    \mathbb E\left|
        \mathrm{DSW}_{L_\xi,L_g}(\bm\mu_n,\bm\nu_n\mid\gamma)
        -\mathrm{DSW}(\bm\mu,\bm\nu\mid\gamma)
    \right|
    \leq
    \frac{\sqrt8\,R\sqrt{m_2(\gamma)}}{n^{1/4}}
    +2R\left(
        \frac{m_2(\gamma)^2}{L_\xi}
        +\frac{m_4(\gamma)}{L_g}
    \right)^{1/4},
    \label{eq:dsw-finite-slice-total}
\end{align}
where $\mathbb E$ denotes the expectation over samples and slices.
\end{proposition}

\begin{proof}
We insert the intermediate quantity
$\mathrm{DSW}_{L_\xi,L_g}(\bm\mu,\bm\nu\mid\gamma)$ and apply the
triangle inequality. This gives
\begin{align*}
    \mathbb E\left|
        \mathrm{DSW}_{L_\xi,L_g}(\bm\mu_n,\bm\nu_n\mid\gamma)
        -\mathrm{DSW}(\bm\mu,\bm\nu\mid\gamma)
    \right|
    &\leq\mathbb E\left|
        \mathrm{DSW}_{L_\xi,L_g}(\bm\mu_n,\bm\nu_n\mid\gamma)
        -\mathrm{DSW}_{L_\xi,L_g}(\bm\mu,\bm\nu\mid\gamma)
    \right|\\
    &\quad+\mathbb E_s\left|
        \mathrm{DSW}_{L_\xi,L_g}(\bm\mu,\bm\nu\mid\gamma)
        -\mathrm{DSW}(\bm\mu,\bm\nu\mid\gamma)
    \right|.
\end{align*}
The second term is bounded by
Proposition~\ref{prop:dsw-projection-approximation}.

For the first term, joint measurability in slices and samples follows
as at the end of the proof of
Proposition~\ref{prop:dsw-measure-approximation}. Since the samples
are independent of the slices, Fubini's theorem allows us to first
take the expectation over the samples with the slices fixed.
Proposition~\ref{prop:dsw-measure-approximation} then bounds the first
term by
\begin{align*}
    \frac{\sqrt8\,R}{n^{1/4}}\,
    \mathbb E_s\left(\int\|g\|_{L_2}^2\,\d\gamma_{L_g}(g)\right)^{1/2}.
\end{align*}
By Jensen's inequality and since the $g_j$ have law $\gamma$,
\begin{align*}
    \mathbb E_s\left(\int\|g\|_{L_2}^2\,\d\gamma_{L_g}(g)\right)^{1/2}
    \leq\left(\mathbb E_s\int\|g\|_{L_2}^2
        \,\d\gamma_{L_g}(g)\right)^{1/2}
    =\sqrt{m_2(\gamma)}.
\end{align*}
Combining both bounds proves the assertion.
\end{proof}

\paragraph{Remarks.}
\emph{Sample sizes.}
The proofs of Propositions~\ref{prop:dsw-measure-approximation}
and~\ref{prop:dsw-sample-complexity} use neither independence between
the $\mu_k$ and the $\nu_k$ nor equal sample sizes. With $m$ samples
of $\bm\nu$, the factor $\sqrt8\,n^{-1/4}$ in
\eqref{eq:dsw-fixed-slice-sampling} and
\eqref{eq:dsw-finite-slice-total} becomes
\begin{align*}
    \sqrt2\,\bigl(n^{-1/4}+m^{-1/4}\bigr).
\end{align*}

\emph{Slice rates.}
The slice rates $L_\xi^{-1/4}$ and $L_g^{-1/4}$ in
\eqref{eq:dsw-projection-approximation} stem from the estimate
$|\sqrt a-\sqrt b|\leq|a-b|^{1/2}$. This estimate is only tight when
$\mathrm{DSW}(\bm\mu,\bm\nu\mid\gamma)$ is close to zero. For the
squared distance, the proof of
Proposition~\ref{prop:dsw-projection-approximation} yields the Monte
Carlo rate
\begin{align*}
    \mathbb E_s\left|
        \mathrm{DSW}^2_{L_\xi,L_g}(\bm\mu,\bm\nu\mid\gamma)
        -\mathrm{DSW}^2(\bm\mu,\bm\nu\mid\gamma)
    \right|
    \leq4R^2\left(
        \frac{m_2(\gamma)^2}{L_\xi}
        +\frac{m_4(\gamma)}{L_g}
    \right)^{1/2}.
\end{align*}
If $\mathrm{DSW}(\bm\mu,\bm\nu\mid\gamma)>0$, the estimate
$|\sqrt a-\sqrt b|\leq|a-b|/\sqrt b$ transfers this rate to
$\mathrm{DSW}$ itself. More precisely,
\begin{align*}
    \mathbb E_s\left|
        \mathrm{DSW}_{L_\xi,L_g}(\bm\mu,\bm\nu\mid\gamma)
        -\mathrm{DSW}(\bm\mu,\bm\nu\mid\gamma)
    \right|
    \leq\frac{4R^2}{\mathrm{DSW}(\bm\mu,\bm\nu\mid\gamma)}
    \left(
        \frac{m_2(\gamma)^2}{L_\xi}
        +\frac{m_4(\gamma)}{L_g}
    \right)^{1/2}.
\end{align*}

\emph{Sample rate.}
In contrast, the rate in Lemma~\ref{lem:empirical-one-dimensional-w2}
cannot be improved in general. Consider
$\alpha=\frac12(\delta_{-B}+\delta_B)$. Then
\begin{align*}
    \text{W}^2(\alpha_n,\alpha)
    =4B^2\,\Bigl|\alpha_n(\{B\})-\frac12\Bigr|,
\end{align*}
and the expectation of the right-hand side is of order $B^2n^{-1/2}$.

\subsection{Double-Sliced Wasserstein compared vs Mean Hidden Token Sliced Wasserstein}
\label{sec:dsw_mean_sw}

The relation between DSW and mean-SW follows directly from the quantile representation of the inner Wasserstein distance. Let $\mu_x$ denote the token cloud of a document $x$ and define its mean by
\[
    m_x := \int z\,\mathrm d\mu_x(z).
\]
For $\xi\in\mathbb S^{d-1}$, decompose the quantile function of the projected token cloud as
\[
    Q_{\xi_\#\mu_x}(u)
    =
    \langle \xi,m_x\rangle+r_{\xi,x}(u),
    \qquad
    \int_0^1 r_{\xi,x}(u)\,\mathrm du=0.
\]
Hence, for a second slicing function $g\in L_2([0,1])$, the corresponding DSW feature satisfies
\[
    M_{\xi,g}(\mu_x)
    :=
    \bigl\langle Q_{\xi_\#\mu_x},g\bigr\rangle_{L_2}
    =
    a_g\langle\xi,m_x\rangle
    +
    \langle r_{\xi,x},g\rangle_{L_2},
    \qquad
    a_g:=\int_0^1 g(u)\,\mathrm du .
\]
The first term depends only on the centroid of the token cloud, whereas the second term captures its centered projected shape.

In particular, for the constant function $g\equiv1$,
\[
    M_{\xi,1}(\mu_x)=\langle\xi,m_x\rangle,
\]
so restricting the second slicing stage to constant functions recovers mean-SW, up to the chosen normalization. Conversely, if two token clouds satisfy $m_x=m_y$, mean-SW cannot distinguish them, while
\[
    M_{\xi,g}(\mu_x)-M_{\xi,g}(\mu_y)
    =
    \langle r_{\xi,x}-r_{\xi,y},g\rangle_{L_2}
\]
can be nonzero. Thus, DSW adds information precisely through nonconstant test functions that detect differences in the projected quantile shapes. This also motivates the ablation between constant $g$, nonconstant centered $g$, and the full family of cutting functions. We note that the degeneration to constant test functions is already discussed in the original DSW construction.

\subsection{Approximation of SOL}
\label{sec:sott-approximation}

Proposition~\ref{prop} concerns the exact SOL distance.
In practice, we approximate it from finitely many text samples, finitely
many slicing directions, and a finite representation of the second-level
slicing functions. We separate these approximation effects below.

\paragraph{Finite samples and slices.}
Let $\bm\mu_n,\bm\nu_n$ denote empirical meta-measures based on
$n$ independent samples, and let
$\mathrm{DSW}_{L_\xi,L_g}$ denote the estimator obtained by replacing
the two slicing measures by $L_\xi$ and $L_g$ Monte Carlo samples.
Under the assumptions of
Propositions~\ref{prop:dsw-projection-approximation} and
\ref{prop:dsw-measure-approximation},
the sample and slicing approximations converge independently as
$n,L_\xi,L_g\to\infty$.

\paragraph{Finite resolution of the second slicing functions.}
The exact second-level feature is
\[
    M_{\xi,g}(\mu)
    =
    \left\langle
        Q_{\mathbf e_\xi\mu},
        g
    \right\rangle_{L_2([0,1])}.
\]
For $M\geq1$, partition $[0,1]$ into $I_m
    =
    \left(
        \frac{m-1}{M},
        \frac{m}{M}
    \right],$ $s_m=\frac{m-\frac12}{M},$
and define the piecewise-constant approximation $(\Pi_M g)(s)
    =
    g(s_m),$
We keep the empirical quantile itself exact and only replace $g$ by
$\Pi_Mg$. If the token clouds are supported in the ball
$B_R(0)\subset\mathbb R^d$, then
\begin{align}
    \left|
        M_{\xi,g}(\mu)
        -
        M_{\xi,\Pi_Mg}(\mu)
    \right|
    &\leq
    \|Q_{\mathbf e_\xi\mu}\|_{L_2}
    \|g-\Pi_Mg\|_{L_2}
    \nonumber \leq
    R\|g-\Pi_Mg\|_{L_2}.
    \label{eq:dsw-grid-feature-error}
\end{align}
Consequently, if $\mathrm{DSW}_M$ denotes the corresponding distance
with $g$ replaced by $\Pi_Mg$, then
\begin{equation}
    \left|
        \mathrm{DSW}_M(\bm\mu,\bm\nu\mid\gamma)
        -
        \mathrm{DSW}(\bm\mu,\bm\nu\mid\gamma)
    \right|
    \leq
    2R
    \left(
        \mathbb E_{g\sim\gamma}
        \|g-\Pi_Mg\|_{L_2}^2
    \right)^{1/2}.
    \label{eq:dsw-grid-error}
\end{equation}

For the RBF Gaussian process used in our experiments, $k_\lambda(s,t)
    =
    \exp\left(
        -(s-t)^2 /({2\lambda^2)}
    \right),$
we have
\begin{align}
    \mathbb E
    \|g-\Pi_Mg\|_{L_2}^2
    &=
    \sum_{m=1}^M
    \int_{I_m}
    2\left(
        1-
        \exp\left(
            -\frac{(s-s_m)^2}{2\lambda^2}
        \right)
    \right)\,\mathrm ds
    \nonumber
    \leq
    \frac{1}{12\lambda^2M^2},
\end{align}
and hence
\begin{equation}
    \left|
        \mathrm{DSW}_M-\mathrm{DSW}
    \right|
    \leq
    \frac{R}{\sqrt 3\,\lambda M}.
    \label{eq:dsw-grid-rbf-rate}
\end{equation}
Thus the finite grid introduces an $O(M^{-1})$ approximation error for
the RBF second-level slicing measure.

The numerical implementation in
Appendix~\ref{sec:dsw-numerical} realizes
$M_{\xi,\Pi_Mg}$ by integrating the empirical quantile exactly over
each bin. In particular, the finite-$M$ approximation acts only on the
second-level slicing function and not on the empirical quantile itself.
The finite estimator need not itself satisfy identity of
indiscernibles, the metric statement in
Proposition~\ref{prop} concerns the exact SOL distance.
%%%%%%%%%%%%%%%%%%%%%%%%%%%%%%%%%%%%%%%%%%%%%%%%%%%%%%%%%%%%%%%%%%%%%%%%%%%%%%%%%%%%%%%%
\newpage
\section{Extended Related Work}
\subsection{Detailed Description of MAUVE} \label{app:MAUVE}
MAUVE \citep{pillutla2021MAUVE,pillutla2023MAUVE} compares two text distributions after
quantization in a Euclidean latent space. Given samples
$\mathbf x_1,\ldots,\mathbf x_n\sim p_{\mathrm{data}}$ and
$\mathbf y_1,\ldots,\mathbf y_m\sim p_{\mathrm{gen}}$, it proceeds as follows:
\begin{enumerate}
\setlength{\itemsep}{1pt}
    \item \textbf{Embedding.} Map the texts to $\R^d$, classically using the
    last-token representation $\mathbf h^\theta$.
    \item \textbf{Quantization.} Jointly cluster the embedded real and generated
    samples by $K$-means into cells $C_1,\ldots,C_K\subset\R^d$.
    \item \textbf{Histograms.} Form cluster weights
    $$\widetilde p_{\mathrm{data},k}
    :=n^{-1}\sum_{i=1}^n\mathbf 1_{\{\mathbf h^\theta(\mathbf x_i)\in C_k\}}
    \quad 
    \text{and} \quad 
    \widetilde p_{\mathrm{gen},k}
    :=m^{-1}\sum_{i=1}^m\mathbf 1_{\{\mathbf h^\theta(\mathbf y_i)\in C_k\}}.$$
\item \textbf{KL divergence frontier.} For
$\widetilde r_\lambda
:=\lambda\widetilde p_{\mathrm{data}}
 +(1-\lambda)\widetilde p_{\mathrm{gen}}$, define
\[
\gamma_1(\lambda):=
\exp\!\left(-c\sum_{k=1}^K \widetilde p_{\mathrm{gen},k}
\log\frac{\widetilde p_{\mathrm{gen},k}}{\widetilde r_{\lambda,k}}\right),
\quad
\gamma_2(\lambda):=
\exp\!\left(-c\sum_{k=1}^K \widetilde p_{\mathrm{data},k}
\log\frac{\widetilde p_{\mathrm{data},k}}{\widetilde r_{\lambda,k}}\right).
\]
\item \textbf{Aggregation.} MAUVE is the area under the divergence frontier,
completed by the endpoints $(0,1)$ and $(1,0)$ and approximated
from finitely many $\lambda$. Its optimal value is $1$.
\end{enumerate}
MAUVE$^\star$ \citep{pillutla2023MAUVE} uses the same procedure, but replaces the empirical histograms by their Krichevsky--Trofimov smoothed versions,
\[
\widetilde p^\star_k=\frac{N_k+\frac12}{N+\frac K2},
\]
before computing the divergence frontier.
Interestingly, the $K$-means step itself is an optimal $K$-point quantization
problem and can equivalently be formulated as approximation in the
Wasserstein distance \citep{kreitmeier2011optimal}.

\subsection{Overview of Related Work}\label{sec:related work}
\subsubsection{Comparing Text Representations via Latent Geometries}
\paragraph{Latent Text Representations.}
Classical bag-of-words representations describe a document through its token
frequencies \citep{kusner2015word}. Modern transformers instead provide
context-dependent representations. In a causal transformer, the final hidden
state depends on the complete preceding sequence and can, under suitable
assumptions, even uniquely determine it \citep{nikolaou2026language}. In contrast, dedicated
text encoders such as BGE, GTE, and E5
\citep{bge-m3,li2023generaltextembeddingsmultistage,wang2024textembeddingsweaklysupervisedcontrastive}
are primarily optimized by contrastive objectives for semantic similarity,
retrieval, or classification and commonly compress a text into a single dense
vector. This is advantageous for semantic retrieval but less suited to our
goal of detecting arbitrary distributional changes: such representations are
designed to be invariant to semantically minor variations and necessarily
discard part of the token-level distributional structure. In contrast, we use
the complete collection of contextual states of a causal or bidirectional encoder.
\paragraph{Text-generation evaluation metrics.}
Reference-based metrics such as BERTScore \citep{zhang2020bertscore} and
MoverScore \citep{zhao2019moverscore} compare individual generated texts with
paired references using contextual token similarities or optimal transport,
and therefore do not directly measure the fit between two unpaired text
distributions. Distributional metrics address this explicitly.
Fr\'echet BERT Distance (FBD)
\citep{alihosseini2019jointly,xiang2021assessing} fits Gaussian distributions
to pretrained text representations and compares them by the Fr\'echet
distance. This restricts the comparison to first and second moments and
inherits the statistical difficulties of estimating high-dimensional
covariances. Indeed, the analogous FID estimator in vision is known to be
finite-sample biased and typically requires large sample sizes
\citep{heusel2017gans,chong2020effectively,berthet2026mind}.
PRD \citep{sajjadi2018precision,xiang2021assessing} separates quality and
coverage, but its practical estimator quantizes the latent distributions by
$K$-means. MAUVE \citep{pillutla2021MAUVE,pillutla2023MAUVE} similarly
quantizes pretrained language-model representations before comparing them
through divergence frontiers, introducing dependence on both sample and
cluster size. Moreover, its standard last-token representation can be
insensitive to perturbations earlier in the sequence
\citep{he-etal-2023-blind}. 
\cite{pynadath2026generativefrontiersevaluationmatters} argue that single-point generative perplexity can be misleading due to its strong dependence on output entropy, and show that entropy–perplexity generative frontiers as a more principled framework for evaluating diffusion language models.
Finally, Gen.\ PPL and token entropy are
sequence-wise averages rather than discrepancies between distributions and
can remain favorable even under complete corpus-level mode collapse.
\subsubsection{Optimal Transport}
\paragraph{Relation to document optimal transport.}
Previous work has been done on evaluating text via optimal transport. The Word Mover's Distance (WMD) \citep{kusner2015word} represents each document as a bag of pretrained word embeddings and compares two documents by optimal transport between their words. Hierarchical Optimal Topic Transport (HOTT) \citep{yurochkin2019hierarchical} is closer in spirit to our approach: documents are distributions over LDA topics, themselves distributions over words, and are compared by OT on the topic space with WMD as ground distance. Our goal is different: we compare \emph{distributions of documents} for generative-model evaluation, represent each document directly by its distribution of contextual transformer states rather than static word embeddings or learned topics, and use Double Sliced Wasserstein to avoid solving the resulting nested OT problem directly. Thus, while HOTT introduces a closely related hierarchical representation, our hierarchy operates at the corpus level and exploits contextual sequence representations together with slicing for efficient computation.

\paragraph{Optimal transport dataset distances.}
A related line of work compares entire datasets by optimal transport. OTDD
\citep{alvarezmelis2020geometric} represents each label by its
class-conditional feature distribution and uses the Wasserstein distance
between these distributions inside an outer OT problem. To reduce the cost of
the inner transports, the practical Bures-OTDD variant approximates each
class-conditional distribution by a Gaussian, yielding closed-form
$W_2$ distances between labels. More recently, s-OTDD
\citep{nguyen2025sotdd} replaces this nested computation by a sliced
construction based on one-dimensional projections. Unlike the moment-based Gaussian approximation in Bures-OTDD, DSW retains the full class-conditional distributions and has been observed to yield stronger agreement with exact nested transport distances. Note that DSW avoids such moment approximations and produces a stronger correlation with the exact nested transport distance than both Gaussian OTDD and s-OTDD \citep{piening2026slicing}.

\paragraph{Further work on nested optimal transport.}
More generally, several works develop tractable approximations for
Wasserstein-type distances between distributions of distributions. For
Gaussian mixtures, slicing has been developed from different geometric
viewpoints \citep{piening2025slicinggmm,nguyen2026summarizing,bonet2025busemann}.
\cite{PDS2026} further study sliced and linearized
Wasserstein surrogates inside Wasserstein-over-Wasserstein constructions,
while \cite{bonet2025flowing} combine
Wasserstein-over-Wasserstein geometry with tractable MMD objectives based on
Sliced-Wasserstein kernels.
 Applications of hierarchical optimal transport beyond text and dataset comparison include stereo vision \citep{Clerc2026}, graph comparison \citep{piening2026novel}
 and generative modeling  in imaging \citep{pmlr-v97-dukler19a, pmlr-v267-geuter25a} and biology \citep{yachimura2024scegot,haviv2025wasserstein}.
%%%%%%%%%%%%%%%%%%%%%%%%%%%%%%%%%%%%%%%%%%%%%%%%%%%%%%%%%%%%%%%%%%%%%%%%%%%%%%%%%%%%%%%%%%
\newpage
\section{Experimental Setup and Implementation}
\subsection{Datasets and Preprocessing}
We use four English text datasets, with preprocessing adapted to each
  experiment and the training format of the evaluated generators.

  \textbf{OpenWebText (OWT).}
  OpenWebText~\citep{Gokaslan2019OpenWeb} is a webtext corpus used for
  the diffusion-model benchmark, numerical-stability experiments, and
  controlled corruption tests. The main benchmark uses held-out
  GPT-2-tokenized sequences packed into blocks of $1{,}024$ tokens,
  retaining end-of-text boundaries. ELF is evaluated separately against
  a uniform sample from its released T5-tokenized OWT training corpus,
  decoded with native special tokens removed. This reference is
  training-derived rather than held out. Real--real controls compare
  disjoint subsets.

  \textbf{WebText.}
 We used the released WebText test \cite{openai2019gpt2output} split for the GPT-2 model-size,
  generation-length, temperature, and decoding experiments.
 For human-judgment comparison, we reuse MAUVE's released pairwise
  annotations~\citep{pillutla2021MAUVE} and regenerate continuations
  from the first $5{,}000$ test documents. Each prompt contains the
  first $35$ GPT-2 tokens, these are removed from both reference and
  generated texts before continuation-only evaluation, and empty
  generated continuations are discarded.

  \textbf{AG News.}
  AG News~\citep{NIPS2015_250cf8b5} contains news texts labeled
  \textsc{World}, \textsc{Sports}, \textsc{Business}, or \textsc{Sci/Tech}.
  We sample class-controlled mixtures to test mode removal and support
  expansion. Each candidate set contains $5{,}000$ unique examples and
  is disjoint from the reference set, we repeat sampling over five seeds.

  \textbf{LM1B.}
  The One Billion Word Benchmark \cite{lm1b} provides the reference for the MAUVE
  sample-size and cluster-count study. We construct $25{,}000$ sequences
  from its held-out test split using the generator's BERT tokenizer
  and $128$-token packing scheme, preserving the same special-token
  formatting as in generated samples.

  For metric evaluation, decoded texts are tokenized with the embedding
  model's native tokenizer and truncated to the experiment-specific
  horizon, padding positions are excluded. Word-level
  corruptions are applied before tokenization, whereas token-level
  positional perturbations operate on tokenized sequences.

\subsection{Encoding Models}
We evaluate SOL using several frozen transformer backbones.
GPT-2 Large \cite{radford2019language} is our default encoder and also
serves as the reference model for generative perplexity.
We additionally consider the causal models Qwen3-0.6B-Base
\cite{yang2025qwen3technicalreport} and OLMo-2-0425-1B
\cite{olmo20252olmo2furious}, as well as Dream-v0-Base-7B
\cite{ye2025Dream7bdiffusionlarge}, which uses bidirectional attention.
Each model is used to generate the hidden states of the last-layer on the evaluated
text, using its native tokenizer.
Since the representation spaces differ across the backbones, absolute SOL
values are only compared within the same encoder.
Additional comparisons with dedicated text encoders BGE-base
\cite{bge-m3}, GTE-base
\cite{li2023generaltextembeddingsmultistage}, E5-base
\cite{wang2024textembeddingsweaklysupervisedcontrastive}, and NeoBERT
\cite{breton2025neobertnextgenerationbert} are provided in the appendix.
\subsection{SOL Estimator and Hyperparameters}\label{sec:setup}

Unless stated otherwise, we use the estimator described in
Appendix~\ref{sec:dsw-numerical} with the settings in
Table~\ref{tab:sott-hyperparameters}. The same first- and second-level
slicing directions are used for all corpora within a comparison.
We report the ordinary SOL distance, i.e., the square root of the
squared estimator.

\begin{table}[h]
      \centering
      \small
      \begin{tabular}{lccl}
          \toprule
          Parameter & Dream-7B & All other encoders & Description \\
          \midrule
          $L_\xi$ & $8192$ & $1024$ & First-level projection directions \\
          $L_g$   & $8192$ & $1024$ & Second-level GP directions \\
          $M$ & \multicolumn{2}{c}{$64$} & Resolution of second-level functions \\
          $\lambda$ & \multicolumn{2}{c}{$0.1$} & RBF-GP length scale \\
          $\alpha$ & \multicolumn{2}{c}{$d$} & Dimension correction factor \\
          Projection seeds & \multicolumn{2}{c}{$3$} & Independent $\xi$ draws \\
          GP seeds & \multicolumn{2}{c}{$3$} & Independent second-level draws \\
          \bottomrule
      \end{tabular}
      \caption{
      Default SOL estimator parameters by encoder.
      Here, $d$ denotes the encoder's hidden dimension.
      We evaluate all $3\times3$ combinations of projection and GP seeds.
      Experiment-specific deviations are stated in the corresponding
      sections.
      }
      \label{tab:sott-hyperparameters}
  \end{table}

\subsection{Correlation with Human Judgments}
  \label{sec:Human Judges correlation}

  We evaluate agreement with the released human judgments from
  MAUVE~\citep{pillutla2021MAUVE}, using its WebText test data, eight
  model--sampling configurations, and raw pairwise Amazon Mechanical Turk
  annotations for \emph{human-like}, \emph{interesting}, and \emph{sensible}
  text. Following the released analysis notebook, we exclude responses with
  an estimated annotation time (\texttt{Answer.te}) of at most \(25\) seconds,
  retaining \(1{,}698\) of \(3{,}240\) annotation rows.

  \paragraph{Human preference scores.}
For each criterion, we fit a Bradley--Terry model to the pairwise
human judgments. Each configuration $i$ is assigned a latent score
$s_i$, and the probability that configuration $i$ is preferred over
configuration $j$ is modeled as
\[
    \Pr(i \succ j)
    =
    \frac{\exp(s_i)}
         {\exp(s_i)+\exp(s_j)}.
\]
The released pairwise annotations are treated as observed preference
outcomes under this model, and the scores are estimated by maximum
likelihood. 
Higher scores therefore correspond to stronger estimated human
preference.

  \paragraph{Models and sampling.}
  We evaluate four model sizes, each with two sampling methods:
  GPT-2 and GPT-2 Medium with nucleus sampling at \(p=0.9\) and ancestral
  sampling, and GPT-2 Large and GPT-2 XL with nucleus sampling at \(p=0.95\)
  and ancestral sampling. 

  While the annotations are still publicly available, the set of continuations seemed to be incomplete. Therefore, we resample all GPT-2 models with the corresponding sampling strategies. For each of the eight configurations, we regenerate continuations for
  the first \(5{,}000\) WebText test documents using five generation seeds.
  Generation uses the first \(35\) GPT-2 tokens as the prompt and a maximum
  total length of \(1{,}024\) tokens, including the prompt. In particular, we regenerate
  all configurations because the
  available released archive contains only GPT-2 and GPT-2 Large samples,
  whose metadata records \(10\)-token prompts, not $35$. Our rerun follows the
  paper's stated \(35\)-token prompt protocol.

  \paragraph{Text evaluation methods.}
  For each generation seed, we average the full-corpus score over all
  nine combinations of three token-projection seeds and three GP seeds.

  \paragraph{Rank correlation.}
  We measure agreement using Spearman rank correlation between the automatic
  metric scores and the Bradley--Terry scores across the eight configurations.
  We negate SOL scores, using \(-\mathrm{DSW}^2\), so that higher values
  indicate stronger human preference for every metric.

  We report the worst-case correlation \(\rho_{\min}\), following
  \citet{pillutla2023MAUVE}. For each configuration, we compute the mean
  and standard deviation of its metric score across the five generation
  seeds. We then independently shift each mean by either plus or minus
  one standard deviation and report the lowest Spearman correlation
  among the resulting  combinations.

\subsection{OpenWebText Benchmark Configurations}
\label{sec:owt-setup}

\paragraph{Data preprocessing.}
We use two reference families. The GPT-2 reference family (G) follows the OpenWebText preprocessing
of \citet{sahoo2024simple} and uses the final
100,000 OpenWebText documents as a held-out split. Documents are
GPT-2-tokenized, terminated with an end-of-text (EOT) token, concatenated
within preprocessing batches, and divided into 1,022-token chunks.
Adding an EOT at each end gives 1,024-token blocks, incomplete remainders
are discarded. The ELF family (E) \citep{hu2026elfembeddedlanguageflows} uses a uniform sample without replacement
from the released T5-tokenized OpenWebText training corpus.
ELF text decoding removes native special tokens, including EOS, PAD, and
UNK, and retains the release's whitespace cleanup. For the algorithm see Alg. \ref{alg:owt-preprocessing}.

For G, GPT-2-large consumes native IDs where available and retokenized
archived text otherwise, OLMo and Dream consume decoded text. Decoding
from available GPT-2 IDs preserves EOT markers and disables whitespace
cleanup. ELF text is retokenized for every evaluator. Evaluator inputs
are truncated to 1,024 tokens, with padding excluded, this horizon need
not equal 1,024 native generator tokens. Each model comparison uses 5,000
generated sequences and 5,000 rows selected from a 10,000-row reference
pool, real--real controls use disjoint subsets. SOL comparisons are made
within the same reference family and encoder. LDDM uses the same OWT
source and split with different packing-worker boundaries, LangFlow's
historical training split is not independently verified.
\begin{algorithm}[h]
\caption{OpenWebText reference preprocessing}
\label{alg:owt-preprocessing}
\begin{algorithmic}[1]
\Require Generator configuration $G$, SOL encoder $E$, evaluation horizon $H$
\Ensure Reference corpus matched to the generator's training distribution

\State Identify the OpenWebText source and preprocessing used to train $G$
\State Reproduce the corresponding tokenization, packing, and special-token conventions
\State Select held-out examples from the same data distribution whenever available
\If{the SOL encoder $E$ does not consume the generator's native representation}
    \State Decode the reference sequences to text using the generator-specific convention
    \State Retokenize the decoded text with the native tokenizer of $E$
\EndIf
\State Exclude padding positions from the token cloud
\State Use the same preprocessing pipeline for the generated and reference corpora
\end{algorithmic}
\end{algorithm}
 \paragraph{Evaluation and sampling budgets.}
  We denote discrete-diffusion sampling steps by $T$ and continuous
  predictor, integration, or flow-map steps by $S$. For BD3-LM, $T$
  is the maximum number of updates per block. NFE counts neural-network
  evaluations, including correction, initialization, cache-update,
  guidance, and final-decoding calls; hence $S$ and $T$ need not equal NFE.
  PC denotes predictor--corrector, ODE ordinary differential equation,
  and SDE stochastic differential equation. EMA denotes exponential
  moving-average weights, and KV denotes key--value caching.
  In the ELF and starred-model rows, CFG denotes self-conditioning
  guidance, called $w$ or $c$ in the original notation; CFG$=1$ means
  no guidance amplification. External-conditioning guidance is fixed
  to $1$. The parameter $\gamma$ denotes CTM noise mixing or ELF SDE
  churn, according to the sampler.
  We cross three direction seeds with three Gaussian-process seeds
  and report the mean and population standard deviation of the nine
  SOL estimates. Auxiliary entropy and perplexity retain the evaluator
  conventions specified in the table. Training data, model sizes,
  guidance, and sampling budgets differ across configurations.

  \paragraph{AR \cite{sahoo2024simple}.}
  The autoregressive model samples token by token at temperature $1.0$,
  requiring one next-token evaluation per decoding step. For the fixed
  1,024-token output with an initial BOS token, this gives 1,023
  next-token evaluations. The initial BOS token is not sampled.

  \paragraph{MDLM \cite{sahoo2024simple}.}
  The masked diffusion language model uses the cached denoising diffusion
  probabilistic model (DDPM) sampler (\texttt{ddpm\_cache}) with $T=1000$.
  We compare the released FP32 Gumbel sampling implementation with exact
  multinomial categorical sampling under the same sampler configuration.
  Both include final noise removal. Cached predictions can skip model
  calls, so $T=1000$ is not a fixed NFE count.

  \paragraph{DUO \cite{sahoo2025the}.}
  The base model uses ancestral sampling with $T=1024$.
  The distilled model uses $T=32$, with separate configurations for
  ancestral and greedy final noise removal.

  \paragraph{LDDM \cite{jo2026loopholing}.}
  The loopholing discrete diffusion models LDDM-M and LDDM-U use the
  released samplers with $T=1024$, EMA weights, FP64 categorical
  probabilities, and nucleus threshold $p=1.0$.
  A final denoising call gives 1,025 actual model evaluations.

  \paragraph{Soft-Masked \cite{hersche2026softmasked}.}
  The released scratch-trained checkpoint uses top-$3$ soft feedback and
  nucleus threshold $p=0.9$. We evaluate the MDLM sampler and the loop
  sampler of the remasking diffusion model (ReMDM-loop), both
  with $T=1024$ and final noise removal. ReMDM-loop uses
  $\eta=0.02$, $t_{\mathrm{on}}=0.55$, $t_{\mathrm{off}}=0.05$,
  and $\alpha_{\mathrm{on}}=0.9$. Caching can skip model calls, so actual
  NFE is sample dependent.

  \paragraph{ADLM \cite{rout2025anchored}.}
  The anchored diffusion language model uses the released large
  checkpoint with EMA weights. We generate 1,024 native GPT-2 tokens
  using ReMDM-loop with $T=1024$, nucleus threshold $p=0.9$,
  $\eta=0.02$, $t_{\mathrm{on}}=0.55$, $t_{\mathrm{off}}=0.05$,
  and $\alpha_{\mathrm{on}}=0.9$, followed by final argmax noise removal.
  The sampler caches predictions when its token state is unchanged;
  with generation batch size four, the number of model calls can depend
  on the batch. The evaluation uses reference family G.

  \paragraph{BD3-LM \cite{arriola2025block}.}
  Block discrete denoising diffusion language models generate blocks
  autoregressively and denoise within each block. We evaluate the
  released OWT checkpoints with block sizes $b\in\{4,16\}$, denoted
  BD3-LM block 4 and block 16. Each output has 1,024 native GPT-2 tokens,
  including the sampler's initial BOS token. We use the released
  first-hitting sampler with $T=5000$, nucleus threshold $p=0.9$,
  KV caching, and generation batch size one. First-hitting updates
  unmask one token at a time and stop when the current block is complete;
  $T=5000$ is therefore not 5,000 evaluations per block.
  The implementation additionally makes a model call to store the KV
  cache after completing each block. We retain the release's rejection
  rule: after the generated prefix exceeds 256 tokens, reject an attempt
  if the unigram entropy of its last 256 tokens is below $4$ nats, then
  retry. Thus, total generation cost also includes rejected attempts.
  There is no separate final decoder call. These checkpoints use
  reference family B, which retains document EOT separators but omits
  artificial special tokens at the boundaries of training chunks.

  \paragraph{vMF and VP \cite{CKS2026}.}
  The three von Mises--Fisher (vMF) PC configurations with $S=128$
  use the warped predictor--corrector sampler
  (\texttt{pc\_softmax\_warp}), with one corrector update every two
  predictor steps. The non-self-conditioned and base self-conditioned
  configurations use corrector step size $\epsilon=10^{-5}$;
  the additional self-conditioned configuration uses
  $\epsilon=2\times10^{-4}$. Predictor and corrector temperatures are
  both $1.0$. 

  For both vMF and variance-preserving (VP) models, we additionally
  report warped ODE sampling with $S=1024$ and warped
  predictor--corrector sampling with $S=1024$, using one corrector
  update after every predictor step. The latter therefore uses
  512 predictor and 512 corrector evaluations; its corrector step
  size is $10^{-5}$ for vMF and $10^{-3}$ for VP.
  These larger-budget ODE and PC runs use self-conditioning onset
  $0.5$. Finally, the self-conditioned SDE configurations use 128
  warped Euler--Maruyama steps, with noise amplitude $\sigma=0.01$
  for vMF and $\sigma=0.1$ for VP, decay exponent $q=1$, and
  self-conditioning onset $0$. Predictor steps exclude final readout.

  \paragraph{LangFlow \cite{chen2026langflowcontinuousdiffusionrivals}.}
  The released OWT checkpoint uses its native Euler--EDM solver, learned
  Gumbel schedule, and self-conditioning. We evaluate NFE budgets of
  128 and 1,024, each including the final model call used for argmax decoding.

  \paragraph{FLM and FMLM \cite{lee2026flowmaplanguagemodels}.}
  The flow language model (FLM) uses the Euler sampler
  (\texttt{flm\_euler}) with $S=1024$ and 1,024 recorded NFE,
  temperature $1.0$, nucleus threshold $p=1.0$, attention softcap $50$,
  and logit cap $30$, without an additional noise-removal step.
  The flow map language model (FMLM) uses the released flow-map
  sampler with $S\in\{1,2,4\}$, temperature $1.0$, and $\gamma=1$;
  the measured NFE equals the respective step count.

  \paragraph{FMLM$^{*}$ \cite{yoo2026selfconditionedflowmaplanguage}.}
  The released flow-map model distilled through the ELF$^{*}$ teacher
  uses consistency trajectory model (CTM) $\gamma$-sampling with
  $S\in\{1,2,5\}$, each evaluated at $\mathrm{CFG}\in\{1,5\}$.
  We use $\gamma=0.75$ for two steps and $\gamma=1$
  for one or five steps. We also report deterministic flow-map sampling
  with $S\in\{8,128\}$, CFG$=5$, and $\gamma=0$.
  The guidance scale is supplied as model conditioning. The released
  student uses a zero self-conditioning input at every step, without
  carrying predictions between steps. Sampling uses one model call per
  step and a separate final decoding call, giving NFE$=S+1$.
  All configurations use reference family G for SOL evaluation.

  \paragraph{ELF$^{*}$ \cite{yoo2026selfconditionedflowmaplanguage}.}
  The released fixed-point denoiser uses consistency deep equilibrium
  (CDEQ) distillation. We report its deterministic CDEQ sampler with
  $S=128$, $\mathrm{CFG}\in\{1,5\}$, and $\gamma=0$, corresponding to Euler
  integration without stochastic churn. Each step uses one student
  forward pass with zero self-conditioning input at the cold-start
  phase $\tau_0=0.002$, without an inner fixed-point iteration or
  cross-step self-conditioning. Final decoding requires one additional
  model call, giving NFE$=129$. The starred ELF and FMLM checkpoints use
  a GPT-2 tokenizer and GPT-2-large latent encoder; both are evaluated
  against reference G. Our starred-model runs use a deterministic
  logit-normal quantile grid, with location $-1.5$ and scale $0.8$.

  \paragraph{ELF \cite{hu2026elfembeddedlanguageflows}.}
  The original T5-based embedded language flows (ELF), with checkpoints
  ELF-B/M/L, are evaluated against reference family E.
  The 32-step ELF-B configurations use SDE sampling with $\gamma=1.5$
  and $\mathrm{CFG}\in\{1,3\}$.
  Additional ELF-B settings use ODE sampling with $S\in\{128,1024\}$
  and $\gamma=0$, and SDE sampling with $S=128$ and $\gamma=1$,
  all at CFG$=1$.

  We also evaluate ELF-B, ELF-M, and ELF-L with 64 integration steps.
  The shared configuration uses SDE sampling with
  $(\gamma,\mathrm{CFG})=(1,1)$ for all three sizes.
  The calibrated configurations use SDE $(0.5,0.5)$ for B,
  ODE $(0,1.25)$ for M, and SDE $(0.75,1)$ for L.
  These were selected in a 256-sample pilot by minimizing generative
  perplexity among settings within $0.02$ nats of the reference token
  entropy.

  These original ELF runs use randomly sampled, sorted logit-normal
  time grids with location $-1.5$ and scale $0.8$, and a separate final
  decoding call. The reported guidance scales are self-conditioning
  guidance; the external-conditioning guidance scale is fixed to $1$.
  Integration steps exclude any additional guidance and final-decoding
  evaluations.

  The separately retrained GPT-2 ELF-B teacher from the ELF$^{*}$ release
  uses reference family G. Its additional run uses ODE sampling with
  $S=128$, CFG$=1$, $\gamma=0$, and the deterministic logit-normal
  quantile grid, giving 128 denoising calls plus one final decoder call.

\subsection{Runtime}
  \begin{table}[h]
      \centering
      \small
      \setlength{\tabcolsep}{3pt}
      \begin{tabular}{c|cccccccccccc}
          \toprule
          $K=L$
          & 8 & 16 & 32 & 64 & 128 & 256 & 512 & 1024 & 2048 & 4096 & 8192 & 16384 \\
          \midrule
          Time [s]
          & 0.38 & 0.39 & 0.38 & 0.39 & 0.41 & 0.47 & 0.63 & 1.02 & 2.16 & 5.87 & 18.95 & 61.16 \\
          \bottomrule
      \end{tabular}
      \caption{Median post-hidden DSW runtime across three repeats on a
      single NVIDIA GeForce RTX 5090, with 5,000 texts per arm and $M=64$.
      Hidden states remain on GPU across direction blocks, the initial
      CPU-to-GPU transfer is included.}
      \label{tab:dsw-runtime-kl}
  \end{table}
We tested the runtime on a single NVIDIA GeForce RTX 5090 using 5,000 samples and report the median over three runs. The DSW runtime is given in Table~\ref{tab:dsw-runtime-kl}. Computing the hidden states takes $71$\,s for GPT-2 Large, $65$\,s for Qwen3-0.6B, $89$\,s for OLMo-2-1B, and $375.8$\,s for Dream-7B.
%%%%%%%%%%%%%%%%%%%%%%%%%%%%%%%%%%%%%%%%%%%%%%%%%%%%%%%%%%%%%%%%%%%%%%%%%%%%%%%%%%%%%%%%%%
\newpage
\section{Additional Ablations}
Because the reference models operate in hidden spaces of different dimensionality, absolute score values are not directly comparable across models, only relative comparisons within each reference model are meaningful.

\subsection{Number of Slices and Samples}\label{subsec:Dream Number of samples}
\begin{wrapfigure}[19]{r}{0.38\textwidth}
    \centering
    \vspace{-2em}

    % ------------------------------------------------
    % First figure
    % ------------------------------------------------
    \includegraphics[width=\linewidth]
        {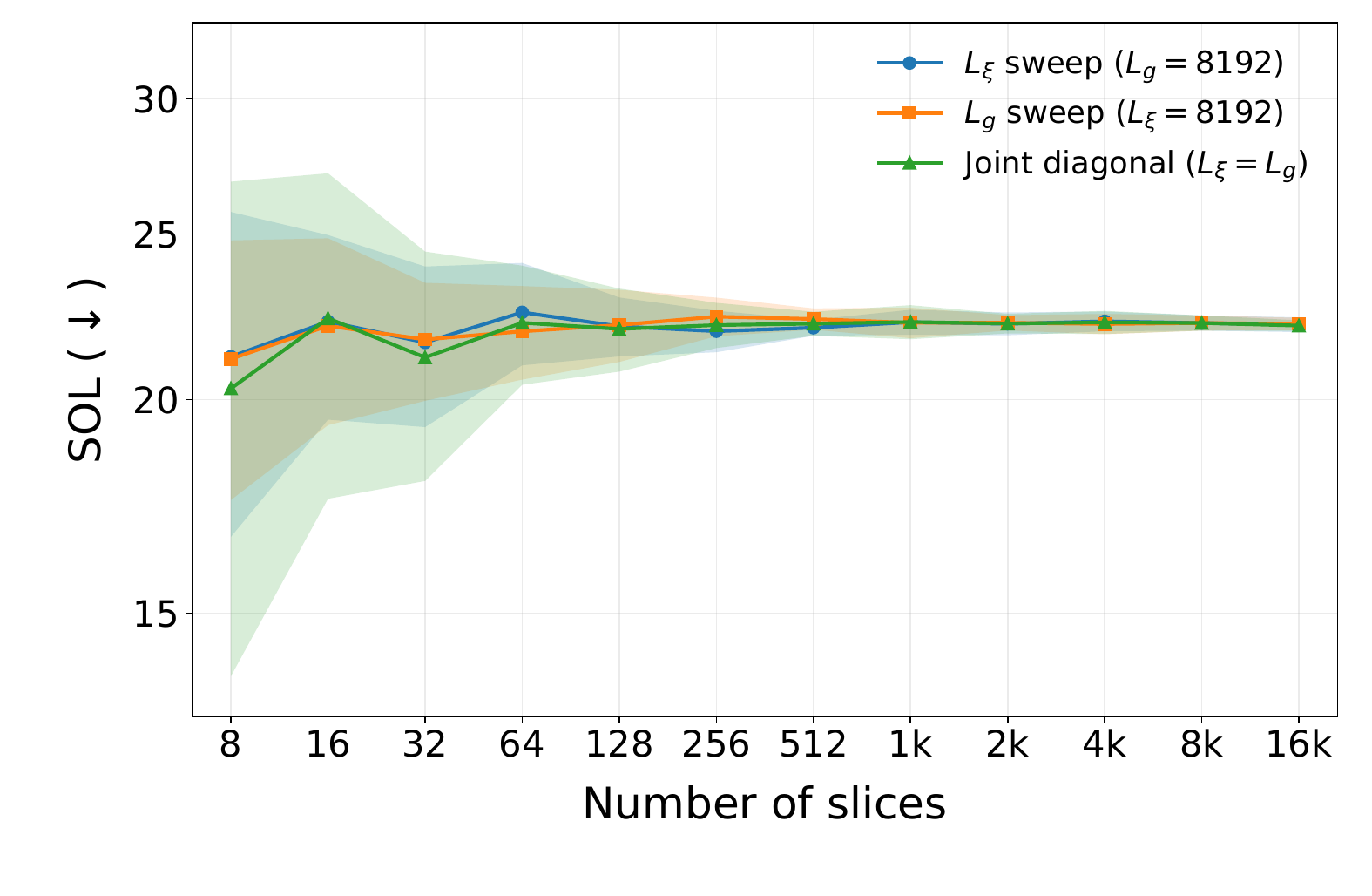}
    \vspace{-8mm}
    \captionof{figure}{\small
        Increasing the number of slices stabilizes SOL, using Dream as embedding model.
    }
    \label{fig:k-l-stabiltiy-Dream}

    \vspace{2mm}

    % ------------------------------------------------
    % Second figure
    % ------------------------------------------------
    \includegraphics[width=\linewidth]
        {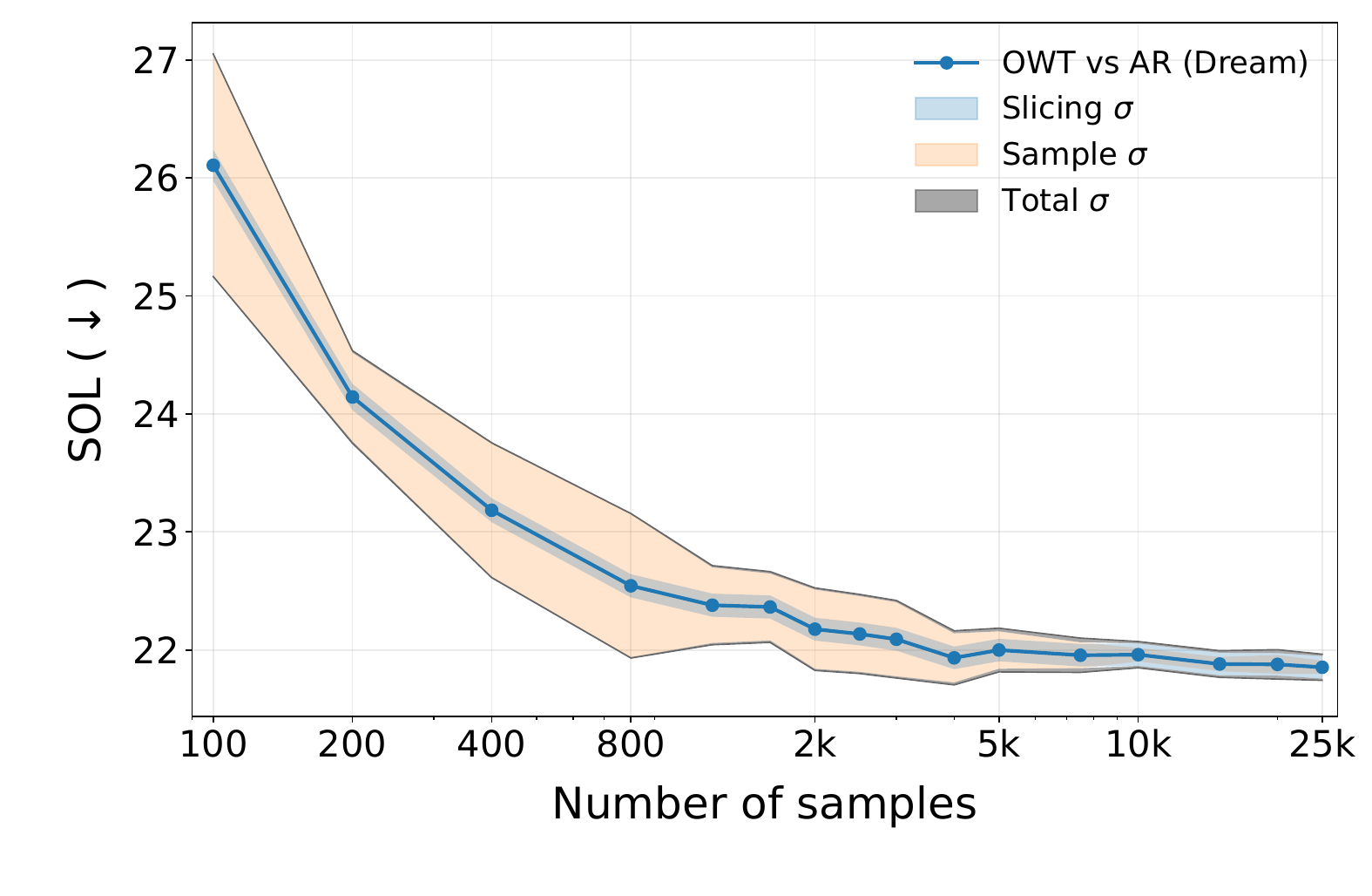}
    \vspace{-8mm}
    \captionof{figure}{\small
        SOL converges with sample size, with std over slices and samples, using Dream as embedding model.
    }
    \label{fig:sample-size-Dream}

    \vspace{-4mm}
\end{wrapfigure}
 As we have compared the GPT-2 model in Section \ref{subsec:Approximation}. We use Dream-v0-Base-7B as the embedding model and compare held-out
  OpenWebText (OWT) texts with autoregressive generations. For the
  slice-count experiment, we fix $5000$ texts per corpus, using
  generations from MDLM's autoregressive baseline and a maximum length
  of $1024$ Dream tokens. SOL stabilizes after a few hundred slices see Figure \ref{fig:k-l-stabiltiy-Dream},
  while variability across projection seeds decreases with increasing
  slice count. Varying the inner and outer slice counts separately or
  jointly yields consistent limiting values.

  For the sample-size experiment, we compared OWT texts with
the same AR model, using the same
  Dream embedding model with a maximum length of $256$ tokens. We fix
  $L_\xi=L_g=8192$ and vary the number of texts per corpus from $100$
  to $25{,}000$. We separate uncertainty due to random projections
  from uncertainty due to data subsampling, reporting both contributions
  alongside the total standard deviation, see Figure \ref{fig:sample-size-Dream}.

\subsection{Quantile Resolution}
\begin{wrapfigure}[12]{r}{0.38\textwidth}
    \centering
    \vspace{-1em}

    % ------------------------------------------------
    % First figure
    % ------------------------------------------------
    \includegraphics[width=\linewidth]
        {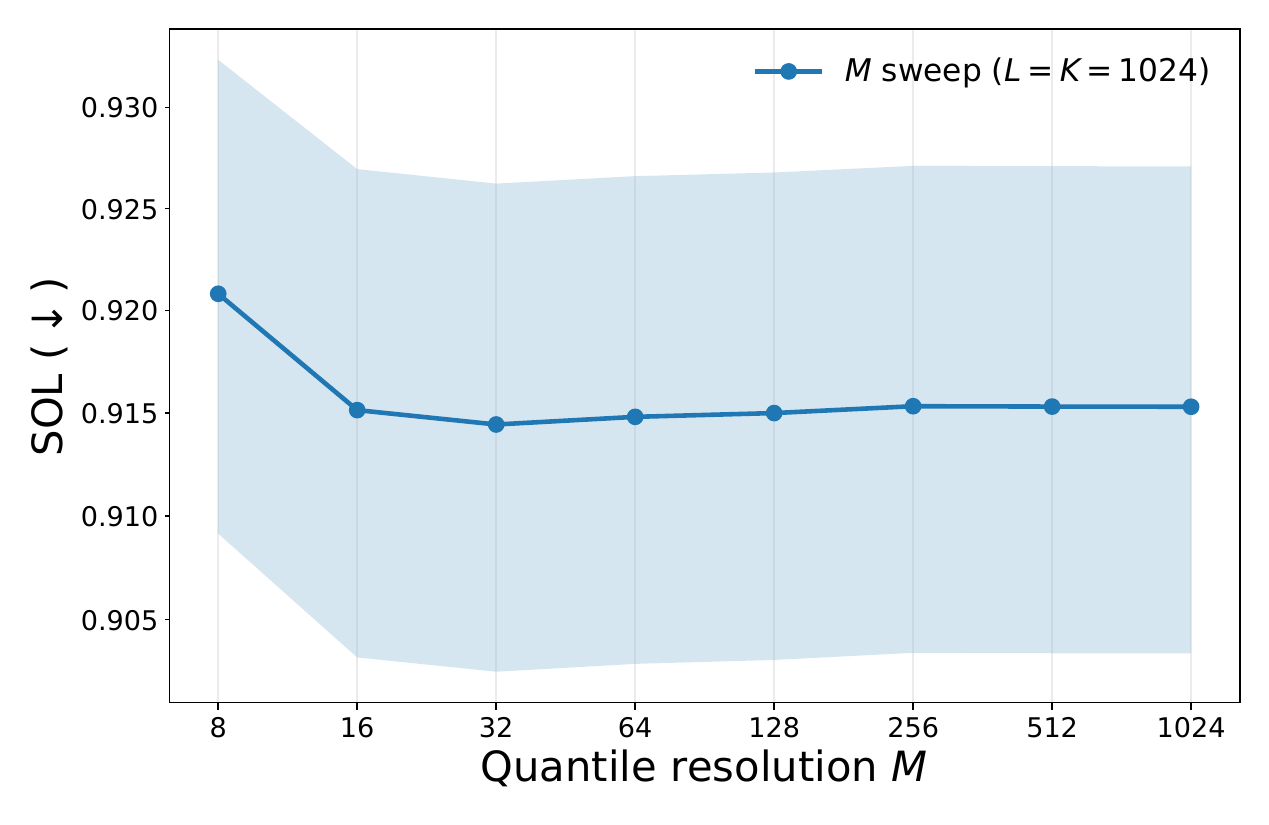}
    \vspace{-8mm}
    \captionof{figure}{\small
        Increasing $M$ for the quantile resolution has a small impact beyond $M=32$.
    }
    \label{fig:M-stabiltiy}

    \vspace{-4mm}
\end{wrapfigure}
We vary $M \in \{8,16,32,64,128,256,512,1024\}$ while fixing
  $L=K=1024$ and using the same 5,000 samples per distribution.
  We use GPT-2-large token embeddings with a 256-token cap and
  cross three direction seeds with eight GP seeds, coupling the
  draws across resolutions.
  Increasing $M$ has only a small effect once the resolution is
  sufficiently large: the plotted SOL values vary by less than
  $0.1\%$ for $M \geq 32$, see Figure \ref{fig:M-stabiltiy}.

  The shaded bands reflect variability across the 24 projection-seed
  pairs, rather than sample resampling. They show the mean squared
  SOL $\pm$ one standard deviation, followed by a square-root
  transformation. Since $L$ and $K$ remain fixed, increasing $M$
  improves quantile resolution without increasing the number of
  random slices, so this variability need not decrease.

\subsection{Gaussian Process Bandwidth}
The RBF length scale $\lambda$ controls the smoothness of the second-level slicing functions. For fixed grid resolution $M$, small $\lambda$ approaches independent Gaussian noise across quantile bins, whereas large $\lambda$ approaches a constant function and hence mean-SW. For $M=64$ and $\lambda=0.01$, adjacent bins have correlation $0.30$, which already drops to $0.008$ at a separation of two bins. At the opposite extreme, for $\lambda=3$, even the two endpoints of $[0,1]$ have correlation $0.95$, so the sampled functions are nearly constant. Our default $\lambda=0.1$ lies between these regimes: one correlation length spans approximately $M\lambda=6.4$ bins.

 \subsection{Mean-SW versus SOL under Structural Perturbations}
\label{app:mean_sw_structural}

As discussed in Appendix~\ref{sec:dsw_mean_sw}, restricting the second slicing
function to $g\equiv 1$ recovers mean-SW at the level of exact quantile
integrals. Mean-SW therefore depends only on the centroid of each token
cloud, whereas nonconstant second-stage functions additionally probe the
centered shape of its projected token distribution. Since mean-SW is already
a strong baseline, we ask whether retaining this additional token-cloud
information increases the separation from the finite-sample real-text
baseline under controlled structural changes.

For each perturbation $c$, we normalize both metrics by their respective
clean real--real distance,
\begin{equation}
    R_D(c)
    :=
    \frac{D(P,Q_c)}
         {D(P,Q_{\mathrm{clean}})}.
    \label{eq:clean_normalized_response}
\end{equation}
We use the ordinary distance, i.e.\ take the square root of the squared
estimator before forming the ratio, and use the same projections in the
numerator and denominator. For the passage-reorganization experiments,
$Q_{\mathrm{clean}}$ denotes the corresponding unperturbed passage
construction.

We repeat the experiment on five disjoint pairs of $5{,}000$ reference and
$5{,}000$ candidate documents, using GPT-2 Large, a horizon of $128$ tokens,
and $L_\xi=L_g=1024$. We use three first-level projection seeds and three GP
seeds for SOL. To compare the relative responses, we report
\begin{equation}
    \Delta(c)
    :=
    \frac{1}{5}\sum_{b=1}^{5}
    \left(
        \mathbb{E}_{\text{sliced-seed}}\log R_{\mathrm{SOL}}(c)
        -
        \mathbb{E}_{\text{sliced-seed}}\log R_{\mathrm{mean\text{-}SW}}(c)
    \right).
\end{equation}
Thus, $e^{\Delta(c)}$ is the geometric-mean relative response of SOL
compared with mean-SW. 

\begin{wraptable}[17]{r}{0.51\textwidth}
    \centering
    \vspace{-1.2\baselineskip}
    \footnotesize
    \setlength{\tabcolsep}{3.2pt}
    \begin{tabular}{@{}p{0.255\textwidth}rrc@{}}
        \toprule
        Perturbation
        & mean-SW
        & SOL
        & $\Delta$ \\
        \midrule
        Selected-word shuffle
        & 10.66 & \textbf{15.75}
        & $0.390$ \\

        Adjacent word swaps
        & 7.09 & \textbf{9.63}
        & $0.305$ \\

        Short-window shuffle
        & 5.63 & \textbf{7.61}
        & $0.301$ \\

        Reverse short spans
        & 7.77 & \textbf{10.81}
        & $0.329$ \\

        Clause permutation
        & 2.61 & \textbf{3.36}
        & $0.252$ \\

        Sentence permutation
        & 2.12 & \textbf{2.60}
        & $0.203$ \\

        Move a phrase
        & 2.14 & \textbf{2.83}
        & $0.277$ \\

        Duplicate/delete repetition
        & \textbf{8.83} & 8.62
        & $-0.024$ \\

        Reorganize passage mixtures
        & 5.82 & \textbf{8.79}
        & $0.413$ \\

        Reorganize repetitions
        & 9.12 & \textbf{10.47}
        & $0.139$ \\
        \bottomrule
    \end{tabular}
     \caption{Clean-normalized ordinary-distance responses of mean-SW and
    SOL under structural perturbations. Response columns are arithmetic
    means of paired ratios. $\Delta$ is the mean paired log-response
    contrast.}
    \label{tab:mean_sw_structural}
\end{wraptable}

SOL has the larger clean-normalized response for nine of the ten
perturbations, see Table \ref{tab:mean_sw_structural}. The difference is strongest for changes in word order and
larger-scale passage organization. For selected-word shuffling,
$\Delta=0.390$, corresponding to
$e^\Delta=1.48$, i.e., a $1.48\times$ larger geometric-mean relative
response for SOL. In contrast, duplicate/delete repetition produces a
slightly larger relative response for mean-SW. Hence, the additional
token-cloud information does not uniformly amplify all perturbations, but
primarily increases the response to changes in within-document contextual
structure.

Since SOL is generally more sensitive to structural perturbations while adding little runtime beyond hidden-state computation, we use SOL in the main experiments

\subsection{Sensitivity to Sampling Temperature}\label{sec:Sensitivity to Sampling Temperature}
\begin{wraptable}[11]{r}{0.63\textwidth}
      \centering
      \vspace{-1.5em}
      \resizebox{0.58\textwidth}{!}{%
      \begin{tabular}{l r ccccc}
          \toprule
          Encoder & Length
          & $t=0.7$ & $t=0.9$ & $t=1.0$ & $t=1.1$ & $t=1.3$ \\
          \midrule
          GPT-2-large
          & 512  & 2.901 & 1.066 & \textbf{0.994} & 1.038 & 1.355 \\
          & 1024 & 2.872 & 1.091 & \textbf{1.014} & 1.070 & 1.434 \\
          \midrule
          Qwen3-0.6B
          & 512  & 6.755 & \textbf{5.085} & 5.204 & 5.440 & 6.332 \\
          & 1024 & 6.779 & \textbf{5.179} & 5.257 & 5.565 & 6.538 \\
          \midrule
          OLMo-2-1B
          & 512  & 7.719 & \textbf{4.051} & 4.864 & 5.947 & 8.211 \\
          & 1024 & 8.562 & \textbf{4.103} & 4.758 & 6.051 & 8.561 \\
          \midrule
          Dream-7B
          & 512  & 27.769 & 21.576 & \textbf{21.090} & 22.066 & 25.889 \\
          & 1024 & 27.229 & 21.495 & \textbf{21.091} & 21.980 & 25.790 \\
          \bottomrule
      \end{tabular}%
      }
      \caption{\small SOL evaluating generated text by GPT-2 Large at different decoding temperatures using differnet embedding models.}
      \label{tab:normal_temperature}
  \end{wraptable}

We evaluate GPT-2-large samples generated at five sampling temperatures
  against WebText, see Table \ref{tab:normal_temperature} .
  All encoders behave as expected, assigning lower SOL scores to
  intermediate temperatures, with slightly different optima:
  GPT-2-large and Dream prefer $t=1.0$, while Qwen3-0.6B and OLMo-2-1B
  prefer $t=0.9$.

\paragraph{Sampling temperature for diffusion models.}
Table~\ref{tab:owt-temperature-sweeps} compare temperature scaling for
vMF~\citep{CKS2026}, Soft-Masked with ReMDM-loop~\citep{hersche2026softmasked},
and MDLM~\citep{sahoo2024simple}. We evaluate SOL using both
Dream-7B~\citep{ye2025Dream7bdiffusionlarge} and
GPT-2 Large~\citep{radford2019language}.
For all three generators, both encoders favor temperature $1.0$.
Importantly, lower temperatures substantially reduce token entropy and can
strongly improve Gen.\ PPL, while producing samples that are farther from the
data distribution according to SOL. Increasing the temperature from this
low-entropy regime therefore improves distributional fit, until excessive
temperature again degrades SOL. This illustrates that improvements in
Gen.\ PPL obtained by reducing sample entropy need not correspond to better
generation.

\begin{table*}[t]
\centering
\footnotesize

% ------------------------------------------------------------------
% Dream
% ------------------------------------------------------------------
\begin{minipage}[t]{0.49\textwidth}
  \centering
  \textbf{(a) Dream-7B}
  \vspace{2pt}

  \resizebox{\linewidth}{!}{%
    \begin{tabular}{@{}clrrr@{}}
      \toprule
      Temp. & Metric
      & vMF
      & \shortstack{Soft-Masked\\+ ReMDM}
      & MDLM \\
      \midrule

      $0.7$ & SOL
      & $48.6981 \pm 0.292$
      & $57.6332 \pm 0.420$
      & $54.3128 \pm 0.404$ \\
      & {\scriptsize Gen.\ PPL ($H$)}
      & {\scriptsize $7.92\;(3.950)$}
      & {\scriptsize $3.24\;(3.490)$}
      & {\scriptsize $7.26\;(3.884)$} \\

      \addlinespace[1pt]
      $0.8$ & SOL
      & $37.9781 \pm 0.160$
      & $41.3254 \pm 0.303$
      & $42.0520 \pm 0.244$ \\
      & {\scriptsize Gen.\ PPL ($H$)}
      & {\scriptsize $15.28\;(4.656)$}
      & {\scriptsize $5.33\;(4.360)$}
      & {\scriptsize $14.18\;(4.657)$} \\

      \addlinespace[1pt]
      $0.9$ & SOL
      & $33.1134 \pm 0.076$
      & $24.7670 \pm 0.161$
      & $37.4379 \pm 0.149$ \\
      & {\scriptsize Gen.\ PPL ($H$)}
      & {\scriptsize $26.36\;(5.054)$}
      & {\scriptsize $9.18\;(4.982)$}
      & {\scriptsize $27.01\;(5.093)$} \\

      \addlinespace[1pt]
      $1.0$ & SOL
      & $\mathbf{31.4917 \pm 0.070}$
      & $\mathbf{18.9293 \pm 0.121}$
      & $\mathbf{36.3186 \pm 0.091}$ \\
      & {\scriptsize Gen.\ PPL ($H$)}
      & {\scriptsize $45.49\;(5.319)$}
      & {\scriptsize $16.85\;(5.343)$}
      & {\scriptsize $59.15\;(5.444)$} \\

      \addlinespace[1pt]
      $1.1$ & SOL
      & $34.4412 \pm 0.172$
      & $23.4306 \pm 0.082$
      & $40.4199 \pm 0.159$ \\
      & {\scriptsize Gen.\ PPL ($H$)}
      & {\scriptsize $79.85\;(5.507)$}
      & {\scriptsize $51.29\;(5.691)$}
      & {\scriptsize $157.49\;(5.766)$} \\

      \addlinespace[1pt]
      $1.2$ & SOL
      & $42.6649 \pm 0.252$
      & $86.6133 \pm 0.505$
      & $49.8600 \pm 0.299$ \\
      & {\scriptsize Gen.\ PPL ($H$)}
      & {\scriptsize $133.33\;(5.618)$}
      & {\scriptsize $3364.94\;(6.544)$}
      & {\scriptsize $540.14\;(6.098)$} \\

      \addlinespace[1pt]
      $1.3$ & SOL
      & $53.8943 \pm 0.305$
      & $111.7374 \pm 0.666$
      & $62.8243 \pm 0.408$ \\
      & {\scriptsize Gen.\ PPL ($H$)}
      & {\scriptsize $206.14\;(5.654)$}
      & {\scriptsize $14157.05\;(6.769)$}
      & {\scriptsize $2096.21\;(6.425)$} \\

      \bottomrule
    \end{tabular}%
  }
\end{minipage}
\hfill
% ------------------------------------------------------------------
% GPT-2 Large
% ------------------------------------------------------------------
\begin{minipage}[t]{0.49\textwidth}
  \centering
  \textbf{(b) GPT-2 Large}
  \vspace{2pt}

  \resizebox{\linewidth}{!}{%
    \begin{tabular}{@{}clrrr@{}}
      \toprule
      Temp. & Metric
      & vMF
      & \shortstack{Soft-Masked\\+ ReMDM}
      & MDLM \\
      \midrule

      $0.7$ & SOL
      & $3.7874 \pm 0.096$
      & $4.9410 \pm 0.148$
      & $4.1873 \pm 0.108$ \\
      & {\scriptsize Gen.\ PPL ($H$)}
      & {\scriptsize $7.92\;(3.950)$}
      & {\scriptsize $3.24\;(3.490)$}
      & {\scriptsize $7.26\;(3.884)$} \\

      \addlinespace[1pt]
      $0.8$ & SOL
      & $2.3689 \pm 0.069$
      & $3.0323 \pm 0.102$
      & $2.8273 \pm 0.075$ \\
      & {\scriptsize Gen.\ PPL ($H$)}
      & {\scriptsize $15.28\;(4.656)$}
      & {\scriptsize $5.33\;(4.360)$}
      & {\scriptsize $14.18\;(4.657)$} \\

      \addlinespace[1pt]
      $0.9$ & SOL
      & $1.6790 \pm 0.042$
      & $1.2886 \pm 0.050$
      & $2.2392 \pm 0.044$ \\
      & {\scriptsize Gen.\ PPL ($H$)}
      & {\scriptsize $26.36\;(5.054)$}
      & {\scriptsize $9.18\;(4.982)$}
      & {\scriptsize $27.01\;(5.093)$} \\

      \addlinespace[1pt]
      $1.0$ & SOL
      & $\mathbf{1.5109 \pm 0.022}$
      & $\mathbf{0.6075 \pm 0.013}$
      & $\mathbf{2.1675 \pm 0.029}$ \\
      & {\scriptsize Gen.\ PPL ($H$)}
      & {\scriptsize $45.49\;(5.319)$}
      & {\scriptsize $16.85\;(5.343)$}
      & {\scriptsize $59.15\;(5.444)$} \\

      \addlinespace[1pt]
      $1.1$ & SOL
      & $1.9666 \pm 0.030$
      & $1.2774 \pm 0.022$
      & $2.8280 \pm 0.036$ \\
      & {\scriptsize Gen.\ PPL ($H$)}
      & {\scriptsize $79.85\;(5.507)$}
      & {\scriptsize $51.29\;(5.691)$}
      & {\scriptsize $157.49\;(5.766)$} \\

      \addlinespace[1pt]
      $1.2$ & SOL
      & $2.9077 \pm 0.049$
      & $6.4263 \pm 0.090$
      & $4.0028 \pm 0.055$ \\
      & {\scriptsize Gen.\ PPL ($H$)}
      & {\scriptsize $133.33\;(5.618)$}
      & {\scriptsize $3364.94\;(6.544)$}
      & {\scriptsize $540.14\;(6.098)$} \\

      \addlinespace[1pt]
      $1.3$ & SOL
      & $4.0478 \pm 0.070$
      & $8.5170 \pm 0.122$
      & $5.3756 \pm 0.079$ \\
      & {\scriptsize Gen.\ PPL ($H$)}
      & {\scriptsize $206.14\;(5.654)$}
      & {\scriptsize $14157.05\;(6.769)$}
      & {\scriptsize $2096.21\;(6.425)$} \\

      \bottomrule
    \end{tabular}%
  }
\end{minipage}

\caption{\footnotesize
Temperature sweep on OpenWebText for vMF, Soft-Masked with ReMDM-loop,
and MDLM. Panels (a) and (b) report SOL using Dream-7B and GPT-2 Large,
respectively. Gen.\ PPL is evaluated with GPT-2 Large and $H$ denotes
mean per-sequence token-frequency entropy in nats; these quantities are
independent of the SOL encoder. SOL uses 5,000 samples per side and a
1,024-token horizon, with mean $\pm$ population SD over nine seed pairs.
Lower SOL is better and bold marks its column minimum. }
\label{tab:owt-temperature-sweeps}
\end{table*}

\subsection{Pollution and Corruption}\label{sec:encoder_robustness}
\begin{figure}[h]
    \centering
    \includegraphics[width=0.85\linewidth]{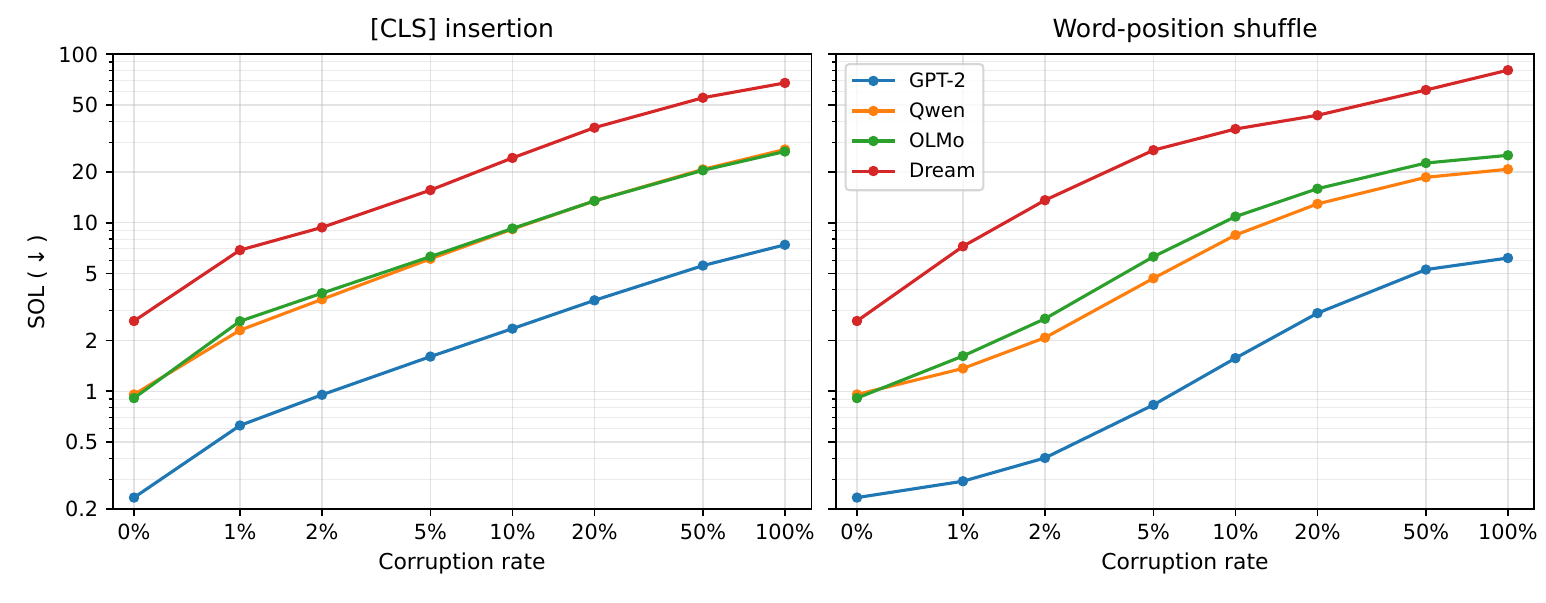}
    \caption{SOL under increasing \texttt{[CLS]} insertion (left) and word-position shuffling (right) for four encoders. For all four language-model encoders, SOL increases monotonically with the corruption rate, and the two corruptions produce responses of similar magnitude. In contrast to the text encoders in Figure~\ref{fig:text-embedding_supp}, no encoder is markedly less sensitive to shuffling than to insertion.}
    \label{fig:text-embedding}
\end{figure}

\begin{wraptable}{r}{0.43\textwidth}
    \vspace{-0.5\baselineskip}
    \centering
    \small
    \setlength{\tabcolsep}{5pt}
    \renewcommand{\arraystretch}{0.95}
    \begin{tabular}{@{}ccc@{}}
        \toprule
        Corruption & \texttt{[CLS]} pollution & Word shuffling \\
        \midrule
          0\% & 0.9702 & 0.9702 \\
          1\% & 0.9367 & 0.9701 \\
          2\% & 0.9009 & 0.9667 \\
          5\% & 0.7888 & 0.9527 \\
         10\% & 0.6559 & 0.8423 \\
         20\% & 0.4558 & 0.3555 \\
         50\% & 0.1977 & 0.0353 \\
        100\% & 0.0545 & 0.0180 \\
        \bottomrule        
    \end{tabular}
    \caption{\small{MAUVE under \texttt{[CLS]} pollution
    and word shuffling using last hidden representation.}}
    \label{tab:MAUVE_corruption}
    \vspace{-1.0\baselineskip}
\end{wraptable}
We further investigate whether the response of DSW to textual perturbations is specific to a particular language-model representation. To this end, we repeat the corruption experiment using four substantially different backbone models: GPT-2-large, Qwen3-0.6B-Base, OLMo-2-1B, and Dream-v0-Base-7B. For all encoders, we use the same split of OpenWebText consisting of $5{,}000$ clean reference documents and $5{,}000$ disjoint candidate documents. Thus, differences between the resulting curves arise from the representation induced by the encoder rather than from different underlying documents.

We consider two types of corruption. For the \emph{insertion} experiment, literal \texttt{[CLS]} strings are inserted at randomly chosen whitespace-word boundaries, where the corruption rate denotes the number of insertions relative to the original number of words. For the \emph{shuffle} experiment, the prescribed fraction of whitespace-word positions is selected and the corresponding words are randomly permuted. Hence, shuffling preserves the word multiset while increasingly disrupting local and global word order. In both cases, we use corruption rates
\[
    0\%,\,1\%,\,2\%,\,5\%,\,10\%,\,20\%,\,50\%,\,100\%.
\]

Importantly, the corruption is applied to the raw text before tokenization. Each corrupted document is subsequently processed using the respective model's own tokenizer and truncated to a horizon of $256$ tokens. We construct the token cloud from the final hidden-state representations without pooling. Consequently, the textual corruption is identical across backbones, whereas differences induced by tokenization, hidden dimension, and truncation remain part of the encoder comparison.

This experiment therefore shows that the continuous corruption response of
DSW is not specific to a single representation space. We emphasize that the
four-backbone comparison itself is DSW-only, the MAUVE results in
Table~\ref{tab:MAUVE_corruption} serve as a separate metric comparison and
were not recomputed for each of the four encoder backbones.

\subsection{Positional Sensitivity}\label{sec: Different Pollution Location}
 \begin{wraptable}[7]{r}{0.35\textwidth}
      \centering
      \vspace{-2.5em}
      \small
      \setlength{\tabcolsep}{3pt}
      \begin{tabular}{@{}lrrr@{}}
        \toprule
        Encoder & Clean & Prefix & Suffix \\
        \midrule
        OLMo-2-1B   & 0.870 & 2.048 & 0.913 \\
        Dream-7B    & 2.493 & \textbf{4.496} & \textbf{4.621} \\
        Qwen3-0.6B  & 0.914 & 1.937 & 0.946 \\
        GPT-2-large & 0.224 & 0.340 & 0.226 \\
        \bottomrule
      \end{tabular}
      \caption{\small{Shuffle
      20\% of the tokens the first or last 5\% of the text.}}
      \label{tab:standard-regional-sensitivity}
  \end{wraptable}

 We shuffle 20\% of the tokens within either the first or last 5\% of
  each sequence.
  Table~\ref{tab:standard-regional-sensitivity} shows a pronounced prefix
 sensitivity for the three causal encoders: prefix shuffling raises their
  scores substantially, while suffix shuffling leaves them close to the
  clean baseline. 
  
  In contrast, the bidirectional Dream encoder responds
  similarly at both locations, with a slightly larger suffix response.
  This is consistent with causal contextualization that amplifies early edits
  through subsequent token states. The comparison concerns positional
  balance within each encoder, since absolute score scales differ across
  encoders.

\subsubsection{Positional bias of last-token SW}
  We evaluate sliced Wasserstein distance on GPT-2-large's final-token hidden states using 5,000 disjoint reference and 5,000 candidate, a 256-token horizon, and 1,024 projection directions.
  We replace or shuffle ten tokens at the beginning, middle, or end, averaging over five data/transformation seeds and three projection seeds.
  For random replacement, the dimension-scaled SW distance increases from a clean baseline of $0.97$ to $2.55$, $3.75$, and $23.92$ for beginning,
  middle, and end corruption, respectively.
  For shuffling, the corresponding scores are $3.50$, $2.65$, and $23.56$.
  After subtracting the clean baseline, end corruption produces approximately $14.5\times$ and $8.9\times$ the beginning-corruption response for
  replacement and shuffling, respectively, demonstrating a strong positional bias toward the final tokens. This means it suffers the same problem as MAUVE, see Figure \ref{fig:MAUVE-random-token-begining}.

\subsection{Mode Coverage}
\begin{wraptable}[8]{r}{0.36\textwidth}
  \centering
  \vspace{-1em}
  \small
  \begin{tabular}{lcc}
    \toprule
    Metric & \(Q_2\!\rightarrow Q_1\) & \(Q_2\!\rightarrow Q_3\) \\
    \midrule
    SOL     & 1.000 & 1.000 \\
    mean-SW & 1.000 & 0.9995 \\
    \bottomrule
  \end{tabular}
  \caption{\small{Spearman correlations with the target share along the
  \(Q_2\!\rightarrow Q_1\) and \(Q_2\!\rightarrow Q_3\) interpolation paths.}}
  \label{tab:agnews-boundary-interpolation}
\end{wraptable}

 AG News \citep{NIPS2015_250cf8b5} contains four classes: \textsc{World}, \textsc{Sports},
  \textsc{Business}, and \textsc{Sci/Tech}. We define the reference distribution
  as \(P=\tfrac{1}{2}\textsc{World}+\tfrac{1}{2}\textsc{Sports}\). The three
  candidate distributions are \(Q_1=\textsc{World}\), representing missing
  diversity, \(Q_2=\tfrac{1}{2}\textsc{World}+\tfrac{1}{2}\textsc{Sports}\),
  which matches the class distribution of \(P\) but contains disjoint examples,
  and \(Q_3=\tfrac{1}{4}(\textsc{World}+\textsc{Sports}+\textsc{Business}
  +\textsc{Sci/Tech})\), representing additional mass outside the reference
  support.

  Holding \(P\) fixed, we interpolate independently from \(Q_2\) toward \(Q_1\)
  and \(Q_3\):
  \[
  Q_{2\rightarrow j}(\lambda)
    =(1-\lambda)Q_2+\lambda Q_j,
    \qquad j\in\{1,3\}.
  \]
  We evaluate
  \(\lambda\in\{0,.01,.02,.05,.10,.20,\ldots,.90,.95,.98,.99,1\}\).
  Each arm contains \(5{,}000\) unique examples and is disjoint from the
  \(5{,}000\)-example reference set. Results are averaged over five independently
  sampled dataset seeds.

  SOL increase monotonically away from the matched distribution \(Q_2\).
  The two paths represent distinct errors: \(Q_1\) removes the \textsc{Sports}
  mode, whereas \(Q_3\) introduces the previously absent \textsc{Business} and
  \textsc{Sci/Tech} modes.
  For mean-sliced Wasserstein, the rank correlation are not always monotone in this test, see Table \ref{tab:agnews-boundary-interpolation}. 
  
\subsection{Alternative to Bidirectional Attention}\label{sec:Alternative to Bidirectional Attention}

\begin{wraptable}{r}{0.58\textwidth}
    \centering
    \vspace{-1.em}
    \resizebox{0.58\textwidth}{!}{%
    \begin{tabular}{l r ccccc}
      \toprule
      Encoder & Length
      & $t=0.7$ & $t=0.9$ & $t=1.0$ & $t=1.1$ & $t=1.3$ \\
      \midrule
      GPT-2-large
      & 512  & 2.141 & 1.299 & 1.138 & 1.035 & \textbf{1.032} \\
      & 1024 & --    & --    & --    & --    & -- \\
      \midrule
      Qwen3-0.6B
      & 512  & 6.083 & 4.310 & \textbf{4.288} & 4.347 & 4.847 \\
      & 1024 & 6.695 & 4.513 & \textbf{4.470} & 4.645 & 5.311 \\
      \midrule
      OLMo-2-1B
      & 512  & 7.822 & 5.452 & 5.016 & \textbf{4.808} & 5.203 \\
      & 1024 & 9.413 & 6.440 & 5.421 & 4.916 & \textbf{4.884} \\
      \bottomrule
    \end{tabular}%
    }
    \caption{\small SOL scores using the
    $[\mathbf x,\mathbf x]$ second-pass representation
    across sampling temperatures.}
    \label{tab:aa_temperature}
  \end{wraptable}

As seen in Section~\ref{sec: Different Pollution Location}, causal transformers prefer changes in the beginning. Bi-directional attention circumvents this problem, but we also test two alternative proposed methods. As proposed in \cite{springer2025repetition} (also used in \cite{Ewais_2026}), one can also repeat the whole sequence $\mathbf x$ and pass  $[\mathbf x, \texttt{<sep>},\mathbf x]$ to make each token conditioned on the whole sequence. %The second options is to reweight the sequence. 

\paragraph{Repeat the sequence.}
  We compare normal encoding $[\mathbf x]$ with repetition $[\mathbf x,\mathbf x]$ and
  EOS-separated repetition $[\mathbf x,\texttt{<sep>},\mathbf x]$ in Table \ref{tab:olmo-repeated-temperature}, using OLMo-2-1B.
  Repeated inputs are encoded in one causal forward pass, and only
  second-copy final-layer states enter the uniform token measure, The EOS
 remains visible to attention, but is not scored and does not reset attention.
   \begin{wraptable}[13]{r}{0.35\textwidth}
    \centering
    \small
    \vspace{-0.em}
    \setlength{\tabcolsep}{3pt}
    \begin{tabular}{@{}lrrr@{}}
      \toprule
      Region & $[\mathbf x]$ & $[\mathbf x,\mathbf x]$ & $[\mathbf x,\texttt{<sep>},\mathbf x]$ \\
      \midrule
      Clean  & 0.910 & 0.732 & 0.853 \\
      Prefix & 1.710 & 1.167 & 1.926 \\
      Middle & 1.335 & 1.347 & 1.375 \\
      Suffix & 0.942 & 1.599 & 0.958 \\
      \bottomrule
    \end{tabular}
    \caption{\small OLMo regional corruption, cap 1,024.
    Mean $\mathrm{SOL}$ for 10\% replacement within a 5\% region,
    taking the square root per seed before averaging,
    clean scores use uncorrupted candidate text.}
    \label{tab:olmo-repeated-regions}
  \end{wraptable}
  All scores average three projection seeds crossed with three GP seeds.
  For regional corruption, we use disjoint 5,000-document reference and
  candidate sets and match the corrupted tokens across representations.

  Table~\ref{tab:olmo-repeated-regions} compares identical corruption
  strengths applied to the prefix, middle, or suffix.
  With standard encoding, prefix corruption increases the score by
  \(88\%\) relative to the clean baseline, compared with only \(4\%\)
  for suffix corruption.
  Repetition without a separator makes these responses more similar
  but reverses their ordering: prefix corruption increases the score
  by \(59\%\), whereas suffix corruption increases it by \(118\%\),
  roughly twice as much.
  EOS-separated repetition retains a strong imbalance, with increases
  of \(126\%\) for prefix corruption and \(12\%\) for suffix corruption.
  The same reversal without a separator occurs at a sequence length
  of \(256\). All three methods respond monotonically to increasing
  replacement strength within these \(5\%\) regions.
  Thus, repetition without a separator reduces positional imbalance
  but does not eliminate it.

 \begin{wraptable}[11]{r}{0.35\textwidth}
    \centering
    \small
    \vspace{-1.em}
    \setlength{\tabcolsep}{3pt}
    \begin{tabular}{@{}lrrr@{}}
      \toprule
      Temp. & $[\mathbf x]$ & $[\mathbf x,\mathbf x]$ & $[\mathbf x,\texttt{<sep>},\mathbf x]$ \\
      \midrule
      $0.7$ & 8.562 & 9.413 & 8.152 \\
      $0.9$ & \textbf{4.103} & 6.440 & \textbf{4.060} \\
      $1.0$ & 4.758 & 5.421 & 4.703 \\
      $1.1$ & 6.051 & 4.916 & 5.931 \\
      $1.3$ & 8.561 & \textbf{4.884} & 8.205 \\
      \bottomrule
    \end{tabular}
    \caption{\small SOL: GPT-2-large generations scored with OLMo at cap 1,024.
    Square root of the mean squared score across seeds.
    Bold marks the minimum within each method.}
    \label{tab:olmo-repeated-temperature}
  \end{wraptable}

 The temperature sweep also reveals a change in ranking
  (Table~\ref{tab:olmo-repeated-temperature}). We score unconditional,
  natural-length GPT-2-large generations against WebText at a maximum
  sequence length of \(1{,}024\).
  Standard encoding and EOS-separated repetition both favor temperature
  \(0.9\), whereas repetition without a separator favors \(1.3\).
  This preference for a higher temperature runs counter to the expectation
  that moderate sampling temperatures provide a better balance between
  coherence and diversity. Repetition also changes the score scale, so
  lower raw scores across methods do not by themselves indicate improvement. Injection of EOS $[\mathbf x,\texttt{<sep>},\mathbf x]$
 restores normal temperature ranking, but sensitivity is not the same as for $[\mathbf x,\mathbf x]$.
 
%\paragraph{Reweight the sequence.}\textcolor{blue}{TBD.}

\subsection{Text Embedding Models} \label{subsec: Text Embedding Models}
\begin{figure}[h]
    \centering
    \includegraphics[width=0.85\linewidth]{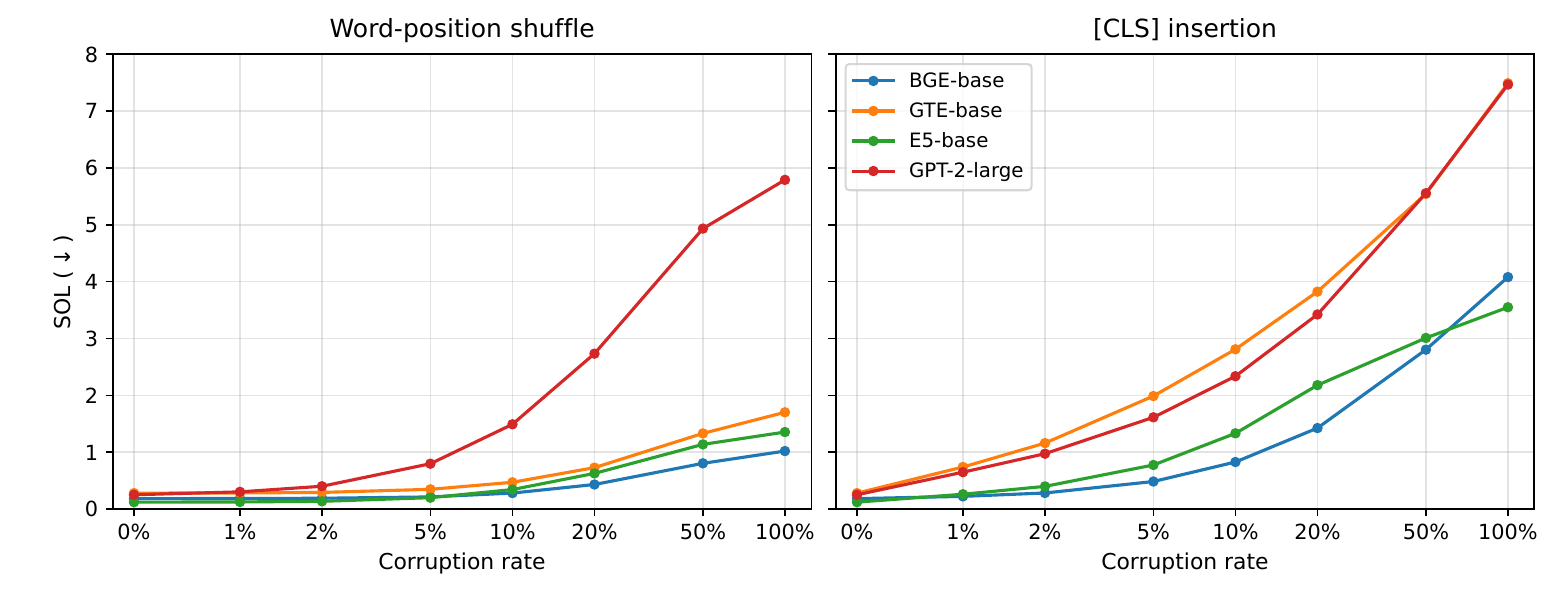}
    \caption{\small{SOL under increasing shuffling of tokens (left) and inserting a CLS token (right). Compared to GPT2-Large, the text encoders are less sensitive to shuffling than GPT-2 Large.} } 
    \label{fig:text-embedding_supp}
\end{figure}
As an alternative to the causal GPT-2 representation, we evaluate modern
text encoders, including BGE-base \cite{bge-m3}, GTE-base
\cite{li2023generaltextembeddingsmultistage}, and E5-base
\cite{wang2024textembeddingsweaklysupervisedcontrastive}.
For this experiment, we split OpenWebText into two disjoint sets of
$5{,}000$ documents, using one as a clean fixed reference and the other
as the candidate distribution. Candidate documents are increasingly
corrupted either by shuffling a fraction of their words or by inserting
the literal string \texttt{[CLS]} into a fraction of word boundaries.
Each corruption level is generated independently of the original clean
candidate set. The resulting texts are retokenized with the native
tokenizer of each encoder and truncated to the first $256$ tokens, and
DSW is computed from all final-layer token representations.

As shown in Figure~\ref{fig:text-embedding_supp}, the text encoders are
substantially less sensitive to small perturbations than GPT-2 Large.
For example, at only $1\%$ word shuffling, DSW$^2$ increases by a factor
of $1.459$ over the clean--clean floor for GPT-2 Large, compared with
only $1.022$, $1.064$, and $1.026$ for GTE, E5, and BGE, respectively.
Thus, the text encoders do detect sufficiently strong corruption, but
resolve small local perturbations much less clearly. This is undesirable
for our purpose, since differences in the underlying text distribution
should already be reflected at low corruption rates. A useful distributional metric should therefore respond at low corruption rates without saturating at high ones. The text encoders fail the first requirement, while mean-pooled MAUVE fails the second (Appendix~\ref{subsec:MAUVE with Mean Hidden Representation}).

\begin{wraptable}{r}{0.25\textwidth}
      \vspace{-1.5em}
      \centering
      \footnotesize
      \setlength{\tabcolsep}{5pt}
      \begin{tabular}{lc}
          \toprule
          Insertion rate & {SOL} $(\downarrow)$ \\
          \midrule
          $20\%$ & 1 \\
          $50\%$ & 0.942 \\
          \bottomrule
      \end{tabular}
      \caption{\small{NeoBERT SOL under increasing \texttt{[CLS]} insertion, normalized to the $20\%$
      value.}}
      \label{tab:neobert-cls}
      \vspace{-0.8em}
  \end{wraptable}

NeoBERT \cite{breton2025neobertnextgenerationbert} exhibits an even stronger failure mode under
\texttt{[CLS]} insertion. As shown in
Table~\ref{tab:neobert-cls}, increasing the corruption rate from
$20\%$ to $50\%$ decreases DSW$^2$ from $2321.64$ to $2060.10$,
although the candidate texts are more heavily corrupted. These results
suggest that representations optimized primarily for text embedding can
be comparatively invariant to precisely the local token-level
perturbations that a distributional text metric should detect.

\newpage
\section{Additional Evaluation Results}
\subsection{Human-Judgment Results}
\label{sec:human-judgment-results}
\begin{table}[h]
  \small
  \centering
  \begin{tabular}{lccc|ccc}
      \toprule
      & \multicolumn{3}{c}{MAUVE}
      & \multicolumn{3}{c}{SOL} \\
      \cmidrule(lr){2-4}
      \cmidrule(lr){5-7}
      Criterion $\uparrow$
      & GPT-2 Large & OLMo2 1B & Dream 7B
      & GPT-2 Large & OLMo2 1B & Dream 7B \\
      \midrule
      Human-like
      & 0.881 & 0.929 & 0.905
      & 0.952 & 0.952 & \textbf{0.976} \\
      Interesting
      & 0.643 & 0.786 & 0.690
      & 0.762 & 0.762 & \textbf{0.857} \\
      Sensible
      & 0.833 & 0.881 & 0.881
      & 0.929 & 0.929 & \textbf{0.976} \\
      \bottomrule
  \end{tabular}
  \caption{\small
  Human judgement comparison at sequence length 256.
  Worst-case Spearman rank correlation between metric-induced model rankings
  and the Bradley--Terry (BT) human ranking.
  Results are shown for MAUVE and SOL using GPT-2 Large,
  OLMo2 1B, and Dream 7B embeddings.
  Reported are the minimum correlations $\rho_{\min}$ across the evaluated
  settings at sequence length 256.
  }
  \label{tab:encoder_human_corr_256}
  \vspace{-0mm}
  \end{table}

Table~\ref{tab:encoder_human_corr_256} reports agreement with the
Bradley--Terry human rankings at a sequence length of $256$ tokens,
matching the maximum amount of text shown in the original human study.
Across all three criteria, SOL attains higher worst-case rank
correlation than MAUVE for each evaluated backbone.
The strongest correlations are obtained with Dream for all three
criteria in this setting.

\subsection{MAUVE Comparison}
\subsubsection{Mean Hidden Representation} \label{subsec:MAUVE with Mean Hidden Representation}
\begin{wraptable}[13]{r}{0.48\textwidth}
    \centering
    \vspace{-2.em}
    \small
    \begin{tabular}{ccc}
        \toprule
        Corruption & \texttt{[CLS]} pollution & Word shuffling \\
        \midrule
        0\%   & 0.9725 & 0.9725 \\
        1\%   & 0.9590 & 0.9710 \\
        2\%   & 0.9183 & 0.9697 \\
        5\%   & \textbf{0.4709} & 0.9597 \\
        10\%  & \textbf{0.0624} & \textbf{0.8236} \\
        20\%  & 0.0181 & \textbf{0.1208 }\\
        50\%  & 0.0127 & 0.0191 \\
        100\% & 0.0118 & 0.0162 \\
        \bottomrule
    \end{tabular}
    \caption{\small{MAUVE under \texttt{[CLS]} pollution and word shuffling using mean pooling.}}
    \label{tab:MAUVE_cls_shuffle_mean}
\end{wraptable}

MAUVE can also be computed on the mean-pooled representation~\eqref{mean_pool} instead of the last-token representation~\eqref{last_hidden}, as in \citet{franca2026hacking}. With mean pooling, however, the score saturates quickly under corruption (Table~\ref{tab:MAUVE_cls_shuffle_mean}). Under \texttt{[CLS]} pollution, MAUVE drops from 0.97 to 0.47 at 5\% and to 0.06 at 10\%, after which it stays near zero (0.018 at 20\%, 0.012 at 100\%). Word shuffling shows the same pattern between 10\% and 20\%. Responding to small corruption is desirable in itself. The problem is that the score has used up its range by 10\%, so it cannot distinguish 10\% from 100\% corrupted text. Once the two samples fall into largely disjoint clusters, the divergence frontier collapses and MAUVE stays at its floor however much worse the text becomes. With the last-token representation, MAUVE degrades more gradually (Table~\ref{tab:MAUVE_corruption}), but it is insensitive to where in the sequence the corruption occurs (Figure~\ref{fig:MAUVE-random-token-begining}). SOL increases monotonically over the full range of corruption rates (Figure~\ref{fig:text-embedding}).
\subsubsection{LM1B Experiment}\label{sec:LM1B MAUVE Experiment}

We evaluated 25,000 samples generated by the vMF model from \citep{CKS2026} using 64 predictor steps and one Langevin corrector step per predictor step, corresponding to 128 function evaluations. We constructed 25,000 reference sequences from the held-out LM1B test split using the model's BERT tokenizer and 128-token packing scheme, while retaining the same special-token formatting in both corpora. We extracted GPT-2-large features in FP32 from the final hidden state at the last non-padding token, using a maximum sequence length of 256 tokens.

We computed $\mathrm{MAUVE}$ on nested subsets of $n\in{500,1000,2000,5000,10000,25000}$ samples per corpus. We varied the number of clusters over $k\in{50,100,200,500,1000,2500}$ and additionally considered the automatic setting $k=n/10$, omitting fixed-$k$ configurations for which $n<2k$. We report means and standard deviations over three clustering seeds, using five $k$-means restarts and at most 500 iterations per restart.

\subsubsection{Comparing MAUVE with SOL}\label{sec:Comparing MAUVE with SOL}
We recap the different drawbacks with the use of the MAUVE score. Throughout the paper, we have seen the affects of different design choices on MAUVE. 

\paragraph{Cluster stability.} MAUVE appears to be heavily dependent on the number of clusters chosen.
As shown in Figure~\ref{fig:MAUVE scaling}, keeping a fixed ratio of clusters to samples decreases the MAUVE score as the sample size increases. 
In contrast when comparing true data in Figure \ref{fig:MAUVE scaling lm1b}, keeping the number of clusters fixed while increasing the number of samples causes the score to converge to 1. This raises the question of which score to report and makes it hard to compare MAUVE scores. SOL, however, converges as we increase the number of samples (GPT-2 Large: Figure~\ref{fig:sample-size}, Dream: Figure~\ref{fig:sample-size-Dream}), the number of slices (GPT-2 Large: Figure~\ref{fig:k-l-stabiltiy}, Dream: Figure~\ref{fig:k-l-stabiltiy-Dream}), or the number of integration points (Figure~\ref{fig:M-stabiltiy}). SOL is therefore numerically much more stable.
\begin{wrapfigure}[13]{r}{0.55\linewidth}
\vspace{-0.2cm}
    \centering
    \includegraphics[width=\linewidth]{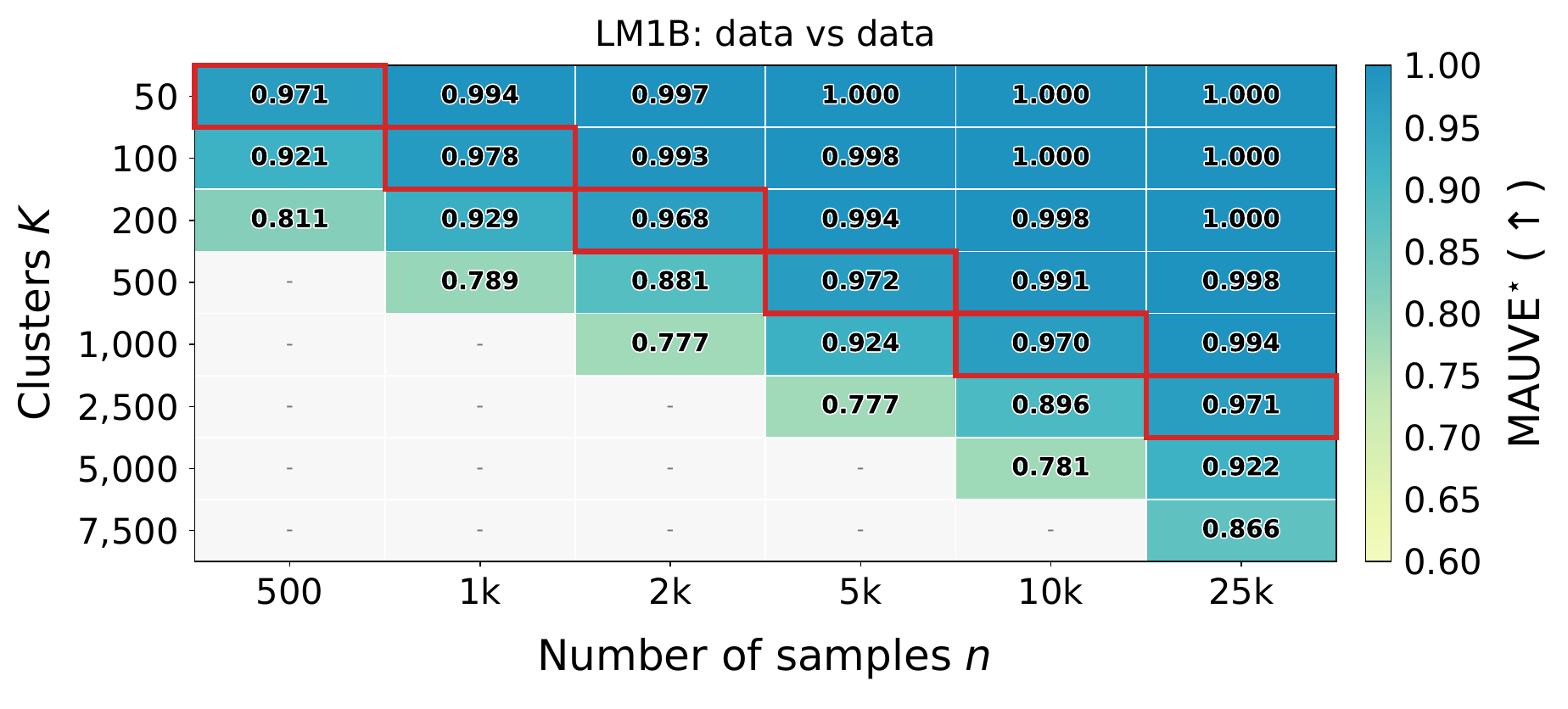}
    \vspace{-8mm}
    \caption{\small{For two disjoint LM1B corpora, the reported MAUVE score depends on the number of evaluation samples $n$ and clusters $K$. In comparison to generated vs true data the ratio $n/10=K$ is stable, but increasing $n$ converges to 1.}}
    \label{fig:MAUVE scaling lm1b}
    \vspace{-2mm}
\end{wrapfigure}
In particular, inexpensive SOL estimates based on only a few samples and slices already provide relatively reliable approximations to the full evaluation, making them suitable for monitoring during training.

\paragraph{Effects of the hidden representation.} MAUVE with the 
(standard) \textit{last-token hidden representation} has a strong bias towards the end of a sequence. As shown in Figure~\ref{fig:MAUVE scaling}, replacing the beginning with random tokens has little effect on the score. This is also visible in Section~\ref{sec:MAUVE depends on last hidden statics} (Diffusion Language Model comparison), where the collapse of the last-token entropy leads to a significantly worse MAUVE score, even when the middle remains comparable. With a \textit{mean-pooled representation}, repeating tokens 10\% of the time produces a MAUVE score of almost 0, indicating a very poor match, see Table~\ref{tab:MAUVE_cls_shuffle_mean}. For the last-token hidden representation, the same test yields a MAUVE score of 0 at repetition rates above $50\%$ (Table~\ref{tab:MAUVE_corruption}). In the same experiment, SOL shows an almost linear degradation (Figure~\ref{fig:text-embedding}).

\paragraph{Human evaluation.}
We use the human evaluation experiment from \cite{pillutla2021MAUVE}, see Section~\ref{sec:Human Judges correlation}. At sequence lengths of 256 and 1024 tokens (Tables~\ref{tab:encoder_human_corr_256} and~\ref{tab:encoder_human_corr_1024}), SOL outperforms MAUVE for almost all embedding models. This indicates that SOL correlates better with human judgments than MAUVE.

 \subsection{OpenWebText complete Model Table} \label{sec: OWT comparision}

We denote discrete-diffusion sampling steps by $T$ and continuous
  predictor steps by $S$. NFE counts actual neural-network evaluations,
  including calls for correction, initialization, final decoding,
  and guidance. Consequently, $S$ and $T$ need not equal NFE.
  The SOL parameters $L_g$, $L_\xi$, and $M$ specify metric estimation
  and are independent of generation cost.

\subsubsection{MAUVE Depends on Last Hidden Tokens}  \label{sec:MAUVE depends on last hidden statics}
\paragraph{Duo.} Comparing MAUVE and SOL on OWT reveals a substantial disagreement for DUO's greedy-tail sampler. To investigate, we
  measure token entropy across documents at a selected position:
  \[
  H_{\mathrm{pos}}(t)
  = -\sum_{v\in\mathcal{V}} p_t(v)\log p_t(v),
  \qquad
  p_t(v)
  = \frac{1}{N}\sum_{i=1}^{N}
  \mathbf{1}\!\left\{x_i^{(t)}=v\right\},
  \]
  where $N=5000$ and
  $x_i^{(t)}$ denotes the token at the selected position in document $i$.
  For the endpoint measurement, we select each document's last valid token.
  This differs from the main
  table's entropy, which averages token-frequency entropy computed within
  each document. 
    \begin{wraptable}{r}{0.35\textwidth}
    \centering
    \small
    \setlength{\tabcolsep}{4pt}
    \begin{tabular}{@{}lrr@{}}
      \toprule
      Model & Middle & Final \\
      \midrule
      Real OWT      & 6.578 & 6.572 \\
      DUO ancestral & 6.455 & 6.441 \\
      DUO greedy    & 6.317 & \textbf{5.011} \\
      \bottomrule
    \end{tabular}
    \caption{Across-document token entropy
      $H_{\mathrm{pos}}$ in nats ($N=5000$),
      measured at position 512 and the last valid token.}
    \label{tab:owt-positional-entropy}
  \end{wraptable}
  Greedy DUO's positional entropy falls from $6.32$ nats
  at the middle position to $5.01$ at the endpoint, whereas real OWT
  remains near $6.58$ at both positions, see Table \ref{tab:owt-positional-entropy}.

  \begin{wraptable}{r}{0.52\textwidth}
    \centering
    \small
    \setlength{\tabcolsep}{3pt}
    \begin{tabular}{@{}lrrrr@{}}
      \toprule
      & \multicolumn{2}{c}{MAUVE $(\uparrow)$}
      & \multicolumn{2}{c}{SOL $(\downarrow)$} \\
      \cmidrule(lr){2-3}\cmidrule(lr){4-5}
      Model & Base & Reweighted & Base & Reweighted \\
      \midrule
      DUO ancestral & 0.836 & 0.843 & 1.670 & 1.7418 \\
      DUO greedy    & 0.492 & 0.591 & 1.748 & 1.8538 \\
      \bottomrule
    \end{tabular}
    \caption{Final-token reweighting on matched
      corrected inputs ($N=5000$). GPT-2-large SOL
      values average nine projection-seed combinations.}
    \label{tab:owt-endpoint-reweight}
  \end{wraptable}

To test whether the final-token distribution drives MAUVE, we reweight the generated documents by their last token. Overrepresented final tokens get less weight, underrepresented ones get more. This makes the weighted final-token frequencies match real OWT without changing the texts.
After this correction, MAUVE rises from $0.492$ to $0.591$ for greedy DUO, but only from $0.836$ to $0.843$ for ancestral DUO. SOL changes by only
$0.29\%$ and $0.07\%$, respectively. This shows that MAUVE is much more sensitive to the final-token distribution than SOL.

For the implementation of our reweighting, we group documents by their final token. We treat each final-token ID that appears at least five times in both corpora as a separate category. All remaining tokens are pooled into an ``other'' category.
Each generated document $i$ gets weight
$w_i~=~p_{c(i)}/q_{c(i)}$,
where $p_c$ and $q_c$ are the final-token frequencies in the reference and generated corpora. We keep all features and cluster assignments fixed. We then recompute MAUVE with weighted generated-cluster counts and the original reference histogram. This changes only the document weights, not the texts.

After matching the final-token frequencies to real OWT, MAUVE improves from $0.492$ to $0.591$ for greedy DUO. For ancestral DUO, it changes only from $0.836$ to $0.843$, see Table \ref{tab:owt-endpoint-reweight}.
This shows that MAUVE's positional bias toward the last token directly affects its evaluation of generative models. SOL is much less affected because it uses hidden states from all token positions.
\paragraph{ELF.}
\begin{wraptable}{r}{0.49\textwidth}
    \centering
    \small
    \setlength{\tabcolsep}{4pt}
    \begin{tabular}{@{}lcc@{}}
      \toprule
      Model & Original & Remove final 8 \\
      \midrule
      ELF-M & $0.455 \pm 0.011$ & $0.823 \pm 0.017$ \\
      ELF-L & $0.287 \pm 0.006$ & $0.798 \pm 0.013$ \\
      \bottomrule
    \end{tabular}
    \caption{MAUVE $(\uparrow)$ for shared ELF sampling.
    Truncation removes eight tokens from both generated and
    reference texts. Mean $\pm$ SD over three clustering seeds.}
    \label{tab:elf-MAUVE-endings}
  \end{wraptable}
  We evaluate ELF-M and ELF-L with shared sampling settings
  (SDE, 64 steps, $\gamma=c=1$), using 5,000 texts per distribution
  and GPT-2-large MAUVE features at a 1,024-token cap.
  We remove the final eight tokens from each generated and reference
  sequence, retaining one token for the three references shorter than
  nine tokens. MAUVE uses the hidden state at the resulting endpoint.
  This reduces the M--L score gap from $0.168$ to $0.024$, while the
  real--real control remains similar ($0.962$ versus $0.965$).

  The tokens at the ends of the texts change much more. For example, the fraction
  of L texts ending with the period token falls from $54.1\%$ to
  $2.7\%$. MAUVE extracts each text's representation at this final
  position. Its large score change therefore demonstrates sensitivity
  to where the text ends, despite almost unchanged token frequencies
  overall, see Table \ref{tab:elf-MAUVE-endings}.

\subsubsection{Precision, Recall, and Coverage}

\paragraph{Setup.}
We complement SOL with the embedding-based precision and recall evaluation of
\citet{bronnec2024exploringprecisionrecallassess} and coverage from
\citet{pmlr-v119-naeem20a}.
We compare MDLM and ReMDM-loop using the same Soft-Masked checkpoint,
1,024-token generations, $T=1024$, final noise removal, and $p=0.9$.
For each of three independent repetitions, we generate 5,000 samples per
sampler and draw disjoint sets of 10,000 calibration and 5,000 reference
examples from held-out OpenWebText. Within each repetition, both samplers
are evaluated against the same reference set and under the same embedding
transformation, yielding a paired comparison.

We use GPT-2-large final-layer last-token embeddings, L2 normalization,
unwhitened PCA, and Euclidean distances. Whereas
\citet{bronnec2024exploringprecisionrecallassess} fit PCA jointly to reference and
generated samples and retain 90\% of the variance, we fit PCA exclusively
on the independent real calibration set. We fix the representation to
350 dimensions, corresponding to 89.32--89.52\% explained variance across
the three calibration sets. This avoids adapting the representation to
either sampler.

\paragraph{Metrics.}
For $k=4$, the radius associated with each point is its distance to its
fourth nearest neighbor within the same corpus. Precision is the fraction
of generated samples contained in the union of reference neighborhoods,
whereas recall is the fraction of reference samples contained in the union
of generated neighborhoods. Coverage is the fraction of reference points
whose reference-defined neighborhood contains at least one generated
sample~\citep{pmlr-v119-naeem20a}. Thus recall depends on generated-sample
neighborhood radii, while coverage uses only reference-sample radii.

\begin{wraptable}[10]{r}{0.35\linewidth}
\centering
\small
\setlength{\tabcolsep}{3pt}
\begin{tabular}{lrrr}
\toprule
Rep. & $\Delta P$ & $\Delta R$ & $\Delta C$ \\
\midrule
1 & +2.14 & +6.70 & +3.10 \\
2 & +3.22 & +4.88 & +1.80 \\
3 & +2.10 & +8.32 & +0.84 \\
\midrule
Mean & +2.49 & +6.63 & +1.91 \\
SD   &  0.64 &  1.72 &  1.13 \\
\bottomrule
\end{tabular}
\caption{ReMDM-loop minus MDLM at $k=4$, in percentage points.
SD is across three independent repetitions.}
\label{tab:softmask-prdc}
\end{wraptable}

\paragraph{Results.}
ReMDM-loop improves precision, recall, and coverage in every repetition
(Table~\ref{tab:softmask-prdc}). Averaged across repetitions, precision
increases from 77.48\% to 79.97\%, recall from 62.63\% to 69.26\%, and
coverage from 74.14\% to 76.05\%. Hence the SOL improvement is accompanied
by both higher fidelity and greater coverage under this latent-space
diagnostic. The coverage gain decreases to $+0.25$ percentage points at
$k=10$, indicating some sensitivity to the neighborhood scale.
ReMDM-loop uses approximately 47\% more forward evaluations, so these
results compare the samplers at the same scheduled-step budget rather
than at matched compute.

\subsubsection{Extended Table}\label{app:Extended Table}
The complete OWT table is given in Table \ref{tab:experiment-8d-owt-baselines}.
\paragraph{Increasing the number of steps.}
Increasing the sampling budget generally improves SOL, but the changes in Gen.\ PPL and entropy can be much larger. For LangFlow, increasing NFE from 128 to 1024 improves Dream SOL from $32.488$ to $30.541$ and Gen.\ PPL from $59.0$ to $36.9$, while entropy moves away from the data value, from $5.425$ to $5.257$. Similarly, ELF-B with ODE sampling improves Dream SOL only from $33.677$ to $33.320$ between 128 and 1024 steps, despite Gen.\ PPL falling from $45.4$ to $24.0$ and entropy moving from $5.308$ to $5.090$.

\paragraph{Changing the sampler and guidance.}
Sampler choice substantially affects distributional quality. For Soft-Masked, replacing MDLM sampling with ReMDM-loop reduces Dream SOL from $21.643$ to $18.966$ and Gen.\ PPL from $23.6$ to $16.7$, despite a small increase in the entropy gap to the data. For vMF at 1024 steps, switching from ODE to predictor--corrector sampling improves Dream SOL from $33.832$ to $30.583$ and Gen.\ PPL from $75.4$ to $56.3$, while entropy remains close to the data value ($5.456$ versus $5.432$). Guidance similarly improves ELF-B at 32 SDE steps, reducing Dream SOL from $39.144$ to $37.584$ and Gen.\ PPL from $29.4$ to $24.0$.

\paragraph{Self-conditioning.}
The vMF comparison provides direct evidence for the benefit of self-conditioning. At the same 128-step predictor--corrector setting, the base self-conditioned model improves Dream SOL from $45.900$ to $31.379$ and Gen.\ PPL from $141.9$ to $45.1$. This improvement is not explained by entropy matching, which changes from $5.533$ to $5.315$ and therefore moves slightly farther from the data entropy. The distilled ELF* model still exhibits a gap to ELF-B, although this comparison does not isolate self-conditioning from other model and training differences.

\paragraph{Comparison with autoregressive generation.}
Soft-Masked with ReMDM-loop improves over the AR baseline on all three SOL encoders, from $(1.064,7.450,22.207)$ to $(0.6101,3.526,18.966)$. It also has lower Gen.\ PPL ($16.7$ versus $34.9$) and entropy closer to the data ($5.341$ versus $5.582$). LDDM-U, however, illustrates that these criteria need not agree: it has both lower Gen.\ PPL ($28.2$) and entropy closer to the data ($5.409$), but improves over AR only under GPT-2-large and OLMo, while Dream SOL is higher ($24.209$ versus $22.207$). Among the evaluated continuous models, self-conditioned vMF with predictor--corrector sampling achieves the lowest SOL, but remains above AR on all three encoders.

\paragraph{Few-step language generation.}
Within FMLM*, additional CTM steps substantially improve generation while keeping entropy close to that of the data. At $CFG=5$, increasing from one to five steps lowers Gen.\ PPL from $116.0$ to $70.8$ and Dream SOL from $46.273$ to $40.177$, with entropy remaining near the reference value ($5.368$ and $5.415$, versus $5.449$). Further computation gives diminishing returns: increasing flow-map sampling from 8 to 128 steps changes Gen.\ PPL only from $82.8$ to $79.0$ and leaves Dream SOL essentially unchanged ($38.233$ versus $38.482$).

\paragraph{Perplexity, entropy, and distributional quality can disagree.}
Gen.\ PPL and entropy should be interpreted jointly: Gen.\ PPL should be small, while entropy should remain close to that of the data. Even together, however, they do not determine distributional quality. For distilled DUO, ancestral and greedy-tail sampling have almost identical entropy deviations from the data ($|5.518-5.449|=0.069$ and $|5.378-5.449|=0.071$). Greedy-tail sampling nevertheless lowers Gen.\ PPL substantially, from $62.3$ to $46.5$, while worsening SOL on all three encoders, including Dream from $30.595$ to $31.802$. More strongly, exact MDLM has both lower Gen.\ PPL than VP with ODE sampling ($59.4$ versus $66.1$) and entropy closer to the data ($5.445$ versus $5.435$), yet has worse SOL on all three encoders: $(2.185,14.119,36.394)$ versus $(1.775,11.997,33.275)$. Thus, even favorable Gen.\ PPL and entropy do not guarantee a better match to the text distribution.

\subsection{Additional Sampler}
 \label{subsec:gpt2xl-samplers}

   We investigate how sampling strategies affect GPT-2-XL's agreement with
  the WebText distribution. We compare ancestral sampling, which draws
  directly from the model's next-token distribution, and nucleus sampling
  with \(p=0.95\)~\citep{Holtzman2020The}, which restricts sampling to the
  smallest set of highest-probability tokens with cumulative probability
  at least \(p\), renormalizing their probabilities. We also evaluate
  greedy decoding, which deterministically selects the most likely token,
  in the conditional setting.

  \paragraph{Conditional generation.}
  We use \(5{,}000\) WebText documents, providing the first \(35\) GPT-2
  tokens of each document as context. Generation stops at the first
  \(\langle\mathrm{EOT}\rangle\) token or when the total sequence length,
  including the prefix, reaches \(256\) tokens. All three samplers receive
  identical prefixes and are evaluated against WebText at the same horizon.

  \begin{table}[h]
    \centering
    \small
    \begin{tabular}{lcc}
      \toprule
      Sampler & MAUVE (\(\uparrow\)) & SOL (\(\downarrow\)) \\
      \midrule
      Greedy
      & \(0.0697 \pm 0.0012\)
      & \(3.7755 \pm 0.1190\) \\
      Ancestral
      & \(0.9407 \pm 0.0028\)
      & \(0.5083 \pm 0.0113\) \\
      Nucleus (\(p=0.95\))
      & \(\mathbf{0.9607 \pm 0.0017}\)
      & \(\mathbf{0.2712 \pm 0.0065}\) \\
      \bottomrule
    \end{tabular}
    \caption{GPT-2-XL conditional completion on \(5{,}000\) WebText
    documents, using shared \(35\)-token prefixes and a maximum total
    sequence length of \(256\) tokens. SOL uses GPT-2-large.}
    \label{tab:gpt2xl-conditional-decoding}
  \end{table}

  Both metrics favor nucleus sampling
  (Table~\ref{tab:gpt2xl-conditional-decoding}):
  SOL decreases from \(0.5083\) to \(0.2712\), while MAUVE increases
  from \(0.9407\) to \(0.9607\). Thus, nucleus truncation improves
  agreement with WebText in this conditional setting.
  Greedy decoding performs substantially worse than both stochastic
  samplers, with MAUVE of \(0.0697\) and SOL of \(3.7755\).

  \paragraph{Unconditional generation.}
  We initialize GPT-2-XL with \(\langle\mathrm{EOT}\rangle\) and generate
  until the first subsequently generated \(\langle\mathrm{EOT}\rangle\)
  token or a maximum of \(256\) generated tokens. With this fixed initial
  context, greedy decoding always produces the same continuation,
  yielding a fully collapsed generated distribution.

  \begin{table}[h]
    \centering
    \small
    \begin{tabular}{lcc}
      \toprule
      Sampler & MAUVE (\(\uparrow\)) & SOL (\(\downarrow\)) \\
      \midrule
      Ancestral
      & \(\mathbf{0.8832 \pm 0.0071}\)
      & \(\mathbf{0.8358 \pm 0.0245}\) \\
      Nucleus (\(p=0.95\))
      & \(0.8740 \pm 0.0073\)
      & \(0.9072 \pm 0.0279\) \\
      \bottomrule
    \end{tabular}
    \caption{GPT-2-XL unconditional generation evaluated against
    \(5{,}000\) WebText documents. Each of the \(5{,}000\) samples
    starts from the end-of-text token and stops at EOS or \(256\)
    generated tokens, the initial seed token is excluded from evaluation.
    SOL uses GPT-2-large.}
    \label{tab:gpt2xl-unconditional-decoding}
  \end{table}

  The ordering reverses in the unconditional setting
  (Table~\ref{tab:gpt2xl-unconditional-decoding}).
  Ancestral sampling achieves lower SOL (\(0.8358\) versus \(0.9072\))
  and higher MAUVE (\(0.8832\) versus \(0.8740\)) than nucleus sampling.

  Overall, SOL and MAUVE agree on the sampling-strategy ordering in both
  settings. Nucleus sampling improves conditional completion, whereas
  ancestral sampling achieves better unconditional agreement with WebText
  in this experiment.

%%%%%%%%%%%%%%%%%%%%%%%%%%%%%%%%%%%%%%%%%%%%%%%%%%%%%%%%%%%%%%%%%%%%%%%%%%%%%%%%%%%%%

\section{Implementation Details}
\label{sec:implementation}

\subsection{Numerical implementation of SOL}
\label{sec:dsw-numerical}

Let $\mathcal X=\{\mathbf x_i\}_{i=1}^{n_x}$ and
$\mathcal Y=\{\mathbf y_j\}_{j=1}^{n_y}$ denote two corpora with empirical
text distributions
\begin{align}
\widehat p_{\mathcal X}
=
\frac{1}{n_x}\sum_{i=1}^{n_x}\delta_{\mathbf x_i},
\qquad
\widehat p_{\mathcal Y}
=
\frac{1}{n_y}\sum_{j=1}^{n_y}\delta_{\mathbf y_j}.
\end{align}
Following Section~\ref{subsec: our method}, every sequence is mapped by a
frozen reference transformer $\mathbf H^\theta$ to its token cloud
$T^\theta(\mathbf x)\in\mathcal P_2(\mathbb R^d)$.

We write $\widetilde h^\theta_{\mathbf x,r}=h^\theta_{\mathbf x_{\leq r}}$
for a causal reference transformer and
$\widetilde h^\theta_{\mathbf x,r}=h^{\theta,\mathrm{bi}}_{\mathbf x,r}$
for a bidirectional reference transformer. We retain all non-padding token
positions up to the encoder horizon $T$ and denote their number by $t_x$.

The exact distance
\begin{align}
\operatorname{SOL}_\theta^2
(\widehat p_{\mathcal X},\widehat p_{\mathcal Y})
=
d \operatorname{DSW}^2
\left(
T^\theta_\#\widehat p_{\mathcal X},
T^\theta_\#\widehat p_{\mathcal Y}
\mid\gamma
\right)
\end{align}
is approximated using $L_\xi$ first-level directions, $L_g$ second-level
directions, and resolution $M$ for the second-level slicing functions.
The same sampled directions are used for both corpora.

\paragraph{First slicing.}
We sample $a_\ell\sim\mathcal N(0,I_d)$ and set
$\xi_\ell=a_\ell/\|a_\ell\|$ for $\ell=1,\ldots,L_\xi$, such that
$\xi_\ell\sim\mathcal U(\mathbb S^{d-1})$. For each sequence $\mathbf x$,
\begin{align}
\mathbf e_{\xi_\ell}T^\theta(\mathbf x)
&=
\frac{1}{t_x}
\sum_{r=1}^{t_x}\delta_{v_{x,\ell,r}},
\qquad
v_{x,\ell,r}
=
\left\langle
\widetilde h^\theta_{\mathbf x,r},
\xi_\ell
\right\rangle .
\label{eq:dsw-first-slice}
\end{align}
Let $v_{x,\ell,(1)}\leq\cdots\leq v_{x,\ell,(t_x)}$ denote the corresponding
order statistics.

\paragraph{Bin-averaged empirical quantiles.}
We partition $[0,1]$ into
$I_m=((m-1)/M,m/M]$ with midpoint $s_m=(m-\frac12)/M$ for
$m=1,\ldots,M$. For every projected token cloud, we compute the exact
average of its empirical inverse-CDF quantile within each bin,
\begin{align}
\bar q_{x,\ell,m}
&=
M
\int_{I_m}
Q_{\mathbf e_{\xi_\ell}T^\theta(\mathbf x)}(s)
\,\mathrm ds
\label{eq:dsw-bin-quantiles} =
M
\sum_{r=1}^{t_x}
v_{x,\ell,(r)}
\left|
I_m
\cap
\left(
\frac{r-1}{t_x},
\frac{r}{t_x}
\right]
\right|.
\end{align}
The bin averages are computed efficiently from the sorted projected values
using cumulative sums. This construction preserves the mean exactly,
since
\begin{align}
\frac{1}{M}\sum_{m=1}^M\bar q_{x,\ell,m}
=
\frac{1}{t_x}\sum_{r=1}^{t_x}v_{x,\ell,r}.
\label{eq:dsw-bin-mean}
\end{align}

\paragraph{Second slicing.}
We draw the second-level slicing functions from an RBF Gaussian process with
covariance $k_\lambda(s,t)=\exp(-(s-t)^2/(2\lambda^2))$. On the midpoint
grid, we define $C_{ab}=k_\lambda(s_a,s_b)$ for $a,b=1,\ldots,M$ and compute
$RR^\top=C+\varepsilon I$ by Cholesky factorization, with a small diagonal
jitter $\varepsilon$ for numerical stability.

For $\tilde\ell=1,\ldots,L_g$, we sample
$\omega_{\tilde\ell}\sim\mathcal N(0,I_M)$ and set
$g_{\tilde\ell}=R\omega_{\tilde\ell}$. Thus $g_{\tilde\ell}$ follows the
finite-dimensional Gaussian-process distribution on the grid. We interpret
it as the piecewise-constant function
$g_{\tilde\ell}^{(M)}(s)=g_{\tilde\ell,m}$ for $s\in I_m$.

The $L_g$ directions are sampled independently of
$\xi_1,\ldots,\xi_{L_\xi}$ and reused for all first-level directions. Hence,
we evaluate all $L_\xi L_g$ pairs $(\xi_\ell,g_{\tilde\ell})$.

For every sequence and pair of directions, we compute
\begin{align}
z_{x,\ell,\tilde\ell}
&=
\widehat M_{\xi_\ell,g_{\tilde\ell}}^{(M)}
\bigl(T^\theta(\mathbf x)\bigr)
=
\frac{1}{M}
\sum_{m=1}^M
\bar q_{x,\ell,m}g_{\tilde\ell,m}
\label{eq:dsw-document-scalars}=
\int_0^1
Q_{\mathbf e_{\xi_\ell}T^\theta(\mathbf x)}(s)
g_{\tilde\ell}^{(M)}(s)
\,\mathrm ds.
\end{align}
Thus the empirical quantile is integrated exactly against the
piecewise-constant second-level slicing function.

\paragraph{One-dimensional Wasserstein distance.}
For fixed $(\ell,\tilde\ell)$, define
\begin{align}
\widehat\mu_{\ell,\tilde\ell}^{\mathcal X}
&=
\frac{1}{n_x}
\sum_{i=1}^{n_x}
\delta_{z_{x_i,\ell,\tilde\ell}},
\qquad
\widehat\mu_{\ell,\tilde\ell}^{\mathcal Y}
=
\frac{1}{n_y}
\sum_{j=1}^{n_y}
\delta_{z_{y_j,\ell,\tilde\ell}}.
\end{align}
Let $z^{\mathcal X}_{(i),\ell,\tilde\ell}$ and
$z^{\mathcal Y}_{(j),\ell,\tilde\ell}$ denote their order statistics.
If $n_x=n_y=n$, then
\begin{align}
\operatorname W^2
\left(
\widehat\mu_{\ell,\tilde\ell}^{\mathcal X},
\widehat\mu_{\ell,\tilde\ell}^{\mathcal Y}
\right)
=
\frac{1}{n}
\sum_{i=1}^n
\left(
z^{\mathcal X}_{(i),\ell,\tilde\ell}
-
z^{\mathcal Y}_{(i),\ell,\tilde\ell}
\right)^2.
\label{eq:dsw-equal-size-wasserstein}
\end{align}
All main experiments use equal numbers of reference and generated samples.

For completeness, the implementation also supports unequal corpus sizes.
Let $N=\max(n_x,n_y)$ and $r_a=(a-\frac12)/N$ for $a=1,\ldots,N$. We use
\begin{align}
\widehat{\operatorname W}^{\,2}_{N}
=
\frac{1}{N}
\sum_{a=1}^{N}
\left(
z^{\mathcal X}_{(\lceil n_xr_a\rceil),\ell,\tilde\ell}
-
z^{\mathcal Y}_{(\lceil n_yr_a\rceil),\ell,\tilde\ell}
\right)^2,
\label{eq:dsw-unequal-size-wasserstein}
\end{align}
which reduces to \eqref{eq:dsw-equal-size-wasserstein} when $n_x=n_y$.

\paragraph{Monte Carlo estimator.}
Combining the two slicing levels gives
\begin{align}
\widehat{\operatorname{SOL}}_{\theta,L_\xi,L_g,M}^{\,2}
(\mathcal X,\mathcal Y)
=
\frac{\alpha}{L_\xi L_g}
\sum_{\ell=1}^{L_\xi}
\sum_{\tilde\ell=1}^{L_g}
\widehat{\operatorname W}^{\,2}_{N}
\left(
\widehat\mu_{\ell,\tilde\ell}^{\mathcal X},
\widehat\mu_{\ell,\tilde\ell}^{\mathcal Y}
\right),
\label{eq:dsw-numerical-score}
\end{align}
where $\alpha>0$ is a fixed scaling factor. We use $\alpha=d$, which
compensates for
\begin{align}
\mathbb E_{\xi\sim\mathcal U(\mathbb S^{d-1})}
\left[
\langle v,\xi\rangle^2
\right]
=
\frac{\|v\|^2}{d}.
\end{align}
Since $\alpha$ is fixed for a given reference model, it does not affect the
ranking of generated corpora.

\begin{algorithm}[t]
\caption{Numerical approximation of SOL}
\label{alg:dsw-numerical}
\begin{algorithmic}[1]
\Require Corpora $\mathcal X,\mathcal Y$, frozen reference transformer
$\mathbf H^\theta$, horizon $T$
\Require Number of slices $L_\xi,L_g$, resolution $M$, GP length scale
$\lambda$, scale $\alpha$
\Ensure $\widehat{\operatorname{SOL}}_{\theta,L_\xi,L_g,M}^{\,2}
(\mathcal X,\mathcal Y)$
\State Construct bins $I_1,\ldots,I_M$ and midpoints $s_1,\ldots,s_M$
\State Sample $\xi_1,\ldots,\xi_{L_\xi}\sim\mathcal U(\mathbb S^{d-1})$
\State Sample RBF-GP directions $g_1,\ldots,g_{L_g}$ on
$\{s_m\}_{m=1}^M$
\ForAll{$\mathbf x\in\mathcal X\cup\mathcal Y$}
    \State Compute the retained hidden states of $\mathbf H^\theta(\mathbf x)$
    up to horizon $T$
    \For{$\ell=1,\ldots,L_\xi$}
        \State Project and sort the token states using
        \eqref{eq:dsw-first-slice}
        \State Compute $\bar q_{x,\ell,1:M}$ using
        \eqref{eq:dsw-bin-quantiles}
        \For{$\tilde\ell=1,\ldots,L_g$}
            \State Compute $z_{x,\ell,\tilde\ell}$ using
            \eqref{eq:dsw-document-scalars}
        \EndFor
    \EndFor
\EndFor
\For{$\ell=1,\ldots,L_\xi$ and $\tilde\ell=1,\ldots,L_g$}
    \State Sort the document scalars in $\mathcal X$ and $\mathcal Y$
    \State Compute the one-dimensional Wasserstein term
\EndFor
\State \Return \eqref{eq:dsw-numerical-score}
\end{algorithmic}
\end{algorithm}

\paragraph{Precision and memory.}
The reference transformer is evaluated in batches.
Hidden states are cast to FP32 before projection.
Token projections, bin-average computation, second-level projections,
sorting, and Wasserstein reductions are performed in FP32.
The Gaussian-process covariance and Cholesky factorization are computed
in FP64, after which the sampled GP values are cast to FP32.

For each first-level projection seed, the bin-average tensors of shape
$n\times L_\xi\times M$ are stored as FP32 memory-mapped arrays.
The second-level computation is performed in blocks of first-level
directions and reuses the reference tensors across generated corpora.
Thus neither all $nL_\xi L_g$ document scalars nor a pairwise document
cost matrix need to be stored simultaneously.
Changing only the GP seed requires no additional transformer evaluation.

\begin{table*}[h]
    \centering
    \small
    \setlength{\tabcolsep}{4pt}
    \renewcommand{\arraystretch}{0.94}
    \resizebox{\textwidth}{!}{%
    \begin{tabular}{@{}llrrrrr@{}}
      \toprule
      & & & & \multicolumn{3}{c}{SOL $(\downarrow)$} \\
      \cmidrule(lr){5-7}
      Model & Setting & $H@1024$ & Gen. ppl. & GPT-2-large & OLMo-2-1B & Dream-7B \\
      \midrule

      \emph{Real OWT}\textsuperscript{G}
      & --- & 5.449 & 14.9
      & $0.2002 \pm 0.0116$
      & $0.7800 \pm 0.0552$
      & $2.066 \pm 0.109$ \\

      \rowcolor{gray!10}
      \emph{Real ELF OWT}\textsuperscript{E}
      & --- & 5.402 & 17.0
      & $0.2457 \pm 0.0239$
      & $0.8857 \pm 0.0636$
      & $2.649 \pm 0.219$ \\

      \midrule
      \addlinespace[2pt]
      \multicolumn{7}{l}{\textit{Autoregressive model} \cite{sahoo2024simple}} \\

      AR\textsuperscript{G}
      & Temperature 1.0 & 5.582 & 34.9
      & $1.064 \pm 0.008$
      & $7.450 \pm 0.132$
      & $22.207 \pm 0.099$ \\

      \midrule
      \multicolumn{7}{l}{\textit{Discrete diffusion}} \\
      \addlinespace[2pt]

      \multicolumn{7}{l}{\textit{DUO} \cite{sahoo2025the}} \\

      DUO-distilled, $T=32$\textsuperscript{G}
      & Ancestral & 5.518 & 62.3
      & $1.670 \pm 0.019$
      & $11.063 \pm 0.228$
      & $30.595 \pm 0.082$ \\

      DUO-distilled, $T=32$\textsuperscript{G}
      & Greedy tail & 5.378 & 46.5
      & $1.748 \pm 0.024$
      & $11.729 \pm 0.253$
      & $31.802 \pm 0.242$ \\

      DUO-base\textsuperscript{G}
      & Ancestral, $T=1024$ & 5.537 & 76.9
      & $1.948 \pm 0.024$
      & $12.730 \pm 0.250$
      & $32.642 \pm 0.084$ \\

      \addlinespace[2pt]
      \multicolumn{7}{l}{\textit{MDLM} \cite{sahoo2024simple}} \\

      MDLM, exact sampling\textsuperscript{G}
      & DDPM-cache, $T=1000$ & 5.445 & 59.4
      & $2.185 \pm 0.028$
      & $14.119 \pm 0.246$
      & $36.394 \pm 0.086$ \\

      MDLM, as released\textsuperscript{G}
      & FP32 Gumbel, $T=1000$ & 5.292 & 42.8
      & $2.340 \pm 0.035$
      & $14.275 \pm 0.238$
      & $38.444 \pm 0.112$ \\

      \addlinespace[2pt]
      \multicolumn{7}{l}{\textit{LDDM} \cite{jo2026loopholing}} \\

      LDDM-M\textsuperscript{G}
      & Loopholing, $T=1024$ & 5.536 & 48.2
      & $1.082 \pm 0.011$
      & $8.266 \pm 0.160$
      & $24.601 \pm 0.038$ \\

      LDDM-U\textsuperscript{G}
      & Loopholing, $T=1024$ & 5.409 & 28.2
      & $1.009 \pm 0.034$
      & $6.851 \pm 0.168$
      & $24.209 \pm 0.095$ \\

      \addlinespace[2pt]
      \multicolumn{7}{l}{\textit{Soft-Masked} \cite{hersche2026softmasked}} \\

      Soft-Masked\textsuperscript{G}
      & MDLM, $T=1024$, $p=0.9$ & 5.364 & 23.6
      & $0.7415 \pm 0.0090$
      & $5.390 \pm 0.110$
      & $21.643 \pm 0.103$ \\

      Soft-Masked\textsuperscript{G}
      & ReMDM-loop, $T=1024$, $p=0.9$ & 5.341 & 16.7
      & $\mathbf{0.6101 \pm 0.0135}$
      & $\mathbf{3.526 \pm 0.066}$
      & $\mathbf{18.966 \pm 0.126}$ \\

      \addlinespace[2pt]
  \multicolumn{7}{l}{\textit{ADLM} \cite{rout2025anchored}} \\

  ADLM-large\textsuperscript{G}
  & ReMDM-loop, $T=1024$, $p=0.9$ & 5.352 & 28.9
  & $1.183 \pm 0.014$
  & $7.534 \pm 0.135$
  & $23.601 \pm 0.088$ \\

  \addlinespace[2pt]
  \multicolumn{7}{l}{\textit{BD3-LM} \cite{arriola2025block}} \\

  BD3-LM, block 4\textsuperscript{B}
  & First-hitting, $T=5000$, $p=0.9$ & 5.336 & 24.4
  & $1.082 \pm 0.023$
  & $8.407 \pm 0.180$
  & $26.070 \pm 0.057$ \\

  BD3-LM, block 16\textsuperscript{B}
  & First-hitting, $T=5000$, $p=0.9$ & 5.344 & 32.9
  & $1.458 \pm 0.027$
  & $10.709 \pm 0.218$
  & $30.315 \pm 0.091$ \\

      \midrule
      \multicolumn{7}{l}{\textit{Continuous diffusion}} \\
      \addlinespace[2pt]

      \multicolumn{7}{l}{\textit{vMF} \cite{CKS2026}} \\

      vMF, self-cond.\textsuperscript{G}
      & PC-warp, $S=128$ & 5.222 & 30.2
      & $1.354 \pm 0.034$
      & $8.607 \pm 0.197$
      & $29.225 \pm 0.089$ \\

      vMF, self-cond. (base)\textsuperscript{G}
      & PC-warp, $S=128$ & 5.315 & 45.1
      & $1.524 \pm 0.022$
      & $10.657 \pm 0.221$
      & $31.379 \pm 0.076$ \\

      vMF, self-cond.\textsuperscript{G}
      & SDE-warp, $S=128$, $\sigma=0.01$, $q=1$ & 5.348 & 43.6
      & $1.386 \pm 0.022$
      & $9.884 \pm 0.196$
      & $30.375 \pm 0.074$ \\

      vMF, no self-cond.\textsuperscript{G}
      & PC-warp, $S=128$ & 5.533 & 141.9
      & $2.928 \pm 0.036$
      & $18.157 \pm 0.343$
      & $45.900 \pm 0.130$ \\

      vMF\textsuperscript{G}
      & ODE-warp, $S=1024$ & 5.456 & 75.4
      & $1.901 \pm 0.024$
      & $12.781 \pm 0.268$
      & $33.832 \pm 0.113$ \\

      vMF\textsuperscript{G}
      & PC-warp, $S=1024$ & 5.432 & 56.3
      & $1.546 \pm 0.021$
      & $10.777 \pm 0.233$
      & $30.583 \pm 0.099$ \\

      \addlinespace[2pt]
      \multicolumn{7}{l}{\textit{VP} \cite{CKS2026}} \\

      VP, self-cond.\textsuperscript{G}
      & SDE-warp, $S=128$, $\sigma=0.1$, $q=1$ & 5.355 & 50.9
      & $1.606 \pm 0.023$
      & $10.898 \pm 0.222$
      & $31.972 \pm 0.072$ \\

      VP\textsuperscript{G}
      & ODE-warp, $S=1024$ & 5.435 & 66.1
      & $1.775 \pm 0.022$
      & $11.997 \pm 0.247$
      & $33.275 \pm 0.076$ \\

      VP\textsuperscript{G}
      & PC-warp, $S=1024$ & 5.430 & 59.0
      & $1.642 \pm 0.022$
      & $11.168 \pm 0.232$
      & $31.664 \pm 0.068$ \\

      \addlinespace[2pt]
      \multicolumn{7}{l}{\textit{LangFlow} \cite{chen2026langflowcontinuousdiffusionrivals}} \\

      LangFlow\textsuperscript{G}
      & NFE$=128$ & 5.425 & 59.0
      & $1.604 \pm 0.018$
      & $11.930 \pm 0.226$
      & $32.488 \pm 0.066$ \\

      LangFlow\textsuperscript{G}
      & NFE$=1024$ & 5.257 & 36.9
      & $1.371 \pm 0.019$
      & $9.886 \pm 0.193$
      & $30.541 \pm 0.071$ \\

    \addlinespace[2pt]
      \rowcolor{gray!10}
      \multicolumn{7}{l}{\textit{ELF} \cite{hu2026elfembeddedlanguageflows}} \\

      \rowcolor{gray!10}
      ELF-B\textsuperscript{E}
      & SDE, 32 steps, CFG$=3$ & 5.156 & 24.0
      & $2.066 \pm 0.055$
      & $12.817 \pm 0.314$
      & $37.584 \pm 0.218$ \\

      \rowcolor{gray!10}
      ELF-B\textsuperscript{E}
      & SDE, 32 steps, CFG$=1$ & 5.219 & 29.4
      & $2.283 \pm 0.057$
      & $13.633 \pm 0.332$
      & $39.144 \pm 0.232$ \\

      \rowcolor{gray!10}
      ELF-B\textsuperscript{E}
      & ODE, 128 steps, CFG$=1$ & 5.308 & 45.4
      & $1.874 \pm 0.029$
      & $13.565 \pm 0.295$
      & $33.677 \pm 0.116$ \\

      \rowcolor{gray!10}
      ELF-B\textsuperscript{E}
      & ODE, 1024 steps, CFG$=1$ & 5.090 & 24.0
      & $1.778 \pm 0.031$
      & $12.357 \pm 0.292$
      & $33.320 \pm 0.188$ \\

      \rowcolor{gray!10}
      ELF-B\textsuperscript{E}
      & SDE, 128 steps, CFG$=1$ & 5.123 & 21.9
      & $1.654 \pm 0.045$
      & $10.980 \pm 0.257$
      & $32.878 \pm 0.260$ \\

      \rowcolor{gray!10}
      ELF-B\textsuperscript{E}
      & SDE, 64 steps, CFG$=0.5$, $\gamma=0.5$ & 5.385 & 56.3
      & $2.181 \pm 0.035$
      & $14.796 \pm 0.318$
      & $35.648 \pm 0.130$ \\

      \rowcolor{gray!10}
      ELF-M\textsuperscript{E}
      & ODE, 64 steps, CFG$=1.25$, $\gamma=0$ & 5.399 & 53.0
      & $1.695 \pm 0.023$
      & $12.805 \pm 0.268$
      & $32.153 \pm 0.099$ \\

      \rowcolor{gray!10}
      ELF-L\textsuperscript{E}
      & SDE, 64 steps, CFG$=1$, $\gamma=0.75$ & 5.398 & 35.7
      & $1.591 \pm 0.045$
      & $10.039 \pm 0.232$
      & $32.013 \pm 0.274$ \\

      \rowcolor{gray!10}
      ELF-B\textsuperscript{E}
      & SDE, 64 steps, CFG$=1$, $\gamma=1$ & 5.216 & 28.5
      & $1.788 \pm 0.043$
      & $11.952 \pm 0.271$
      & $33.689 \pm 0.228$ \\

      \rowcolor{gray!10}
      ELF-M\textsuperscript{E}
      & SDE, 64 steps, CFG$=1$, $\gamma=1$ & 5.309 & 33.4
      & $1.665 \pm 0.033$
      & $11.719 \pm 0.261$
      & $32.671 \pm 0.205$ \\

      \rowcolor{gray!10}
      ELF-L\textsuperscript{E}
      & SDE, 64 steps, CFG$=1$, $\gamma=1$ & 5.378 & 32.9
      & $1.601 \pm 0.048$
      & $9.829 \pm 0.227$
      & $31.969 \pm 0.290$ \\
    \midrule
      \multicolumn{7}{l}{\textit{Few-Step Models}} \\
      \addlinespace[2pt]
      \addlinespace[2pt]
      \multicolumn{7}{l}{\textit{FLM} \cite{lee2026flowmaplanguagemodels}} \\

      FLM\textsuperscript{G}
      & Euler, 1024 steps & 5.345 & 63.1
      & $3.458 \pm 0.068$
      & $20.126 \pm 0.387$
      & $48.387 \pm 0.189$ \\

      FMLM\textsuperscript{G}
      & 1 step & 5.168 & 167.8
      & $6.857 \pm 0.100$
      & $30.806 \pm 0.647$
      & $79.042 \pm 0.658$ \\

      FMLM\textsuperscript{G}
      & 2 steps & 5.295 & 135.1
      & $5.683 \pm 0.085$
      & $27.826 \pm 0.544$
      & $71.134 \pm 0.494$ \\

      FMLM\textsuperscript{G}
      & 4 steps & 5.361 & 111.0
      & $4.743 \pm 0.070$
      & $25.303 \pm 0.492$
      & $63.074 \pm 0.409$ \\

      \addlinespace[2pt]
      \multicolumn{7}{l}{\textit{FMLM*} \cite{yoo2026selfconditionedflowmaplanguage}} \\

     FMLM$^{*}$\textsuperscript{G}
      & CTM, 1 step, CFG$=1$, $\gamma=1$ & 5.319 & 120.3
      & $3.690 \pm 0.054$
      & $18.808 \pm 0.373$
      & $48.859 \pm 0.169$ \\

      FMLM$^{*}$\textsuperscript{G}
      & CTM, 1 step, CFG$=5$, $\gamma=1$ & 5.368 & 116.0
      & $3.431 \pm 0.046$
      & $18.160 \pm 0.382$
      & $46.273 \pm 0.099$ \\

      FMLM$^{*}$\textsuperscript{G}
      & CTM, 2 steps, CFG$=1$, $\gamma=0.75$ & 5.411 & 99.5
      & $2.932 \pm 0.042$
      & $16.491 \pm 0.320$
      & $43.430 \pm 0.195$ \\

      FMLM$^{*}$\textsuperscript{G}
      & CTM, 2 steps, CFG$=5$, $\gamma=0.75$ & 5.458 & 97.4
      & $2.737 \pm 0.035$
      & $15.941 \pm 0.322$
      & $41.120 \pm 0.139$ \\

      FMLM$^{*}$\textsuperscript{G}
      & CTM, 5 steps, CFG$=1$, $\gamma=1$ & 5.294 & 66.5
      & $2.860 \pm 0.054$
      & $15.556 \pm 0.295$
      & $43.731 \pm 0.188$ \\

      FMLM$^{*}$\textsuperscript{G}
      & CTM, 5 steps, CFG$=5$, $\gamma=1$ & 5.415 & 70.8
      & $2.486 \pm 0.042$
      & $14.869 \pm 0.291$
      & $40.177 \pm 0.114$ \\

      \addlinespace[2pt]
      FMLM$^{*}$\textsuperscript{G}
      & Flow map, 8 steps, CFG$=5$, $\gamma=0$ & 5.464 & 82.8
      & $2.370 \pm 0.032$
      & $14.438 \pm 0.295$
      & $38.233 \pm 0.158$ \\

      FMLM$^{*}$\textsuperscript{G}
      & Flow map, 128 steps, CFG$=5$, $\gamma=0$ & 5.456 & 79.0
      & $2.313 \pm 0.032$
      & $14.268 \pm 0.285$
      & $38.482 \pm 0.152$ \\

      \addlinespace[2pt]

      ELF$^{*}$\textsuperscript{G}
      & CDEQ, 128 steps, CFG$=1$, $\gamma=0$ & 5.414 & 73.8
      & $2.222 \pm 0.034$
      & $14.245 \pm 0.267$
      & $37.827 \pm 0.128$ \\

      ELF$^{*}$\textsuperscript{G}
      & CDEQ, 128 steps, CFG$=5$, $\gamma=0$ & 5.384 & 60.8
      & $1.952 \pm 0.029$
      & $13.105 \pm 0.252$
      & $35.950 \pm 0.103$ \\

      \bottomrule
    \end{tabular}%
    }
    \caption{\small
      OpenWebText baselines using 5,000 samples per distribution and a
      1,024-token horizon. The standard SOL hyperparameters are used;
      see Section~\ref{sec:setup}.
      Superscripts G, E and B identify the evaluation reference family.
      SOL is reported as mean $\pm$ population std over
      nine crossed direction/GP seed pairs.
    }
    \label{tab:experiment-8d-owt-baselines}
  \end{table*}
\end{document}